\documentclass[accepted]{uai2022} 

\usepackage[american]{babel}

\usepackage{natbib} 
\usepackage{mathtools} 
\usepackage{booktabs} 
\usepackage{tikz} 

\usepackage[utf8]{inputenc} 
\usepackage[T1]{fontenc}    
\usepackage{hyperref}       
\usepackage{url}            
\usepackage{amsfonts}       
\usepackage{nicefrac}       
\usepackage{microtype}      
\usepackage{xcolor}         

\usepackage{microtype}
\usepackage{graphicx}
\usepackage{booktabs} 

\usepackage{commath}
\usepackage{amsthm}
\usepackage{amsmath}
\usepackage{amssymb}
\usepackage{bbm}
 \usepackage{subcaption}

\usepackage[toc,page]{appendix}

\usepackage{algorithm}
\usepackage{algorithmic}

\allowdisplaybreaks

\newtheorem{theorem}{Theorem}

\newtheorem{proposition}{Proposition}
\newtheorem{lemma}{Lemma}

\newtheorem{definition}{Definition}

\title{On Provably Robust Meta-Bayesian Optimization}

\author[1]{\href{mailto:<dzx@nus.edu.sg>?Subject=Your UAI 2022 paper}{Zhongxiang Dai}{}}
\author[1]{Yizhou Chen}
\author[2]{Haibin Yu}
\author[1]{Bryan Kian Hsiang Low}
\author[3]{Patrick Jaillet}
\affil[1]{%
Department of Computer Science, National University of Singapore, Republic of Singapore
}
\affil[2]{%
Department of Data Platform, Tencent
}
\affil[3]{%
Department of Electrical Engineering and Computer Science, Massachusetts Institute of Technology, USA
}
  
\begin{document}
\maketitle

\begin{abstract}
\emph{Bayesian optimization} (BO) has become popular for sequential optimization of black-box functions. When BO is used to optimize a target function, we often have access to previous evaluations of potentially related functions. This begs the question as to whether we can leverage these previous experiences to accelerate the current BO task through \emph{meta-learning} (meta-BO), while ensuring \emph{robustness} against potentially harmful dissimilar tasks that could sabotage the convergence of BO. This paper introduces two scalable and provably robust meta-BO algorithms: \emph{robust meta-Gaussian process-upper confidence bound} (RM-GP-UCB) and \emph{RM-GP-Thompson sampling} (RM-GP-TS). We prove that both algorithms are asymptotically no-regret even when some or all previous tasks are dissimilar to the current task, and show that RM-GP-UCB enjoys a better theoretical robustness than RM-GP-TS. We also exploit the theoretical guarantees to optimize the weights assigned to individual previous tasks through regret minimization via online learning, which diminishes the impact of dissimilar tasks and hence further enhances the robustness. Empirical evaluations show that (a) RM-GP-UCB performs effectively and consistently across various applications, and (b) RM-GP-TS, despite being less robust than RM-GP-UCB both in theory and in practice, performs competitively in some scenarios with less dissimilar tasks and is more computationally efficient. 
\end{abstract}


\section{Introduction}
\label{sec:intro}
\emph{Bayesian optimization} (BO) has recently gained immense popularity as an efficient method to optimize black-box functions~\citep{shahriari2016taking},
and it has found success in a variety of applications such as 
automated \emph{machine learning} (ML)~\citep{snoek2012practical}, \emph{reinforcement learning} (RL)~\citep{wilson2014using}, among others.
BO uses a \emph{Gaussian process} (GP)~\citep{rasmussen2004gaussian} as a surrogate to represent the belief about the objective function and, in each iteration, 
queries the input parameters that maximize an \emph{acquisition function}.
In particular, the BO algorithms based on the \emph{GP-upper confidence bound} (GP-UCB)~\citep{srinivas2009gaussian} and \emph{GP-Thompson sampling} (GP-TS)~\citep{chowdhury2017kernelized} acquisition functions have been shown to be asymptotically \emph{no-regret} and perform competitively in practice.
When using BO to optimize a \emph{target function}, we sometimes have access to a set of evaluations of some potentially related functions.
For example, when using BO for hyperparameter optimization of an ML model trained on a target dataset, 
we often have access to some previously completed BO tasks using other potentially related datasets~\citep{golovin2017google}.
These previous tasks, if similar to the current task, may be exploited to accelerate the current BO task.
However, if some (or even all) previous tasks are in fact dissimilar to the current task, their use may turn out to 
incorporate harmful information and 
sabotage the convergence of BO~\citep{feurer2018scalable}.
This begs the question as to whether we can leverage 
previous tasks to improve the efficiency of the current BO task, while
ensuring \emph{robustness} against 
harmful dissimilar tasks such that they do not affect the trademark \emph{no-regret} convergence of BO. 

Exploiting previous learning experiences to improve the efficiency of the current task is the goal of \emph{meta-learning}~\citep{vanschoren2018meta}.
Meta-learning is a broad field with various applications in supervised learning~\citep{finn2017model}, 
RL~\citep{xu2018meta}, active learning~\citep{pang2018meta}, among others. 
The major challenges in meta-learning include (a) the transfer of information from previous tasks to the current task, and 
(b) characterization of task similarity which is crucial for identifying harmful dissimilar tasks~\citep{vanschoren2018meta}.
The application of meta-learning to BO (or \emph{meta-BO}) has been explored by previous studies which differ in how these two challenges are addressed.
Some works, such as multitask BO~\citep{swersky2013multi}, transfer the information from previous tasks by building a joint GP surrogate using the observations from all previous and current tasks,
with the task similarity either represented by meta-features~\citep{bardenet2013collaborative,yogatama2014efficient} or learned from observations~\citep{swersky2013multi,wang2018regret}.
These works, however, are limited by the scalability of GP due to including all previous and current observations in a single GP~\citep{feurer2018scalable}.\footnote{Some works such as~\citet{perrone2018scalable} and~\citet{volpp2020meta} replace GP by other surrogate models such as neural networks for scalability, however, they lack the principled uncertainty estimate and theoretical guarantee offered by GP.}
To this end, other recent works transfer information from previous tasks using a more scalable approach:
They build a separate GP surrogate for each individual task and use a weighted combination of either the individual surrogate functions 
or acquisition functions for query selection~\citep{feurer2018scalable,wistuba2016two,wistuba2018scalable}. 
A more detailed review of related works is presented in Sec.~\ref{sec:related_works}.
However, none of the previous works has provided a theoretical performance guarantee to ensure robust performances in the presence of harmful dissimilar tasks.
A robust theoretical guarantee is important for guaranteeing the consistent performances of meta-BO algorithms in various real-world applications, which is crucial for their practical deployment.

To this end, this paper introduces two scalable and provably robust meta-BO algorithms: \emph{robust meta-GP-upper confidence bound} (RM-GP-UCB) and \emph{robust meta-GP-Thompson sampling} (RM-GP-TS).
Both algorithms compute the acquisition function (GP-UCB or GP-TS) for each individual task 
and select the next query via either a weighted combination (RM-GP-UCB) or in a probabilistic way (RM-GP-TS) (Sec.~\ref{sec:om_gp_ucb}).
As a result, like the works of~\citet{feurer2018scalable,wistuba2016two,wistuba2018scalable}, a separate GP surrogate is built for each previous task, making our algorithms scale well in the number of meta-tasks and observations in each meta-task.
Our major contributions include:
\textbf{Firstly}, we prove robust theoretical convergence guarantees for both RM-GP-UCB and RM-GP-TS (Sec.~\ref{sec:theoretical_analysis}). In particular, both algorithms are asymptotically \emph{no-regret} for \emph{any} given set of previous tasks, i.e., even if some or all previous tasks are dissimilar to the target task. Moreover, we show that RM-GP-UCB enjoys a superior robustness guarantee compared with RM-GP-TS (Sec.~\ref{subsec:theory:ts}).
\textbf{Secondly}, to further enhance our robustness against dissimilar tasks, we exploit the theoretical guarantees to learn the task similarity (and hence identify dissimilar tasks) in a principled way, 
by minimizing the regret upper bounds via a computationally cheap online learning algorithm known as \emph{Follow-The-Regularized-Leader} (Sec.~\ref{sec:online_regret_minimization}).
\textbf{Lastly}, we use extensive empirical evaluations to show that: RM-GP-UCB performs effectively and consistently across a wide range of tasks; RM-GP-TS, despite under-performing in adverse scenarios (i.e., when a large number of previous tasks are dissimilar), performs competitively in some favorable cases with less dissimilar tasks and is much more computationally efficient.
Of note, our theoretical and empirical comparisons between RM-GP-UCB and RM-GP-TS may provide useful insights for other meta-BO algorithms in general (and potentially for other related algorithms such as meta-RL) in terms of the relative strengths and weaknesses of UCB- and TS-based meta-learning algorithms.
\section{Background and Problem Formulation}
\label{sec:background}
\textbf{Bayesian Optimization.}
This work tackles the problem of sequentially maximizing an unknown function $f:\mathcal{D}\rightarrow\mathbb{R}$. 
In each iteration $t=1,\ldots,T$, an input $\mathbf{x}_t \in \mathcal{D}$ (a $D\geq 1$-dimensional vector) is 
queried to yield
$y_t\triangleq f(\mathbf{x}_t) + \epsilon$ where $\epsilon \sim \mathcal{N}(0,\sigma^2)$ is a Gaussian noise with variance $\sigma^2$.
The performance of BO is typically measured by \emph{cumulative regret}: $R_T\triangleq\sum_{t=1,\ldots,T}[f(\mathbf{x}^*) - f(\mathbf{x}_t)]$ 
where $\mathbf{x}^* \in \arg\max_{\mathbf{x}\in \mathcal{D}}f(\mathbf{x})$ is a global maximizer of $f$. 
It is desirable for a BO algorithm to achieve \emph{no regret} by making its $R_T$ grow sublinearly 
such that its \emph{simple regret} $S_T\triangleq\min_{t=1,\ldots,T}[f(\mathbf{x}^*) - f(\mathbf{x}_t)] \leq R_T/T$ goes to $0$ asymptotically. 
During BO, we model the belief about $f$ using a \emph{Gaussian process} (GP) $\{f(\mathbf{x})\}_{\mathbf{x}\in {\mathcal{D}}}$.
That is, any finite subset of $\{f(\mathbf{x})\}_{\mathbf{x}\in {\mathcal{D}}}$ follows a multivariate Gaussian distribution~\citep{rasmussen2004gaussian}.
A GP is fully specified by its prior mean $\mu(\mathbf{x})$ and kernel function $k(\mathbf{x}, \mathbf{x}')$, and we assume w.l.o.g. that $\mu(\mathbf{x})=0$ and $k(\mathbf{x}, \mathbf{x}') \leq 1$ $\forall \mathbf{x}, \mathbf{x}' \in \mathcal{D}$.
We focus on the widely used Squared Exponential (SE) kernel.
Given 
$T$ noisy observations $\mathbf{y}_{T}\triangleq [y_t]^{\top}_{t=1,\ldots,T}$ at inputs $\mathbf{x}_1,\ldots,\mathbf{x}_T$, the posterior GP belief of $f$ at 
input $\mathbf{x} \in \mathcal{D}$ is Gaussian with the following posterior mean and variance:
\begin{equation}
\begin{split}
    \mu_T(\mathbf{x}) \triangleq\displaystyle\mathbf{k}_T(\mathbf{x})^\top(\mathbf{K}_T+\lambda I)^{-1}\mathbf{y}_{T}, \,\\
    \sigma_T^2(\mathbf{x}) \triangleq\displaystyle k(\mathbf{x},\mathbf{x})-\mathbf{k}_T(\mathbf{x})^\top(\mathbf{K}_T+\lambda I)^{-1}\mathbf{k}_T(\mathbf{x}),
\end{split}
\label{gp_posterior}
\end{equation}
where $\mathbf{K}_T\triangleq [k(\mathbf{x}_t,\mathbf{x}_{t'})]_{t,t'=1,\ldots,T}$, $\mathbf{k}_T(\mathbf{x})\triangleq [k(\mathbf{x}_t,\mathbf{x})]^\top_{t=1,\ldots,T}$, $\lambda$ is a regularization parameter.

\textbf{Meta-Bayesian Optimization.} 
We refer to the function $f$ being maximized as the \emph{target function} and the functions $f_i$ for $i=1,\ldots, M$ of the $M$ previous tasks as \emph{meta-functions}.
We use \textit{target task/observations} and \textit{meta-tasks/observations} in a similar manner. 
All functions are defined on the same domain $\mathcal{D}$ which is assumed to be discrete for simplicity, 
but the theoretical results can be easily generalized to continuous domains following the analysis of previous works~\citep{chowdhury2017kernelized,srinivas2009gaussian}.
We assume that $f$ and all $f_i$'s lie in the \emph{reproducing kernel Hilbert space} (RKHS) associated with the kernel $k$ such that their norm induced by the RKHS is bounded: $\norm{f}_{k} \leq B$, $\norm{f_i}_{k}\leq B,\forall i=1,\ldots,M$.
This assumption intuitively suggests that the target and meta-functions have the same degree of smoothness.
Same as the work of~\citet{wang2018regret} which has also performed theoretical analysis of a meta-learning algorithm for BO, 
we also assume that all meta- and target observations are corrupted by a Gaussian noise $\epsilon \sim \mathcal{N}(0,\sigma^2)$ with variance $\sigma^2$.
The number of observations from meta-task $i$ is a constant denoted as $N_i$, and  $N\triangleq\max_{i=1,\ldots,M}N_i$.
$\mathbf{x}_{i,j}$ and $y_{i,j}$ represent the $j$-th input and noisy output of meta-task $i$ respectively.
We define the \emph{function gap} $d_i\triangleq \max_{j=1,\ldots,N_i}\left|f(\mathbf{x}_{i,j})-f_i(\mathbf{x}_{i,j})\right| < \infty$  
which represents the maximum difference between the function values of $f$ and $f_i$ at any corresponding input $\mathbf{x}_{i,j}$ of meta-task $i$.
Note that for a given set of meta-observations for meta-task $i$, the function gap $d_i$ is an unknown constant characterizing the similarity between meta-task $i$ and the target task: 
a smaller function gap implies a stronger similarity.


\section{Robust Meta-Bayesian Optimization}
\label{sec:om_gp_ucb}
The acquisition function~\eqref{acq_func} adopted by RM-GP-UCB in iteration $t$ is a weighted combination of $M+1$ individual GP-UCB acquisition functions~\citep{srinivas2009gaussian} for the target task and the $M$ meta-tasks, each of which is calculated using the observations from a particular task:
\begin{equation}
\begin{split}
\overline{\zeta}^{\text{UCB}}_t(\mathbf{x})\triangleq &\nu_t\Big[{\sum}^M_{i=1}\omega_i \left[\overline{\mu}_{i}(\mathbf{x}) + \tau\overline{\sigma}_{i}(\mathbf{x})\right]\Big] + \\
&\left(1-\nu_t\right)\left[\mu_{t-1}(\mathbf{x})+\beta_t\sigma_{t-1}(\mathbf{x})\right].
\end{split}
\label{acq_func}
\end{equation}
In~\eqref{acq_func}, $\mu_{t-1}(\mathbf{x})$ and $\sigma_{t-1}(\mathbf{x})$ represent, respectively, the GP posterior mean and standard deviation~\eqref{gp_posterior} at $\mathbf{x}$ 
calculated using the target observations from iterations 1 to $t-1$. 
$\overline{\mu}_{i}(\mathbf{x})$ and $\overline{\sigma}_{i}(\mathbf{x})$
are computed using all meta-observations from meta-task $i$. 
$\beta_t>0$ and $\tau>0$ will be defined in Sec.~\ref{sec:theoretical_analysis}.
$\nu_t \in [0, 1]$ can be interpreted as the overall weight given to all meta-tasks in iteration $t$ and should be chosen to be non-increasing in $t$, 
which enforces the impact of meta-tasks in~\eqref{acq_func} to be non-increasing.
The \emph{meta-weights} $\omega_i$'s 
can be understood as the weights assigned to individual meta-tasks.
Note that since the dataset used to calculate $\overline{\mu}_{i}(\mathbf{x})$ and $\overline{\sigma}_{i}(\mathbf{x})$ 
is fixed with size $N_i$, the matrix inversion in~\eqref{gp_posterior} (i.e., the computational bottleneck for GP) can be pre-computed.
So, after $T$ iterations, RM-GP-UCB incurs $\mathcal{O}(T^3)$ time for covariance matrix inversion
(since only the target covariance matrix of size $T\times T$ needs to be inverted) 
and $\mathcal{O}(MN^2+T^2)$ time during predictive inference,  
which are less than the respective $\mathcal{O}((MN+T)^3)$ and $\mathcal{O}((MN+T)^2)$ time 
when all observations are included in a single GP.
In practice, the total number of BO iterations ($T$) is usually small, therefore, the differences between these corresponding computational costs can be large, especially when $M$ and $N$ are large.
Hence, RM-GP-UCB is scalable in the number of meta-tasks ($M$) and observations in each meta-task ($N$).

The acquisition function of RM-GP-TS is defined as:
\begin{equation}
\overline{\zeta}^{\text{TS}}_t(\mathbf{x})\triangleq
\begin{cases}
f^t(\mathbf{x}) & \text{with probability } 1-\nu_t \ ,\\
{\sum}^M_{i=1}\omega_i \overline{f}^t_{i}(\mathbf{x}) & \text{with probability } \nu_t,
\end{cases}
\label{eq:acq_func_ts}
\end{equation}
in which $f^t$ is a function sampled from the GP posterior of the target task: $f^t \sim \mathcal{GP}\left(\mu_{t-1}(\cdot), \beta_t^2 \sigma_{t-1}^2(\cdot)\right)$, and $f^t_i$ is sampled from the GP posterior of meta-task $i$: $\overline{f}^t_i \sim \mathcal{GP}\left(\overline{\mu}_{i}(\cdot), \tau^2 \overline{\sigma}_{i}^2(\cdot)\right)$.
Using approximation techniques such as random Fourier features (RFF) approximation~\citep{rahimi2008random} (which we use in all our experiments), the functions $f^t$ and $f_i^t$'s can be sampled efficiently, hence making RM-GP-TS computationally efficient (as we will demonstrate in Sec.~\ref{sec:experiment}).
Moreover, since the meta-observations of every meta-task is fixed, the use of approximation techniques such as RFF allows the functions $\overline{f}^t_i$'s to be sampled beforehand before the algorithm starts.
Refer to Appendix~\ref{app:ts:details} for more details on RM-GP-TS.

In iteration $t$ of either RM-GP-UCB or RM-GP-TS (Algorithm~\ref{OM_GP_UCB}), we first optimize the meta-weights and update $\nu_t$ (Sec.~\ref{sec:online_weight_estimation}), which corresponds to line $2$ of Algorithm~\ref{OM_GP_UCB}.
Next, the input $\mathbf{x}_t$ is selected by maximizing the acquisition function~\eqref{acq_func} (RM-GP-UCB) or~\eqref{eq:acq_func_ts} (RM-GP-TS), after which we query $\mathbf{x}_t$ and 
use the newly collected 
$(\mathbf{x}_t, y_t)$ to update the GP posterior belief~\eqref{gp_posterior}.

\begin{algorithm}
\begin{algorithmic}[1]
	\FOR{$t=1,2,\ldots, T$}
        \STATE Update $\omega_i$ for $i=1,\ldots,M$ via online meta-weight optimization and update $\nu_t$ (Sec.~\ref{sec:online_weight_estimation})
        \STATE $\mathbf{x}_t \leftarrow {\arg\max}_{\mathbf{x} \in \mathcal{D}} \overline{\zeta}^{\text{UCB}}_t(\mathbf{x})$ (for RM-GP-UCB)~\eqref{acq_func}, or
        $\mathbf{x}_t \leftarrow {\arg\max}_{\mathbf{x} \in \mathcal{D}} \overline{\zeta}^{\text{TS}}_t(\mathbf{x})$ (for RM-GP-TS)~\eqref{eq:acq_func_ts}
        \STATE Query $\mathbf{x}_t$ to observe $y_t$, and update GP posterior belief~\eqref{gp_posterior} using $(\mathbf{x}_t, y_t)$
	\ENDFOR
\end{algorithmic}
\caption{RM-GP-UCB/RM-GP-TS}
\label{OM_GP_UCB}
\end{algorithm}

\section{Theoretical Analysis}
\label{sec:theoretical_analysis}
\subsection{RM-GP-UCB}
\label{subsec:theory:rm_gp_ucb}
Theorem~\ref{regret_bound} presents an upper bound on the cumulative regret of RM-GP-UCB (proof in Appendix~\ref{app:first_section}).
\begin{theorem}[RM-GP-UCB]
\label{regret_bound}
Let $\delta \in (0,1)$.
Denote by $\gamma_t$ the maximum information gain about $f$ from observing any set of $t$ observations.
If RM-GP-UCB is run with:
$\lambda = 1+2/T$,
$\beta_t=B + \sigma \sqrt{2(\gamma_{t-1} + 1 + \log(4/\delta))}$,
$\tau=B + \sigma \sqrt{2(\gamma_{N} + 1 + \log(4M/\delta))}$,
$\nu_t\in [0, 1]$ and $\nu_{t+1} \leq \nu_{t}$, 
$\omega_i\geq 0$ and $\sum^M_{i=1}\omega_i=1$.
Then, with probability of $\geq 1 - 3\delta / 4$,
\begin{align}
    R_T \leq 2(\alpha+\tau) \sum^T_{t=1} \nu_t + \beta_T\sqrt{C_1 T \gamma_T} \nonumber \\
    = \widetilde{\mathcal{O}}\big( \big(\sum^M_{i=1}d_i\big) \sum^T_{t=1} \nu_t + \gamma_T\sqrt{T} \big),
\label{eq:regret:ucb:within:theorem}
\end{align}
where 
$C_1 \triangleq \frac{8}{1+\sigma^{-2}}$,
and $\alpha \triangleq \sum^M_{i=1}\omega_i \frac{N_i}{\sigma^2}(2\sqrt{2\sigma^2\log\frac{8N_i}{\delta}}+d_i).$
\end{theorem}

The second term $\gamma_T \sqrt{T}$ in the regret upper bound~\eqref{eq:regret:ucb:within:theorem} grows sub-linearly for the SE kernel for which $\gamma_T=\mathcal{O}((\log T)^{D+1})$.
Therefore, if $\nu_t$ is designed such that $\nu_t \rightarrow 0$ as $t \rightarrow \infty$, the first term also grows sub-linearly and hence RM-GP-UCB is asymptotically no-regret.


Theorem~\ref{regret_bound} holds for a given set of meta-tasks with fixed yet unknown $d_i$'s.
Note that we do not impose assumptions on the values of $d_i$'s, i.e., the similarities between the meta- and target tasks.
Therefore, Theorem~\ref{regret_bound} gives a robust regret upper bound which holds for \emph{any} given set of meta-tasks.
In other words, even in adverse scenarios where some or all meta-tasks are extremely dissimilar to the target task (i.e., when some or all $d_i$'s are very large),
RM-GP-UCB is still asymptotically no-regret, which indicates the robustness and generality of our algorithm.
This provides an assurance about the \emph{worst-case behavior} in any given scenario.\footnote{This notion of robustness is in line with that of \emph{robust optimization} (RO)~\citep{beyer2007robust} which also attempts to optimize the performance in the worst-case scenario. The difference is that RO optimizes an explicit objective, while we aim at preserving the no-regret property in the worst case.}
In our proof, the key step (Lemma~\ref{ucb_diff} in Appendix~\ref{app:first_section}) is to upper bound (by $\alpha$ in Theorem~\ref{regret_bound})
the overall error induced by the use of any given set of meta-observations, instead of the target observations at the same corresponding input locations, when calculating the acquisition function~\eqref{acq_func}.
These interpretations also explain the dependence of $\alpha$, hence the regret bound, on $d_i$ and $N_i$: 
Larger function gaps increase the error resulting from the use of the meta-observations, and 
a larger number of meta-observations also inflates the worst-case upper bound by accumulating the individual errors.
Of note, a limitation of our regret upper bound (Theorem~\ref{regret_bound}) is that it does not reflect the benefit of the use of the meta-tasks when they are indeed similar to the target task. Next, we use our theoretical analysis to give some insights on how the meta-tasks, if similar to the target task, help improve the convergence of our algorithm.

\textbf{Meta-tasks Can Improve the Convergence by Accelerating Exploration.}
In addition to characterizing the worst-case behavior, we also use our theoretical analysis to illustrate how meta-tasks can help RM-GP-UCB converge faster than standard GP-UCB.
As we have proved in Appendix~\ref{app:improved_bound}, at the early stage of the algorithm, the meta-tasks (if similar to the target task) can help RM-GP-UCB obtain a smaller regret upper bound than GP-UCB by \emph{reducing the uncertainty at the selected input}. 
Equivalently, the additional information from the meta-tasks allows RM-GP-UCB to \emph{reduce the degree of exploration at the early stage}.
Since initial exploration of BO usually incurs large regrets,
less exploration results in smaller regrets.
At later stages when $\nu_t$ becomes close to $0$, RM-GP-UCB converges to no regret at a similar rate to GP-UCB (i.e., the second term $\gamma_T \sqrt{T}$ in the regret upper bound~\eqref{eq:regret:ucb:within:theorem} dominates).

\subsection{RM-GP-TS}
\label{subsec:theory:ts}
Theorem~\ref{regret_bound_ts} gives an upper bound on the cumulative regret of RM-GP-TS (proof in Appendix~\ref{app:proof:theorem:ts}).
\begin{theorem}[RM-GP-TS]
\label{regret_bound_ts}
Define $d_i' \triangleq \max_{\mathbf{x}\in\mathcal{D}}| f(\mathbf{x}) - f_i(\mathbf{x}) |$.
With the same parameters as those defined in Theorem~\ref{regret_bound},
we have that with probability of at least $1 - 3\delta / 4$,
\[
R_T = \widetilde{\mathcal{O}}\big(\big(\sum^M_{i=1}\omega_i d_i' \big) \sum^T_{t=1} \nu_t + 
\sum^T_{t=1} \nu_t \sqrt{\gamma_t} + 
\gamma_T \sqrt{T} \big).
\]
\end{theorem}
Note that by definition, we have that $d_i' \geq d_i, \forall i$.
Similar to RM-GP-UCB, as long as $\nu_t$ is chosen such that $\nu_t \rightarrow 0$ as $t \rightarrow \infty$ and that
$\nu_t = o(1/\sqrt{\gamma_t})$,
all three terms in Theorem~\ref{regret_bound_ts} are sub-linear (for the SE kernel).
That is, RM-GP-TS is also asymptotically no-regret 
for any set of meta-tasks, even when some or all meta-tasks are dissimilar to the target task.
Moreover, comparing the extra terms in the regret upper bounds resulting from the use of the meta-tasks for both RM-GP-UCB (i.e., the first term of equation~\eqref{eq:regret:ucb:within:theorem} in Theorem~\ref{regret_bound})
and RM-GP-TS (i.e., the first two terms of Theorem~\ref{regret_bound_ts}) reveals that compared with RM-GP-UCB, RM-GP-TS suffers from a worse extra dependence on $T$ due to the meta-tasks.
Specifically, while the first terms of Theorems~\ref{regret_bound} and~\ref{regret_bound_ts} have the same dependence on $T$, the second term of Theorem~\ref{regret_bound_ts} introduces an extra dependence on $T$ which dominates the first term.
This suggests that in adverse scenarios with a large number of dissimilar tasks, RM-GP-TS may suffer from a worse convergence than RM-GP-UCB.
In other words, RM-GP-UCB enjoys a better theoretically guaranteed robustness against dissimilar tasks.



\subsection{Practical Implications}
Besides the theoretical insights, Theorems~\ref{regret_bound} and~\ref{regret_bound_ts} also provide two natural hints to the practical algorithmic design.
Firstly, note that both Theorems hold for all choices of meta-weights $\omega_i$'s.
Therefore,
we have the flexibility to choose the optimal $\omega_i$'s (i.e., learn the task similarity) by minimizing the regret upper bounds in Theorems~\ref{regret_bound} and~\ref{regret_bound_ts}.
Secondly, the first term in Theorem~\ref{regret_bound} suggests that we can lower the regret by making $\nu_t$ (i.e., the influence of the meta-tasks) decay faster if $\alpha$ in Theorem~\ref{regret_bound} (i.e., an upper bound on the error produced by using the meta-tasks) is larger.
The same reasoning applies to Theorem~\ref{regret_bound_ts}, i.e., we can decay $\nu_t$ faster if $\sum_{i=1,\ldots,M}\omega_i d_i'$ in Theorem~\ref{regret_bound_ts} is larger.
Both design choices can further strengthen the robustness of our algorithms against dissimilar meta-tasks by lessening their impact.
Unfortunately, they both require the values of the function gaps $d_i$'s which are unavailable.\footnote{$d_i$ can be used as an estimate of $d_i'$ since $d_i' \geq d_i$ (Sec.~\ref{subsec:theory:ts}).}
To this end, we devise a principled technique to estimate upper bounds on the function gaps, which is presented
in the next section.


\section{Online Meta-Weight Optimization}
\label{sec:online_regret_minimization}
In this section, we first introduce a principled technique for estimating high-probability upper bounds on the function gaps (Sec.~\ref{sec:estimate_d}) that, 
when combined with Theorems~\ref{regret_bound} and~\ref{regret_bound_ts}, naturally yields a principled method for optimizing the meta-weights 
through regret minimization via online learning.

\subsection{Online Estimation of Function Gaps}
\label{sec:estimate_d}
Inspired by the confidence region constructed by GP-UCB~\citep{srinivas2009gaussian,chowdhury2017kernelized} that contains the target function with high probability, 
after $t\geq 1$ target observations have been collected, define
\begin{equation}
\begin{split}
    &U_{t,i,j}\triangleq\mu_{t}(\mathbf{x}_{i,j}) + \beta_{t+1}\sigma_{t}(\mathbf{x}_{i,j})\, , \\
    &L_{t,i,j}\triangleq\mu_{t}(\mathbf{x}_{i,j}) - \beta_{t+1}\sigma_{t}(\mathbf{x}_{i,j}),
\end{split}
\label{UL}
\end{equation}
where $\mathbf{x}_{i,j}$ is the $j$-th input of meta-task $i$, $\beta_{t+1}$ is previously defined in Theorem~\ref{regret_bound}, and 
$U_{t,i,j}$ and $L_{t,i,j}$ can be interpreted, respectively, as the upper and lower confidence bounds of $f$ at $\mathbf{x}_{i,j}$ after $t$ iterations.
Lemma \ref{gaussian_bound} (Appendix \ref{app:first_section}) implies that
with probability of at least $1 - \delta/4$ 
($\delta$ is defined in Theorem \ref{regret_bound}):
$L_{t,i,j} \leq f(\mathbf{x}_{i,j}) \leq U_{t,i,j}, \forall t, i, j\ $.
Consequently, the following result gives high-probability upper bounds on the function gaps (proof in Appendix~\ref{app:upper_bound_func_gap}):
\begin{lemma}
\label{estimate_di}
With probability of at least 
$1 - \delta$,
\begin{equation*}
\begin{split}
    d_i \leq &\sqrt{2\sigma^2\log \Big[\big(8{\sum}^M_{i=1}N_i\big)/\delta\Big] } + \\
    &\max_{j=1,...,N_i}\left[\max \{|y_{i,j} - U_{t,i,j}|, |y_{i,j} - L_{t,i,j}|\}\right] \triangleq \overline{d}_{i,t},
\end{split}
\end{equation*}
for $t=1,\ldots, T$ and $i=1,\ldots,M$.
\end{lemma}
Unlike $d_i$, $\overline{d}_{i,t}$ can be efficiently calculated as its incurred time is linear in both $M$ and $N$.

\subsection{Online Meta-Weight Optimization through Regret Minimization}
\label{sec:online_weight_estimation}
In this section, we focus on RM-GP-UCB since the analysis for RM-GP-TS (deferred to Appendix~\ref{app:meta:weight:optimization:ts}) is similar and leads to the same update rules for $\omega_i$'s and $\nu_t$.
Combining Lemma~\ref{estimate_di} and Theorem~\ref{regret_bound} allows us to derive the following result for RM-GP-UCB (proof in Appendix~\ref{app:prop_1_proof}):
\begin{proposition}[RM-GP-UCB]
\label{regret_bound_2}
With probability of $\geq 1 - \delta$,
\begin{equation*}
\begin{split}
    R_T \leq &\frac{2}{\sigma^2} \Big[{\sum}^T_{t=1}\boldsymbol{\omega}^{\top} \boldsymbol{l}_t\Big] \Big[{\sum}^T_{t=1}\nu_t\Big] +\\
    &\quad 2\tau {\sum}^T_{t=1} \nu_t + \beta_T\sqrt{C_1 T \gamma_T},
\end{split}
\end{equation*}
where $\boldsymbol{\omega}\triangleq [\omega_i]_{i=1,\ldots,M}$, $\boldsymbol{l}_t\triangleq [l_{i,t}]_{i=1,\ldots,M}$, and $l_{i,t}\triangleq N_i (2\sqrt{2\sigma^2\log({8N_i/\delta})}+\overline{d}_{i,t})$.
\end{proposition}
Note that $\boldsymbol{l}_t$ can be efficiently computed after the $t$-th observation is collected.
The regret upper bound in Proposition~\ref{regret_bound_2} depends on  $\omega_i$'s only through the term $\sum^T_{t=1}\boldsymbol{\omega}^{\top} \boldsymbol{l}_t$  
which can be minimized to derive the optimal meta-weights.
This constitutes an \emph{online learning} problem with linear loss function and its solution $\boldsymbol{\omega}$ constrained to a probability simplex. 
An additional entropic regularization term is usually preferred so as to encourage a solution with a large entropy to stabilize it~\citep{bubeck2011introduction}.
This corresponds to encouraging the meta-weights to spread across a large number of meta-tasks, in order to discover as many similar meta-tasks as possible. 
As a result, by using $1/\eta$ ($\eta>0$) as the regularization parameter, the optimal $\boldsymbol{\omega}$ in iteration $t>1$ is obtained by solving the following optimization problem:
\begin{equation}
\label{online_learning_objective}
    \boldsymbol{\omega} \triangleq \mathop{\arg\min}_{\boldsymbol{\omega}'} {\sum}^{t-1}_{s=1}\boldsymbol{\omega}'^{\top} \boldsymbol{l}_s + 
    \eta^{-1}{\sum}^M_{i=1}\omega_i'\log \omega_i', 
\end{equation}
subject to the constraints: $\omega_i'\geq 0,\forall i$ and $\sum^M_{i=1}\omega_i'=1$.
When $t=1$, the optimal $\boldsymbol{\omega}$ follows from optimizing only the entropic regularization term, thus naturally entailing the uniform distribution $\omega_{i}=1/M,\forall i$.
Consequently,~\eqref{online_learning_objective} corresponds exactly to the online learning algorithm called \emph{Follow-The-Regularized-Leader} with an entropic regularizer~\citep{bubeck2011introduction} 
where $\eta$ represents the learning rate. Its optimal solution in iteration $t$ can be derived via Lagrange multiplier (Appendix~\ref{lagran}) as
\begin{equation}
\begin{split}
    \omega_{i} = \frac{e^{-\eta \sum^{t-1}_{s=1}l_{i,s}}}{\sum^M_{j=1}e^{-\eta \sum^{t-1}_{s=1}l_{j,s}}} \stackrel{(a)}{\approx}\frac{e^{-\eta N \sum^{t-1}_{s=1}\overline{d}_{i,s}}}{\sum^M_{j=1}e^{-\eta N \sum^{t-1}_{s=1}\overline{d}_{j,s}}},
\end{split}
\label{estimate_wi}
\end{equation}
for $i=1,\ldots,M$ where (a) follows from assuming that all $N_i$'s are close to $N$ for simplicity.
With this simplification, the first (noise-correction) term in the expression of $\overline{d}_{i,t}$ from Lemma~\ref{estimate_di} also cancels out, thus 
leading to a neat and elegant update rule for $\omega_i$ 
which we use in all our experiments.
As is evident from~\eqref{estimate_wi}, the update of $\omega_i$'s in each iteration only involves computing $\overline{d}_{i,t}$'s (incurring $\mathcal{O}(MN)$ time), 
adding one term to the summation on the exponent ($\mathcal{O}(M)$ time), and a normalization step ($\mathcal{O}(M)$ time), all of which are computationally cheap.
Intuitively,~\eqref{estimate_wi} assigns small weights to meta-tasks with a large cumulative estimated function gap which implies a less similar meta-task.

In addition, $\overline{d}_{i,t}$ from Lemma~\ref{estimate_di} also allows for the estimation of an upper bound on $\alpha$ (Theorem~\ref{regret_bound}) in each iteration 
(i.e., by simply replacing $d_i$ with $\overline{d}_{i,t}$) and thus facilitates an adaptive selection of $\nu_t$, as mentioned in Sec.~\ref{sec:theoretical_analysis}. 
Specifically, we set $\nu_1= 1$ and $\nu_t = \nu_{t-1}\times \min(r, (\sum^M_{i=1}\omega_i\overline{d}_{i,t})^{-\epsilon})$ for $t>1$,
in which we have dropped the constants independent of $\overline{d}_{i,t}$. 
$r \in (0,1)$ represents the minimum decaying rate to ensure the monotonic decay of $\nu_t$ such that RM-GP-UCB is no-regret (Sec.~\ref{subsec:theory:rm_gp_ucb}).
$\epsilon>0$ controls the aggressiveness of the adaptive decay such that a larger $\epsilon$ results in a faster decay.
With this scheme, when the overall estimated function gaps are larger (the meta-tasks are dissimilar), 
$\nu_t$ decays faster and thus the impact of the meta-tasks vanishes more quickly.


Importantly, when optimizing the values of $\omega_i$'s and $\nu_t$ as described above, we have taken into account the limitation of our regret upper bounds (i.e., they do not reflect the benefit of the use of the meta-tasks, Sec.~\ref{subsec:theory:rm_gp_ucb}) and hence incorporated additional practical considerations. Specifically, we have optimized the $\omega_i$'s with an additional entropic regularization term to encourage the $\omega_i$'s to spread across a large number of meta-tasks, and optimized $\nu_t$ such that it decreases faster if $\alpha$ (i.e., an upper bound on the error induced by the use of the meta-tasks) is larger.

\section{Experiments and Discussion}
\label{sec:experiment}
We use extensive real-world experiments to compare our RM-GP-UCB and RM-GP-TS with \emph{(1)} standard GP-UCB, two other GP-based scalable meta-BO algorithms: 
\emph{(2)} \emph{ranking-weighted Gaussian process ensemble} (RGPE)~\citep{feurer2018scalable} and \emph{(3)} \emph{transfer acquisition function} (TAF)~\citep{wistuba2018scalable},
\emph{(4)} multitask BO (MTBO)~\citep{swersky2013multi}, and \emph{(5)} the method from~\citep{wang2018regret} named \emph{point estimate meta-BO} (PEM-BO).
Since MTBO is relatively not scalable (Sec.~\ref{sec:intro}), we only apply it to those experiments with relatively small number of meta-tasks and observations for which MTBO is still computationally feasible.
We compare with PEM-BO~\citep{wang2018regret} in the experiment that is most favorable for this algorithm, i.e., with the largest number of meta-observations and a discrete domain (refer to Sec.~\ref{subsec:automl} for more details).
We set $\eta= 1/N$, $\epsilon= 0.7$ and $r= 0.7$ in all real-world experiments to demonstrate the robustness of our algorithm against the choice of these parameters.
In practice, the upper bound on the function gap, $\overline{d}_{i,t}$, from Lemma~\ref{estimate_di} may be too conservative; 
so, we replace the outer $\max$ operator over $j=1,...,N_i$ with the empirical mean in our experiments.\footnote{We explore the difference between 
them
in Appendix~\ref{app:subsec_max_mean}.}
Some details and results are deferred to Appendix~\ref{app:experiments} due to lack of space.
All error bars represent standard errors.
Our code is available at \url{https://github.com/daizhongxiang/meta-BO}.

\subsection{Synthetic Experiments}
\label{exp:synth}
We firstly explore the effectiveness of our online meta-weight optimization (Sec.~\ref{sec:online_regret_minimization}) and the impact of different algorithmic parameters by optimizing synthetic functions drawn from GPs.
For each objective function, we construct $M=4$ meta-tasks with $N=N_i=20$ meta-observations each. 
The function gaps are chosen as $d_1=d_2=0.05$ and $d_3=d_4=4.0$ such that 
the last $2$ meta-tasks are dissimilar to the target task.
Fig.~\ref{fig:synth_func_results}a plots the simple regrets averaged over $20$ randomly drawn synthetic functions, with $\eta N =1.0$, $\epsilon=0.7$, and $r=0.7$.
The figure shows that RM-GP-UCB with online meta-weight optimization 
significantly outperforms RM-GP-UCB with fixed meta-weights ($\omega_i=1/4$ for all $i$).
Fig.~\ref{fig:synth_func_results}b plots the meta-weights optimized by RM-GP-UCB for the red curve in Fig.~\ref{fig:synth_func_results}a, showing that the weights given to the last two meta-tasks which are dissimilar to the target task are rapidly reduced.
These results verify the effectiveness of online meta-weight optimization in reducing the impact of dissimilar meta-tasks.

\begin{figure}
	\centering
	\begin{tabular}{cc}
		\hspace{-4mm} \includegraphics[width=0.47\linewidth]{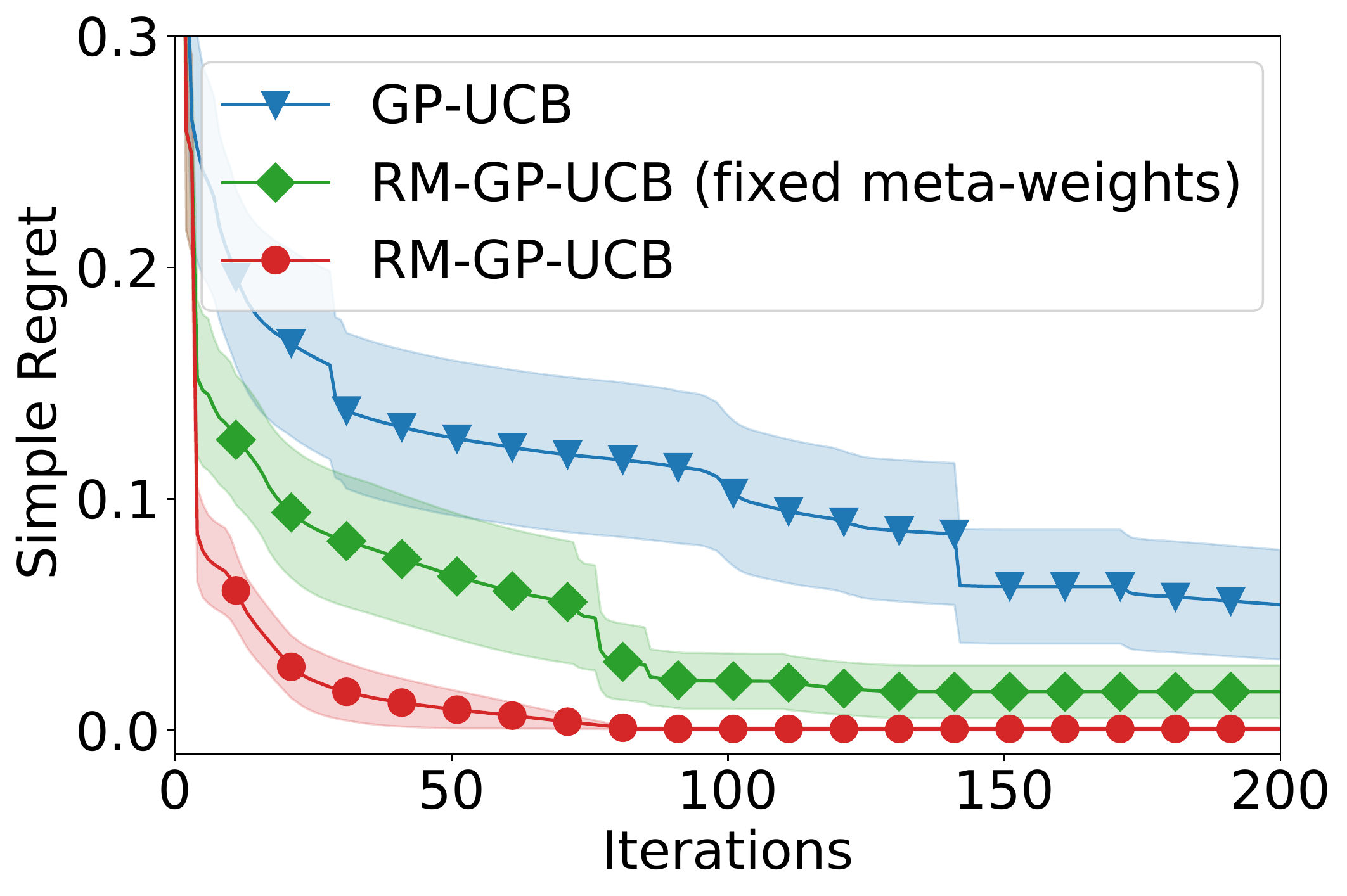} & \hspace{-6.2mm}
		\includegraphics[width=0.47\linewidth]{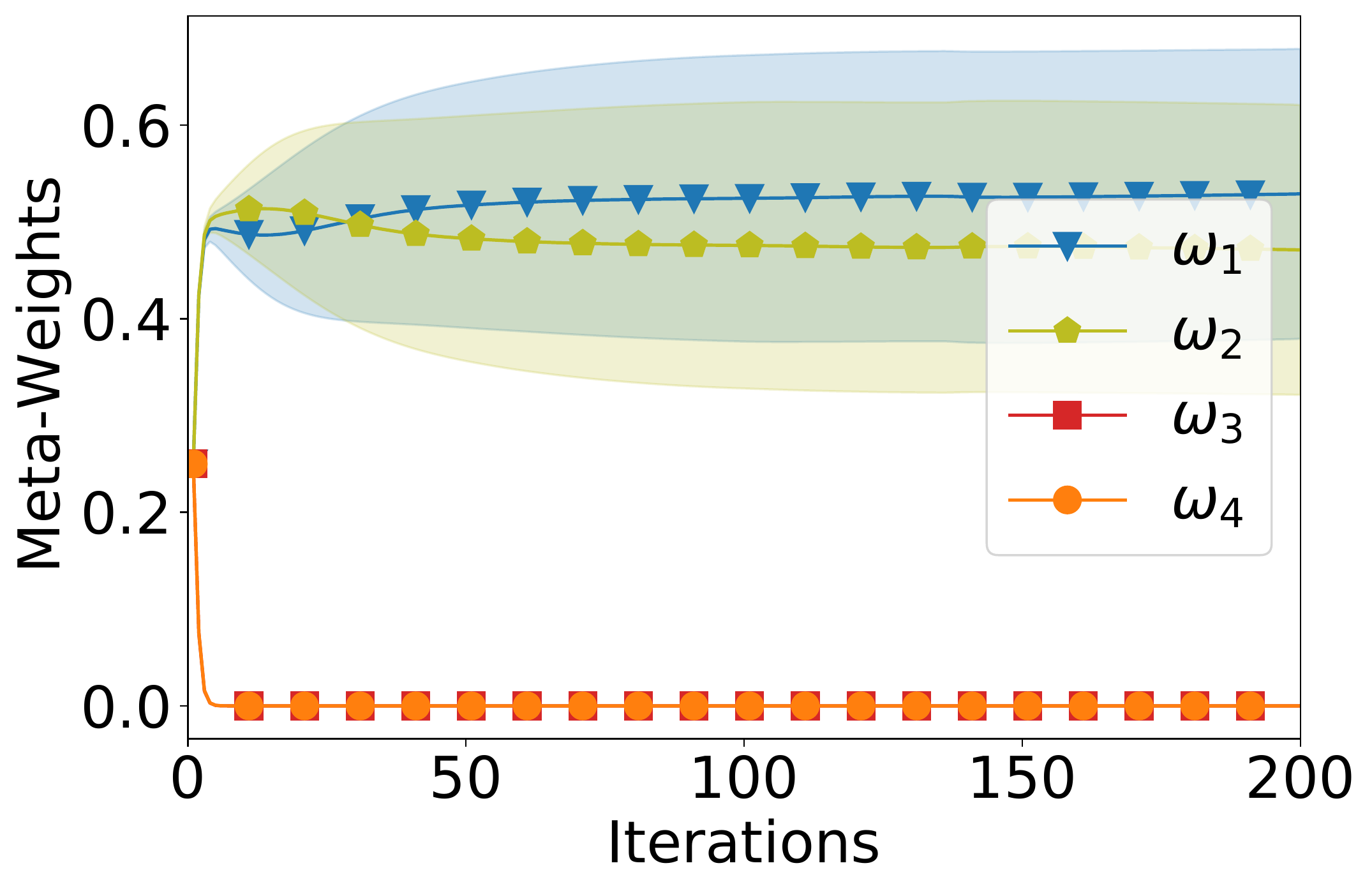}\\
		{(a)} & {(b)}\\
		\hspace{-4mm} \includegraphics[width=0.47\linewidth]{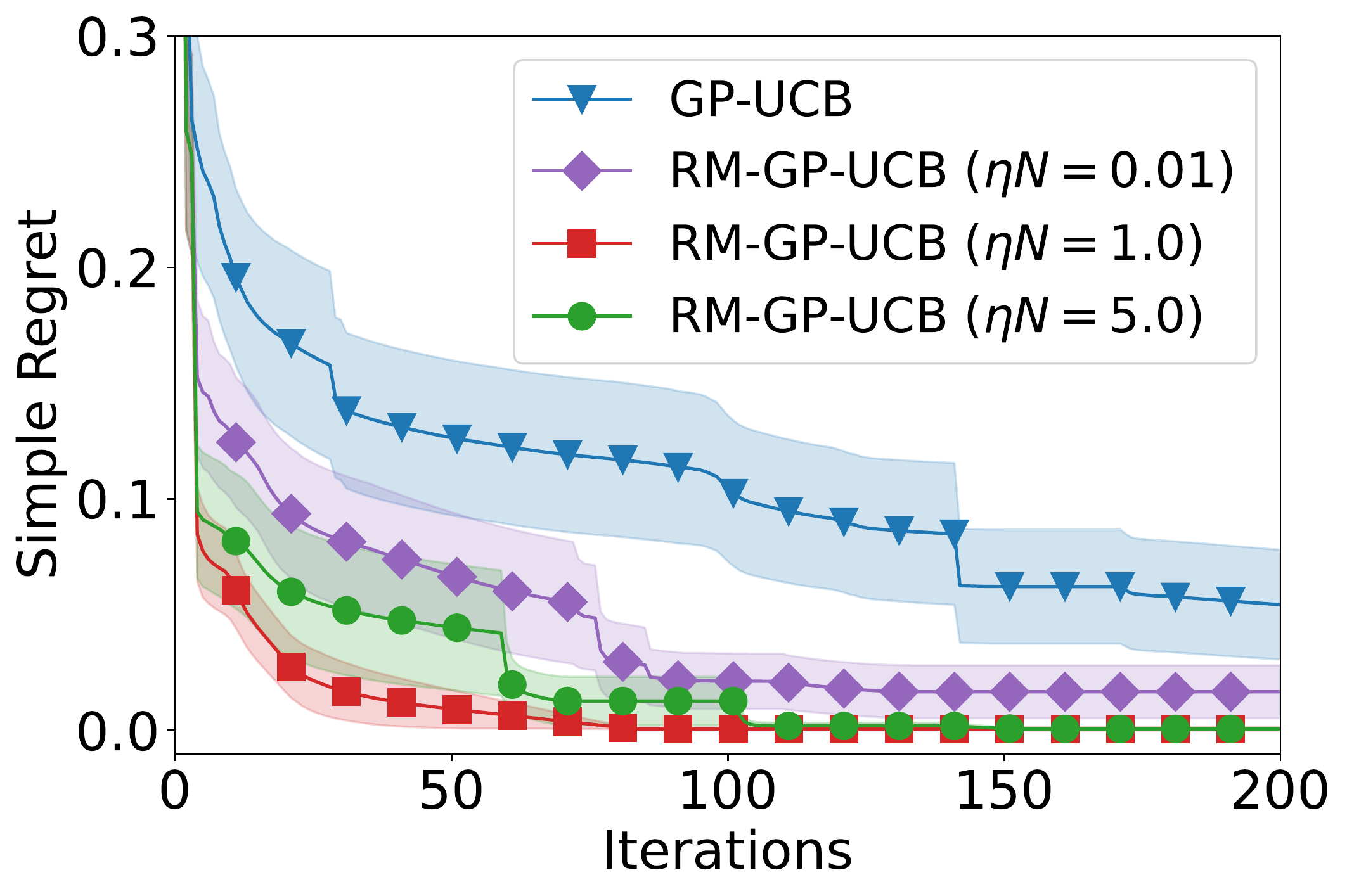}& \hspace{-6.2mm}
		\includegraphics[width=0.47\linewidth]{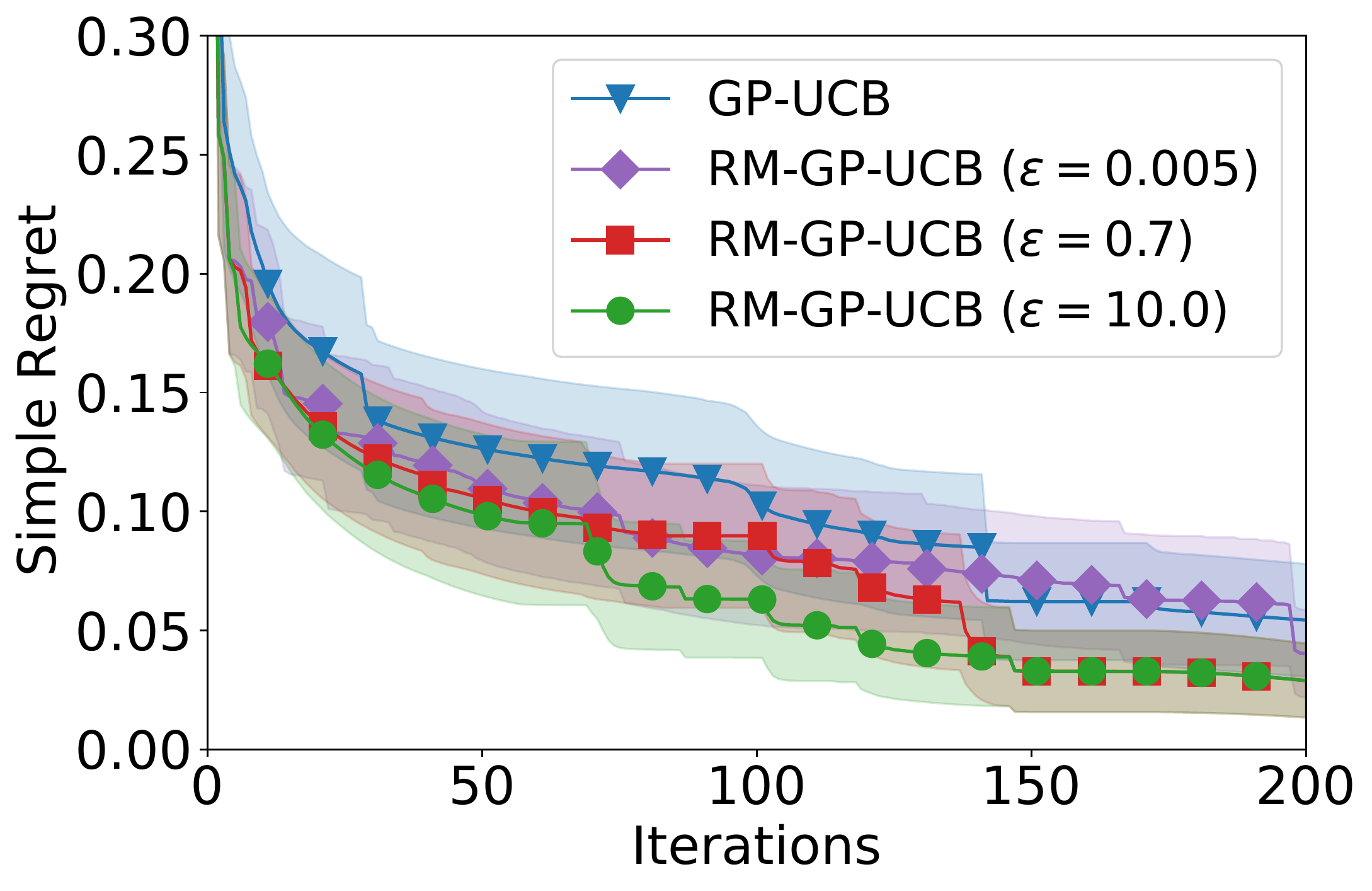}\\
		{(c)} & {(d)}\\
		\includegraphics[width=0.47\linewidth]{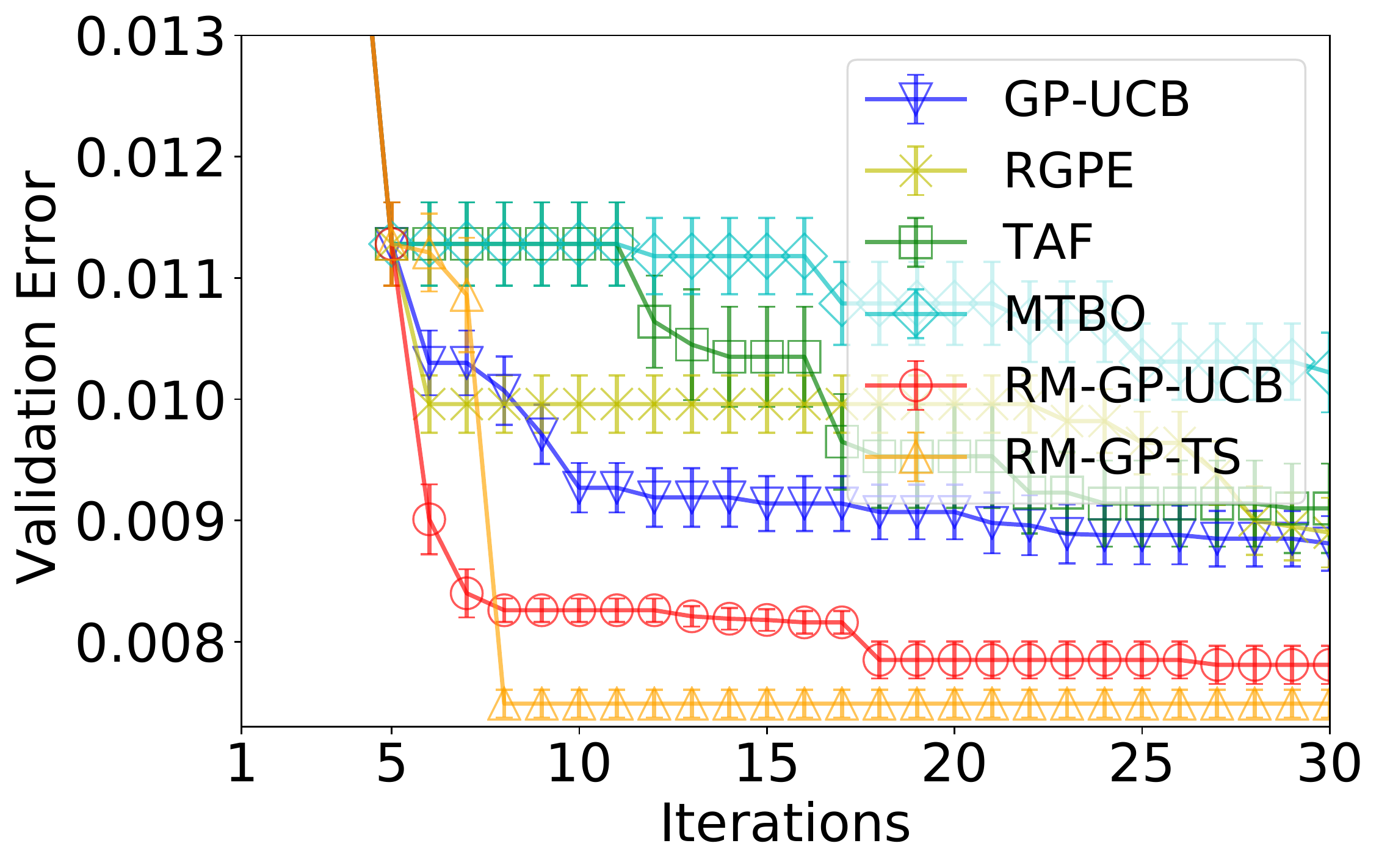} & \hspace{-5.3mm}
		\includegraphics[width=0.47\linewidth]{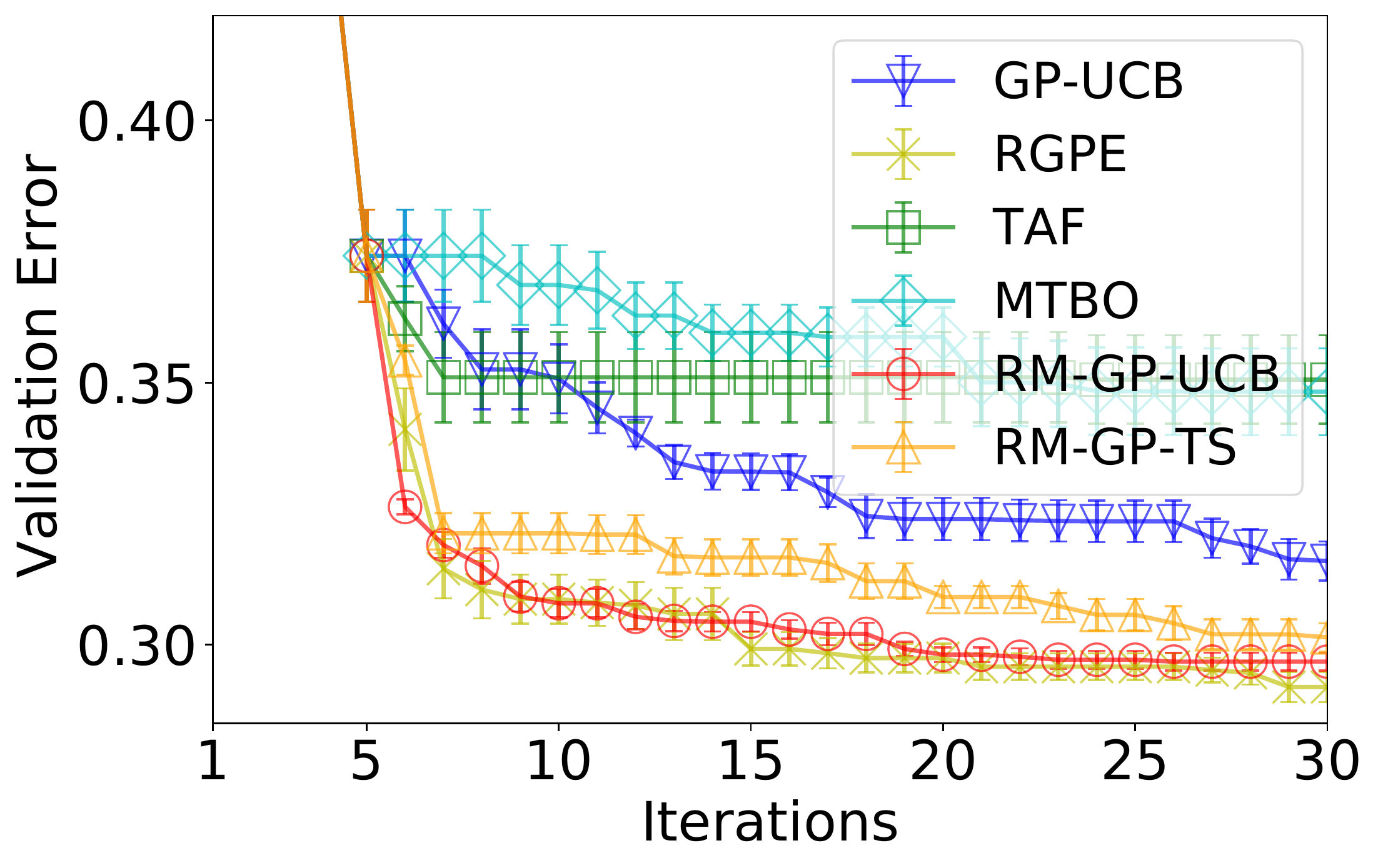}\\
		{(e)} & {(f)}\\
	\end{tabular}
	\caption{(a) The simple regret and (b) meta-weights optimized by RM-GP-UCB.
	The impact of (c) $\eta$ and (d) $\epsilon$. 
	Best validation error of CNN for (e) MNIST and (f) CIFAR-$10$.}
	\label{fig:synth_func_results}
\end{figure}

We also investigate the impact of $\eta$ and $\epsilon$.
Fig.~\ref{fig:synth_func_results}c shows the performances of different values of $\eta$, with fixed $\epsilon=0.7$ and $r=0.7$.
The figure demonstrates that an excessively small $\eta$ (purple curve) negatively impacts the performance, since RM-GP-UCB is unable to quickly reduce the weights of dissimilar meta-tasks (Fig.~\ref{fig:meta_weights_curves}a in Appendix~\ref{app:synth}).
Moreover, an overly large $\eta$ is also slightly detrimental (green curve) since it rapidly assigns a large weight to one of the two useful meta-tasks 
(Fig.~\ref{fig:meta_weights_curves}c in Appendix~\ref{app:synth}), thus failing to utilize the other useful meta-task.
Fig.~\ref{fig:synth_func_results}d illustrates the impact of $\epsilon$ when all function gaps are large: $d_i=8.0$ for all $i$.\footnote{We use $\eta=1/N$ and fix $r$ at a large value ($0.99$) so that the decaying rate of $\nu_t$ is purely decided by $\epsilon$.}
The figure shows that even when all meta-tasks are dissimilar, our adaptive selection of $\nu_t$
is able to diminish their negative impact and allow RM-GP-UCB to perform comparably to GP-UCB.
Furthermore, in this adverse scenario, a faster decline of the impact of the meta-tasks (i.e., faster decay of $\eta_t$ via larger $\epsilon$) leads to slightly better performance.

\subsection{Real-world Experiments}
\label{subsec:automl}
\textbf{Hyperparameter Tuning for Convolutional Neural Networks (CNNs).} 
We apply meta-BO to hyperparameter tuning of ML models with the previous tasks using other datasets as the meta-tasks.
We tune $3$ hyperparameters of CNNs using $4$ widely used image datasets: MNIST, SVHN, CIFAR-$10$ and CIFAR-$100$.
Specifically, in each experiment, one of the four datasets is selected to produce the target function $f$  
which maps a hyperparameter setting to a validation accuracy obtained using this dataset.
The meta-observations are generated from $3$ independent BO tasks (each with $50$ iterations) using the other $3$ datasets, 
i.e., $M=3$ and $N_i=50$ for $i=1,2,3$ in all $4$ experiments. 
The results for MNIST and CIFAR-$10$ are plotted in Figs.~\ref{fig:synth_func_results}e and~\ref{fig:synth_func_results}f while
the remaining results are shown in Appendix~\ref{app:auto_ml} (Fig.~\ref{fig:cnn_2}).
The results show that RM-GP-UCB is the only method that consistently performs well in all tasks, and that RM-GP-TS performs much better than RM-GP-UCB (and other methods) for MNIST, yet worse in the other tasks.
We have also adopted the 
Omniglot dataset~\citep{lake2015human} commonly used in meta-learning, for which RM-GP-UCB performs the best (Fig.~\ref{fig:omniglot}, Appendix~\ref{app:auto_ml}).

\textbf{Non-stationary Bayesian Optimization.}
Meta-BO can be naturally applied to non-stationary BO problems in which the unknown objective function evolves over time  
since the previous (outdated) observations can be treated as the meta-observations. 
We consider here automated ML for clinical diagnosis.
As the data from new patients becomes available regularly, clinicians often need to 
periodically 
update the dataset and
re-run hyperparameter optimization 
for the ML model used for clinical diagnosis.
This stimulates the question as to whether the previous hyperparameter tuning tasks using the outdated patients data can help accelerate the current task.
We consider the problem of diabetes prediction~\citep{smith1988using} with \emph{logistic regression} (LR) and tune $3$ LR hyperparameters.
We create $5$ progressively growing datasets (including the full dataset), 
treating (the hyperparameter tuning task using) the full dataset as the target task and the $4$ smaller datasets as the meta-tasks.
Specifically, the entire dataset consists of 768 data instances, among which 77 instances are set aside to measure the validation accuracy. 
The sizes of the 5 progressively growing training datasets (i.e., corresponding to the 4 meta-tasks and the target task, respectively) are 138, 276, 414, 552, and 691.
The results (Fig.~\ref{fig:cnn}a) show that RM-GP-TS outperforms all other methods in this task.
Moreover, we also compare the runtime of different methods in Fig.~\ref{fig:cnn}b: RM-GP-TS is significantly more efficient than all other methods, and the methods building separate GP surrogates for different tasks (i.e. RM-GP-UCB, RGPE and TAF) are more efficient than MTBO which includes all observations in a single GP (Sec.~\ref{sec:intro}).

\textbf{Hyperparameter Tuning for Support Vector Machines (SVMs).} 
We also tune the hyperparameters of SVMs using a tabular benchmark dataset~\citep{wistuba2015learning} which has also been adopted by RGPE~\citep{feurer2018scalable}.
The benchmark was constructed by evaluating a fixed grid of $288$ SVM hyperparameter configurations using $50$ \emph{diverse} datasets (i.e., containing many dissimilar tasks).
We follow the setting used by RGPE~\citep{feurer2018scalable}: In every trial, we fix one of the tasks as the target task, and the remaining $M=49$ tasks as the meta-tasks; for every meta-task $i$, we randomly select $N_i=50$ hyperparameter configurations 
as the meta-observations.
The results in Fig.~\ref{fig:cnn}c 
show that our RM-GP-UCB performs comparably to RGPE, 
outperforming the other methods; 
RM-GP-TS performs unsatisfactorily in this experiment with diverse tasks.
Of note, this experiment has the most favorable setting for PEM-BO~\citep{wang2018regret} because (a) PEM-BO has been shown to require a massive set of meta-observations ($\geq5000$)
to perform well~\citep{wang2018regret}, and this experiment has the largest number ($49\times50=2450$) of meta-observations among all experiments;
(b) the domain here is discrete, which is much easier for the application of PEM-BO.



\begin{figure}
	\centering
	\begin{tabular}{cc}
		\hspace{-4mm} \includegraphics[width=0.47\linewidth]{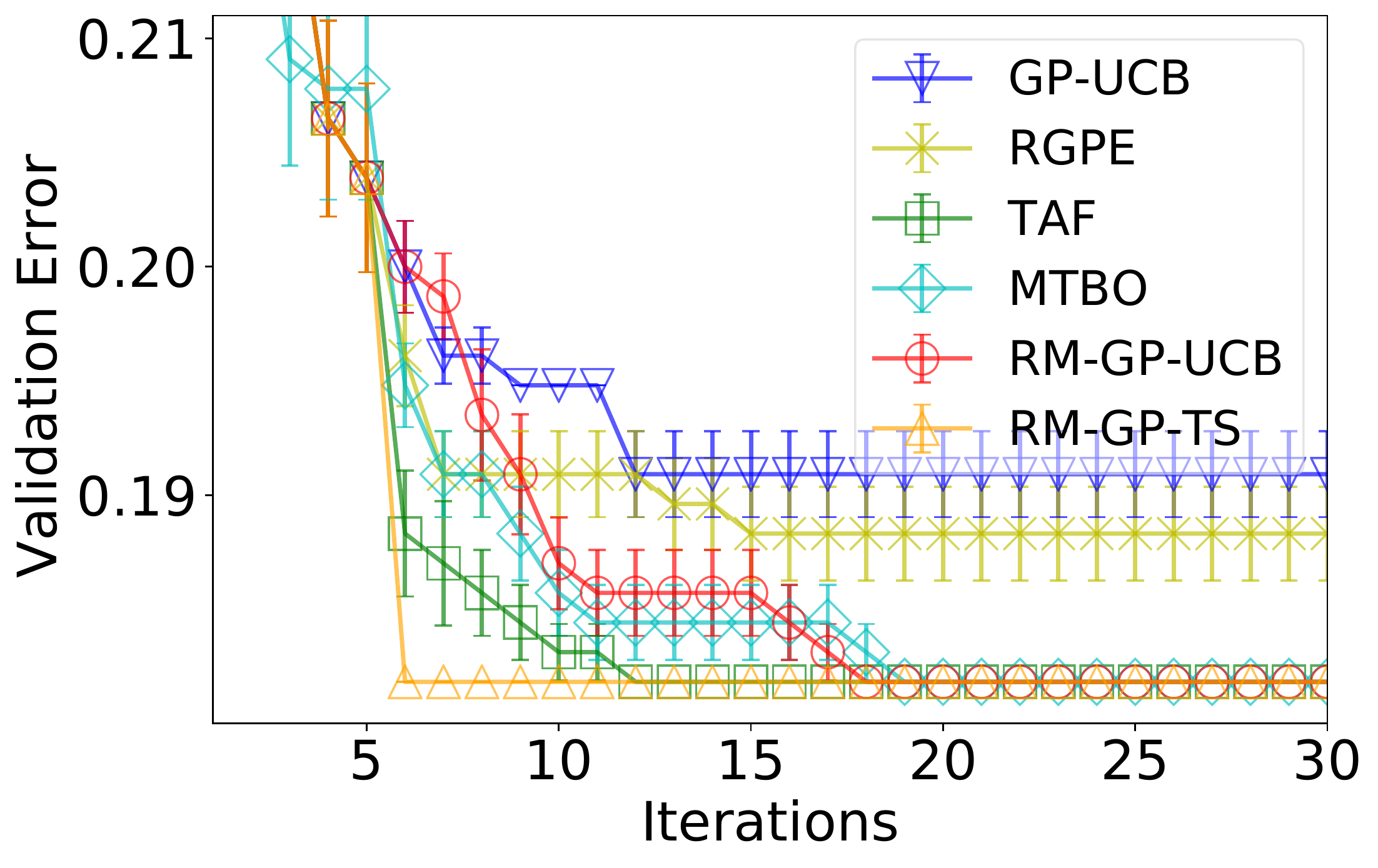} & \hspace{-4mm}
		\includegraphics[width=0.47\linewidth]{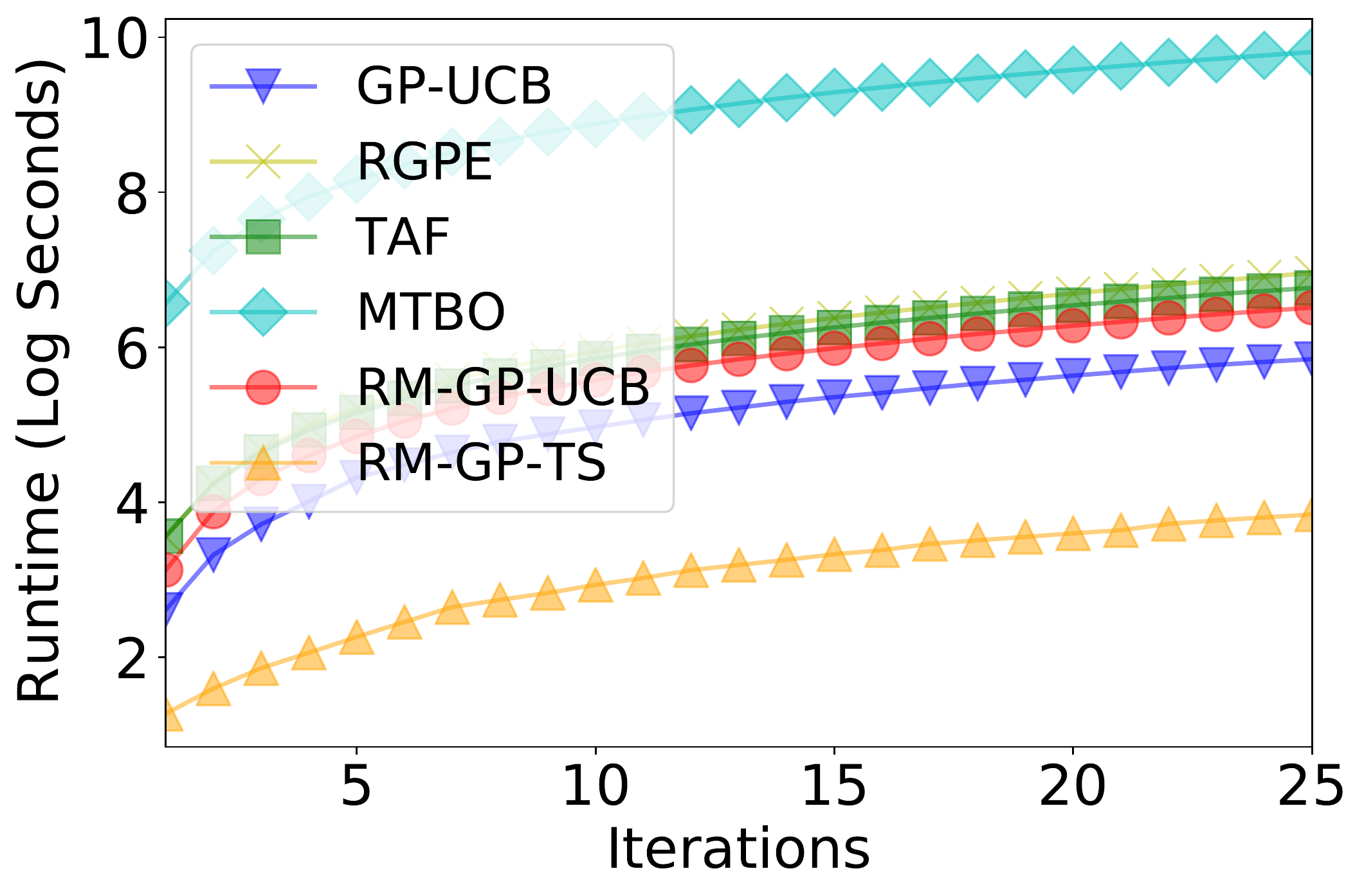}\\
		{(a)} & {(b)}\\
		\hspace{-4mm} \includegraphics[width=0.47\linewidth]{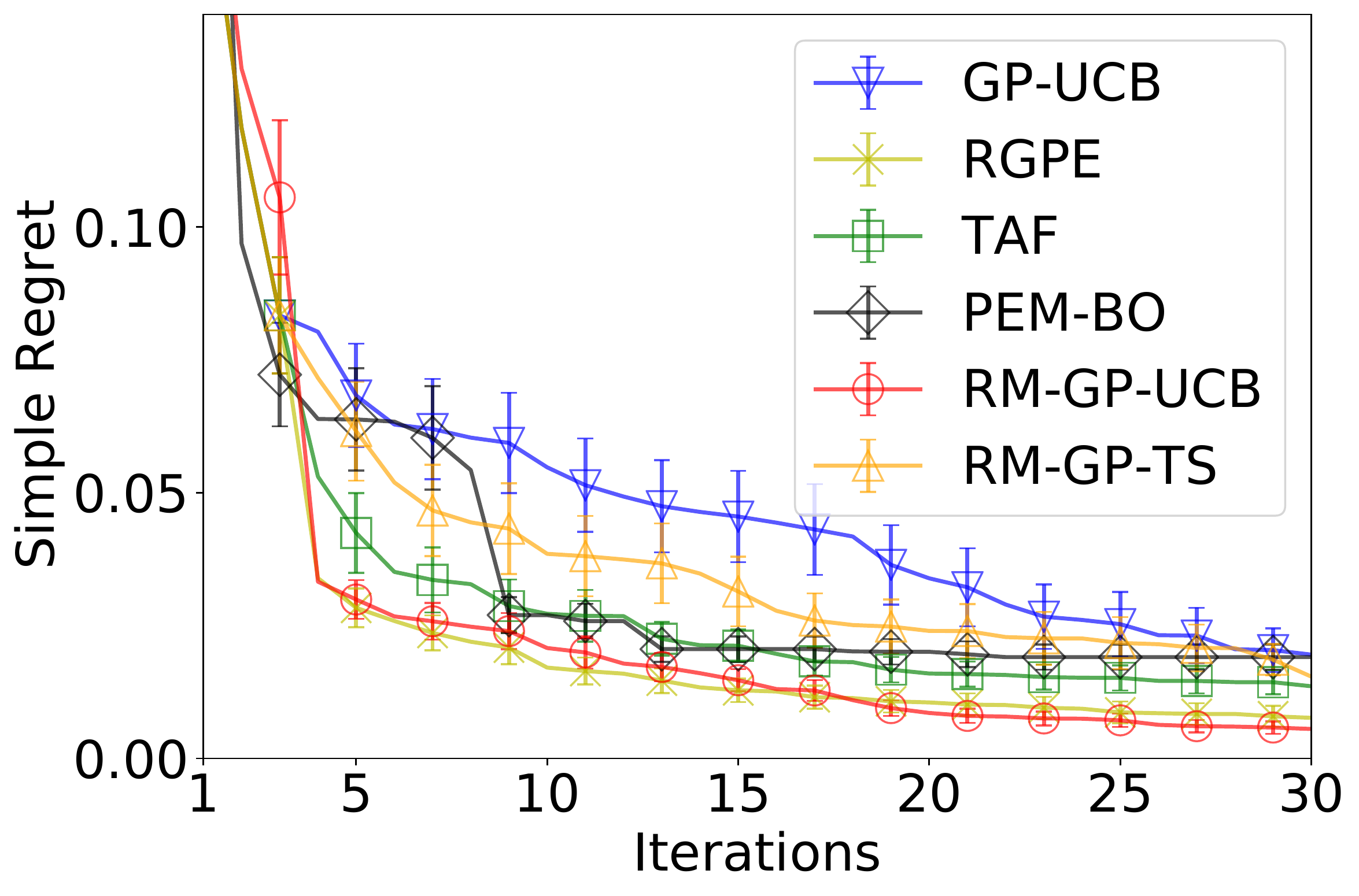}& \hspace{-4mm}
		\includegraphics[width=0.47\linewidth]{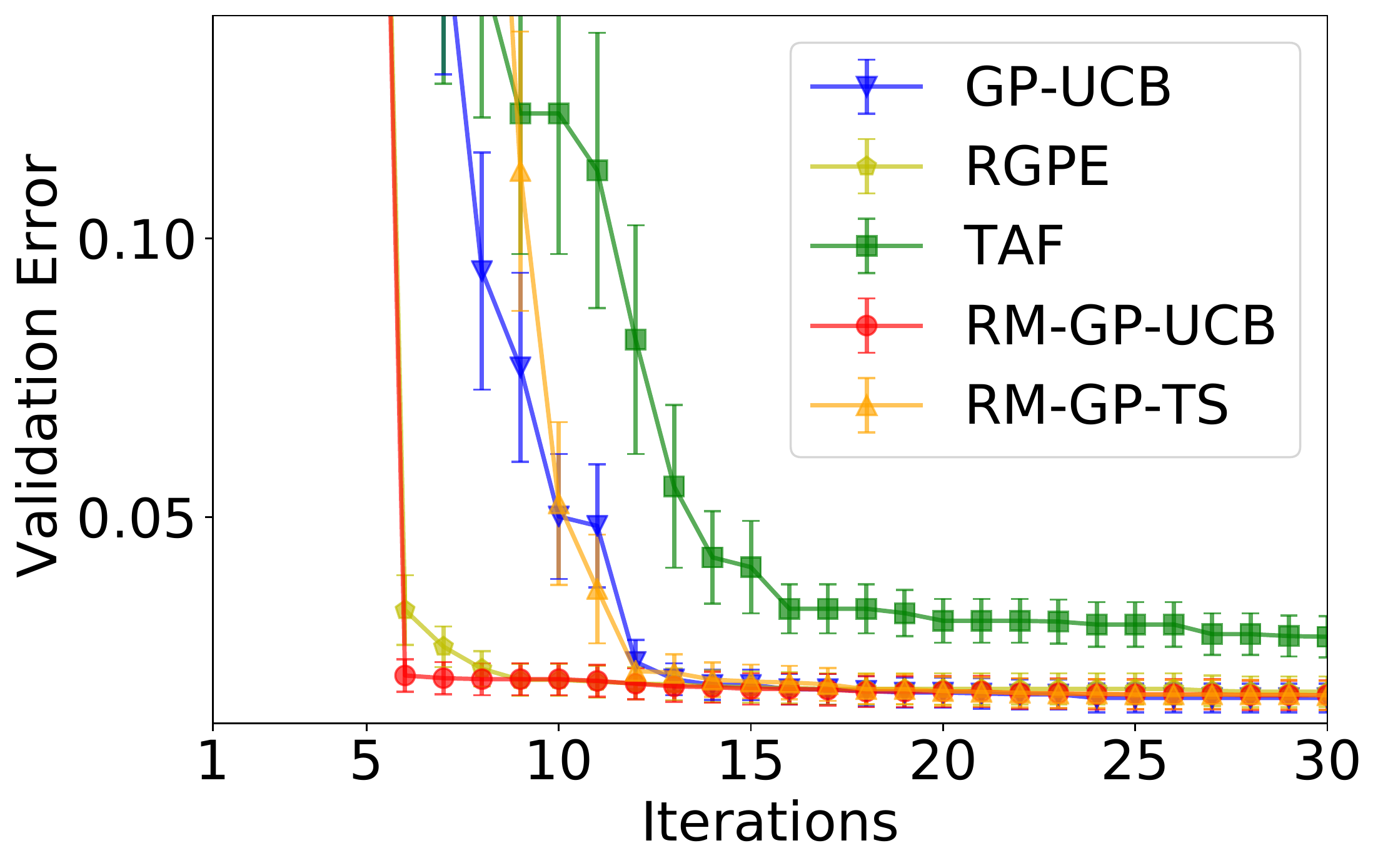}\\
		{(c)} & {(d)}\\
		\hspace{-4mm} \includegraphics[width=0.47\linewidth]{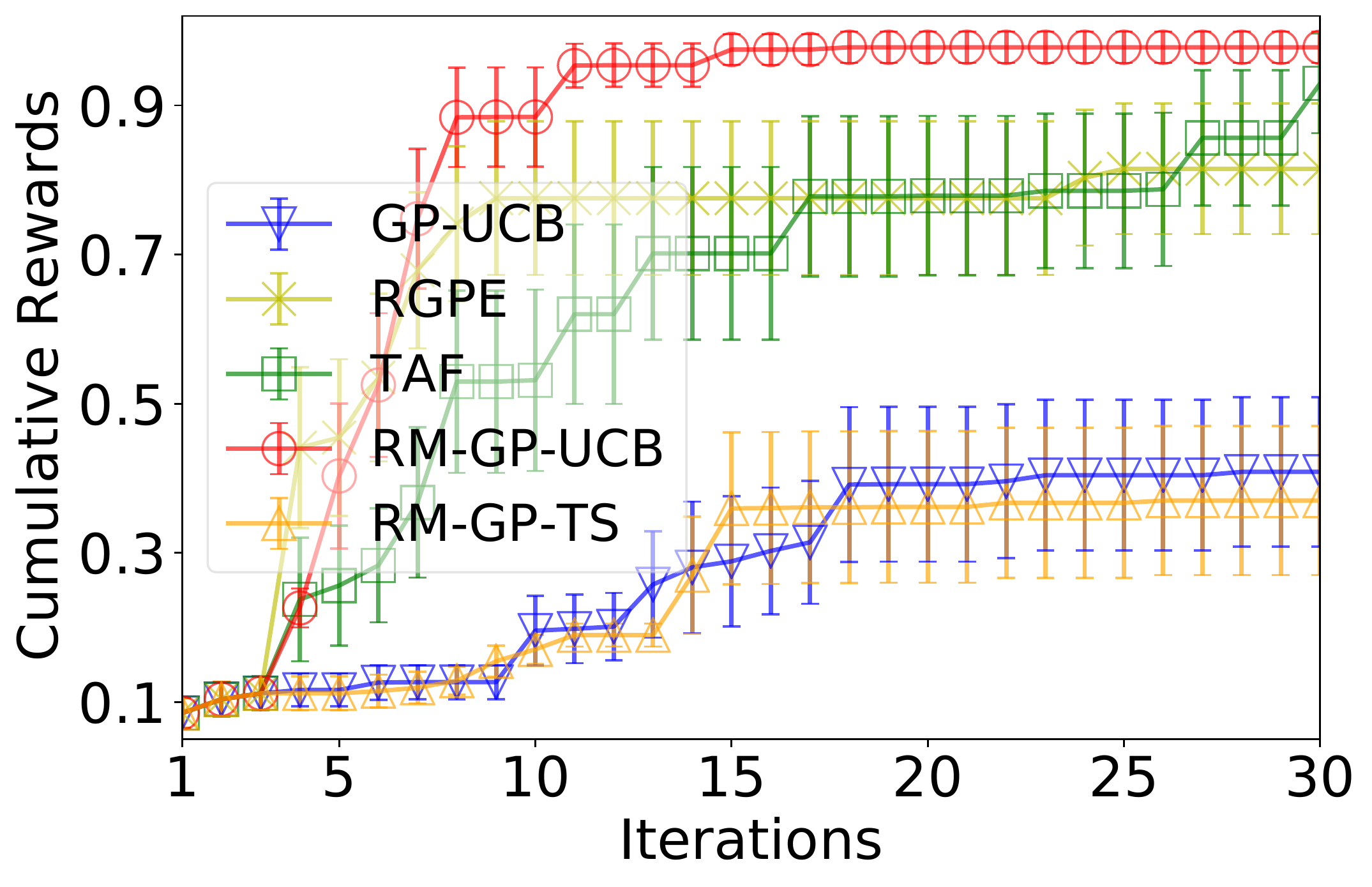} & \hspace{-4mm}
		\includegraphics[width=0.47\linewidth]{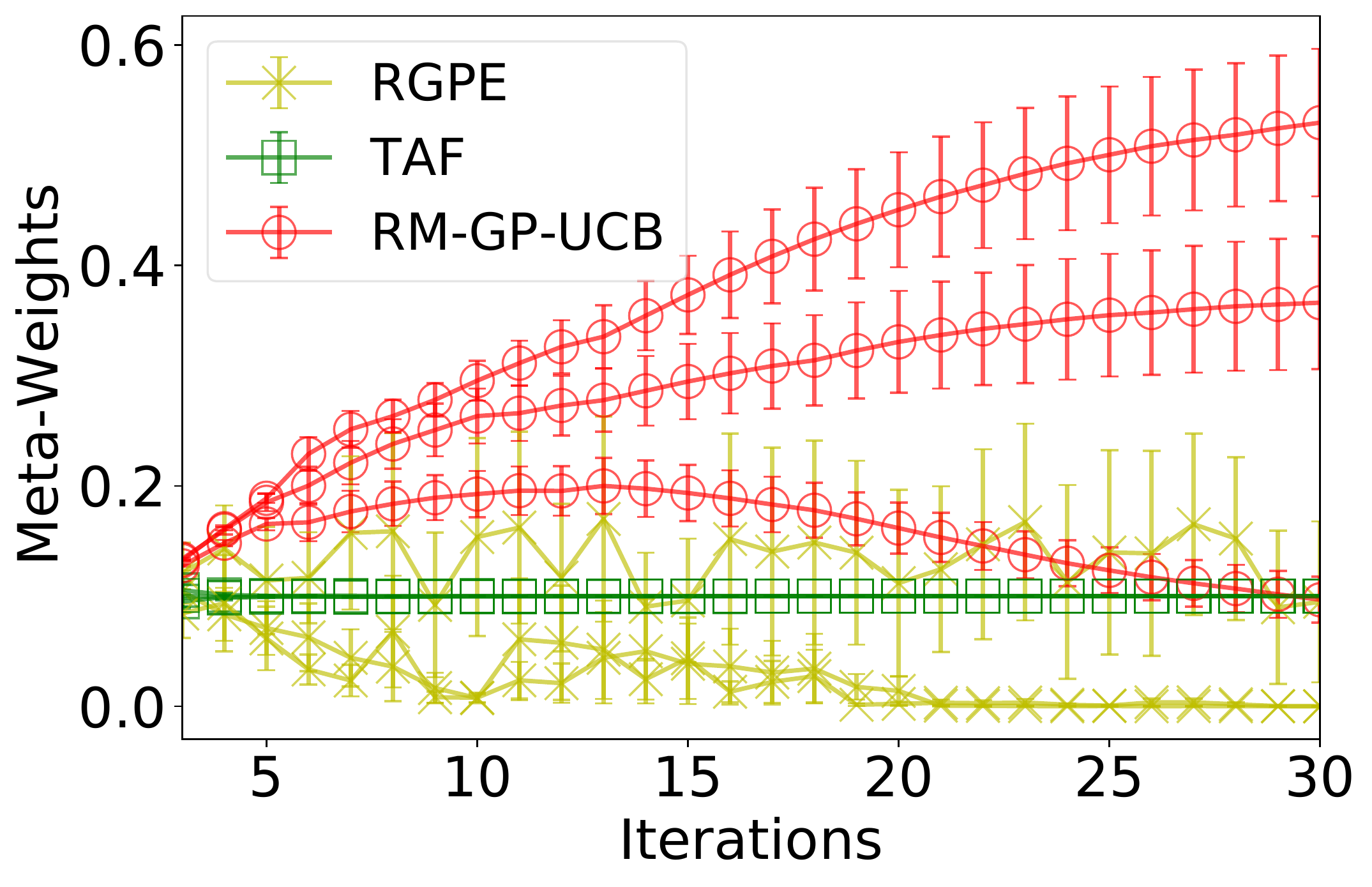}\\
		{(e)} & {(f)}
	\end{tabular} 
	\caption{
	(a) Best validation error of LR for diabetes diagnosis.
	(b) Runtime 
	in non-stationary BO experiment.
	(c) Simple regret on SVM benchmark.
	(d) Best validation errors for HAR.
	(e) Best cumulative rewards and (f) learned meta-weights for the $3$ similar meta-tasks for the RL experiment.}
	\label{fig:cnn}
\end{figure}

\textbf{Human Activity Recognition (HAR).}  
HAR using mobile devices has promising applications in various domains such as healthcare~\citep{reyes2013human}.
When optimizing the configurations (hyperparameters) of the activity prediction model (ML model) for a subject, the previous optimization tasks for other subjects might be helpful.
However, cross-subject transfer in HAR is challenging due to high \emph{individual variability}~\citep{soleimani2019cross},
which makes HAR suitable for evaluating the robustness of a meta-BO algorithm against dissimilar meta-tasks.
We use the data collected through mobile phones from $30$ subjects performing $6$ activities and 
use \emph{support vector machines} (SVM) for activity prediction. 
Every task corresponds to tuning $2$ SVM hyperparameters for a subject.
We run a separate BO ($30$ iterations) for each of the $21$ subjects to generate the meta-observations ($M=21$, $N_i=30$ for $i=1,\ldots,21$)  
and use the other $9$ subjects for validation.
The results are shown in Fig.~\ref{fig:cnn}d (averaged over the $9$ subjects, each further averaged over $5$ random initializations),
in which RM-GP-UCB delivers the best performance, followed by RGPE; RM-GP-TS again fails to perform effectively, suggesting that it is less robust against the individual variability in HAR.
\textbf{Policy Search for Reinforcement Learning (RL).}
When optimizing the RL policy of an agent in an environment, the agent's experience in other related environments may help to make learning more efficient~\citep{duan2016rl, wang1611learning}.
We apply meta-BO to policy search in RL to maximize the cumulative rewards in an episode, using the Cart-Pole environment from OpenAI Gym~\citep{brockman2016openai} with $8$ policy parameters.
We simulate different environments by setting the agent to different initial states. 
In particular, we choose $M=10$ different initial states, among which the majority (i.e., $7$) are randomly generated (i.e., dissimilar meta-tasks) and 
the other $3$ are designed to be close to the initial state of the target task so that they are similar to the target task.
An independent BO task with $50$ iterations is run for every initial state, i.e., $N_i=50$ for $i=1,\ldots,10$.
Figs.~\ref{fig:cnn}e and~\ref{fig:cnn}f plot the (normalized) cumulative rewards of different algorithms and their learned meta-weights for the $3$ similar meta-tasks.
The results show that RM-GP-UCB achieves the best performance (Fig.~\ref{fig:cnn}e), 
and it is more effective than RGPE and TAF at identifying the $3$ similar meta-tasks (Fig.~\ref{fig:cnn}f).
RGPE and TAF fail to correctly identify similar meta-tasks because they learn the meta-weights based on how accurately each GP surrogate predicts the \emph{pairwise ranking of the target observations} (more details in Sec.~\ref{sec:related_works}). However, in the Cart-Pole environment, many target observations have equal values, which confuses the pairwise ranking and makes the learned meta-weights unreliable.
RM-GP-TS again only performs comparably with standard GP-UCB (Fig.~\ref{fig:cnn}e).



\subsection{Experimental Discussion}
\label{sec:experiment_discussion}
In most experimental results (Figs.~\ref{fig:synth_func_results} and~\ref{fig:cnn}), 
the performance advantage of RM-GP-UCB is most evident at the initial stage.
This is likely to corroborate our theoretical insights that 
the meta-tasks can help improve the convergence of RM-GP-UCB at the initial stage by reducing the degree of exploration (Sec.~\ref{subsec:theory:rm_gp_ucb}).
A potential limitation of our online meta-weight optimization (Sec.~\ref{sec:online_regret_minimization}) is that it does not account for the scenario where the meta-functions are shifted or scaled versions of the target function.
However, note that in some scenarios, the scale of the meta-functions is informative about task similarity and thus should not be removed. For example, in our clinical diagnosis (i.e., non-stationary BO) experiment, the more recently completed meta-tasks (with larger training set, smaller validation errors, and thus smaller function gaps) are expected to be more similar to the target task.
Furthermore, as demonstrated by the green curve in Fig.~\ref{fig:synth_func_results}a, in some cases, even though the meta-weights are not optimized, RM-GP-UCB still performs favorably. This implies its robustness against mis-specification of the meta-weights.

RM-GP-UCB is the only method that consistently outperforms standard GP-UCB in \emph{all} experiments (Figs.~\ref{fig:synth_func_results} and~\ref{fig:cnn}), whereas other methods perform either comparably with or worse than GP-UCB in some experiments (e.g., RGPE in Figs.~\ref{fig:synth_func_results}e and~\ref{fig:cnn}a, TAF in Figs.~\ref{fig:synth_func_results}e,~\ref{fig:synth_func_results}f and~\ref{fig:cnn}d).
This might be attributed to RM-GP-UCB's theoretically guaranteed robustness against dissimilar meta-tasks (Sec.~\ref{sec:theoretical_analysis}) and its ability to diminish their impact in a principled way (Sec.~\ref{sec:online_regret_minimization}).
In particular, RM-GP-UCB performs significantly better than RM-GP-TS in those experiments with a large number of dissimilar meta-tasks (Figs.~\ref{fig:cnn}c-e), which may be explained by RM-GP-UCB's better theoretically guaranteed robustness against dissimilar meta-tasks than RM-GP-TS (Sec.~\ref{subsec:theory:ts}).
However, Figs.~\ref{fig:synth_func_results}e-f and Fig.~\ref{fig:cnn}a show that RM-GP-TS performs competitively in some experiments with more favorable settings (i.e., less dissimilar meta-tasks), which might result from the repeatedly observed empirical effectiveness of TS-based algorithms~\citep{chapelle2011empirical,russo2017tutorial}.
Moreover, the computational efficiency of RM-GP-TS is markedly superior to other methods (Fig.~\ref{fig:cnn}b).
These theoretical and empirical comparisons between RM-GP-UCB and RM-GP-TS may provide useful insights for other meta-BO algorithms and potentially for a broader range of problems (e.g., meta-learning for multi-armed bandits and RL) in terms of the relative strengths and weaknesses of UCB- and TS-based algorithms.


\section{Related Works}
\label{sec:related_works}
Some previous works on meta-BO build a joint GP surrogate using all previous and current observations, 
and represent task similarity through meta-features~\citep{bardenet2013collaborative,schilling2016scalable,yogatama2014efficient}.
However, these algorithms suffer from the requirement of handcrafted meta-features, which is avoided in other works that learn task similarity from the observations~\citep{swersky2013multi,shilton2017regret}. For example, multitask BO~\citep{swersky2013multi} uses a multitask GP as a surrogate and models each task as an output of the GP. 
These works include all previous and current observations in a single GP surrogate and are thus limited by the scalability of GPs.
There have also been other empirical works which replace GP by Bayesian linear regression for scalability~\citep{perrone2018scalable},
tackle sequentially arriving tasks~\citep{golovin2017google,poloczek2016warm}, learn a set of good initializations~\citep{feurer2015initializing,wistuba2015sequential}, learn a reduced search space for BO from previous tasks~\citep{perrone2019learning}, handle the issue of different function scales using Gaussian Copulas~\citep{salinas2020quantile}, 
learn the task similarities through the distance between the distributions of the optima from different tasks~\citep{ramachandran2018information},
or use the meta-observations to learn the entire acquisition function through RL~\citep{volpp2020meta}.
\citet{wang2018regret} have learned the GP prior from previous tasks and given theoretical guarantees. 
However, they have shown in both theory and practice that a large training set of meta-observations ($\geq 5000$) is required for their method to work well, while we focus on the more practical setting of meta-BO where the number of available meta-observations may be small.
We have also verified that our algorithm outperforms the method from~\citet{wang2018regret} in the experiment that is most favorable for their method among all our experiments (more details in the third paragraph of Sec.~\ref{subsec:automl}).
Meta-BO is also related to the works on multi-fidelity BO~\citep{dai2019bayesian,kandasamy2016gaussian,poloczek2017multi,wu2020practical,zhang2020bayesian,zhang2017information}, since the previous tasks can be viewed as low-fidelity functions which can approximate the target function and are cheap to query. However, multi-fidelity BO allows querying the low-fidelity functions during the BO process, whereas meta-BO algorithms can only query the target function, i.e., the highest-fidelity function.
Moreover, meta-BO is also related to the previous works on BO which involve multiple agents (i.e., analogous to multiple tasks in meta-BO), such as federated BO~\citep{dai2020federated,dai2021differentially,sim2021collaborative} or BO methods based on game-theoretical approaches~\citep{dai2020r2,sessa2019no}.

Some works have aimed to improve the scalability of GP-based meta-BO algorithms by building a separate GP surrogate for each task~\citep{feurer2018scalable,wistuba2016two,wistuba2018scalable}.
\citet{wistuba2016two} use a weighted combination of the posterior mean of each individual GP surrogate as the joint posterior mean 
while the posterior variance is derived using only the target observations.
RGPE~\citep{feurer2018scalable} has extended the work of~\citet{wistuba2016two} by estimating the joint objective function as a weighted combination of individual objective functions,
such that the resulting joint surrogate remains a GP (unlike~\citet{wistuba2016two}) and can thus be plugged into standard BO algorithms.
Note that RGPE differs from our RM-GP-UCB algorithm in that RGPE uses a weighted combination of individual GP surrogates to derive a joint GP surrogate, whereas our RM-GP-UCB leverage a weighted combination of individual acquisition functions.
\citet{wistuba2018scalable} have proposed TAF, which also uses a weighted combination of the acquisition functions (i.e., expected improvement) from the individual tasks for query selection.
In these works, the weight of a previous task is heuristically chosen to be proportional to the accuracy of the \emph{pairwise ranking of the target observations} produced by 
either (a) the posterior mean of the GP surrogate of the previous task (TAF)~\citep{wistuba2018scalable} or 
(b) functions sampled from the posterior GP surrogate (RGPE)~\citep{feurer2018scalable}.

\section{Conclusion}
\label{sec:conclusion}
We have introduced
RM-GP-UCB and RM-GP-TS, both of which are asymptotically no-regret even if all meta-tasks are dissimilar to the target task.
We leverage the theoretical results to learn the task similarities in a principled way via online learning.
Theoretical and empirical comparisons show that RM-GP-UCB is more robust against dissimilar tasks, whereas RM-GP-TS performs effectively in more favorable cases and is more computationally efficient.

\begin{acknowledgements} 
This research/project is supported by A*STAR under its RIE$2020$ Advanced Manufacturing and Engineering (AME) Industry Alignment Fund – Pre Positioning (IAF-PP) (Award A$19$E$4$a$0101$) and by the Singapore Ministry of Education Academic Research Fund Tier $1$. This research is part of the programme DesCartes and is supported by the National Research Foundation, Prime Minister’s Office, Singapore under its Campus for Research Excellence and Technological Enterprise (CREATE) programme.

\end{acknowledgements}

\newpage


\bibliography{dai_226}

\begin{thebibliography}{52}
\providecommand{\natexlab}[1]{#1}
\providecommand{\url}[1]{\texttt{#1}}
\expandafter\ifx\csname urlstyle\endcsname\relax
  \providecommand{\doi}[1]{doi: #1}\else
  \providecommand{\doi}{doi: \begingroup \urlstyle{rm}\Url}\fi

\bibitem[Bardenet et~al.(2013)Bardenet, Brendel, K{\'e}gl, and
  Sebag]{bardenet2013collaborative}
R{\'e}mi Bardenet, M{\'a}ty{\'a}s Brendel, Bal{\'a}zs K{\'e}gl, and Michele
  Sebag.
\newblock Collaborative hyperparameter tuning.
\newblock In \emph{Proc. {ICML}}, pages 199--207, 2013.

\bibitem[Beyer and Sendhoff(2007)]{beyer2007robust}
Hans-Georg Beyer and Bernhard Sendhoff.
\newblock Robust optimization--a comprehensive survey.
\newblock \emph{Computer methods in applied mechanics and engineering},
  196\penalty0 (33-34):\penalty0 3190--3218, 2007.

\bibitem[Brockman et~al.(2016)Brockman, Cheung, Pettersson, Schneider,
  Schulman, Tang, and Zaremba]{brockman2016openai}
Greg Brockman, Vicki Cheung, Ludwig Pettersson, Jonas Schneider, John Schulman,
  Jie Tang, and Wojciech Zaremba.
\newblock {OpenAI} {Gym}.
\newblock {arXiv:1606.01540}, 2016.

\bibitem[Bubeck(2011)]{bubeck2011introduction}
S{\'e}bastien Bubeck.
\newblock Introduction to online optimization.
\newblock Lecture notes, 2011.

\bibitem[Chapelle and Li(2011)]{chapelle2011empirical}
Olivier Chapelle and Lihong Li.
\newblock An empirical evaluation of {Thompson} sampling.
\newblock In \emph{Proc. {NeurIPS}}, volume~24, pages 2249--2257. Citeseer,
  2011.

\bibitem[Chowdhury and Gopalan(2017)]{chowdhury2017kernelized}
Sayak~Ray Chowdhury and Aditya Gopalan.
\newblock On kernelized multi-armed bandits.
\newblock In \emph{Proc. {ICML}}, pages 844--853, 2017.

\bibitem[Dai et~al.(2019)Dai, Yu, Low, and Jaillet]{dai2019bayesian}
Zhongxiang Dai, Haibin Yu, Bryan Kian~Hsiang Low, and Patrick Jaillet.
\newblock Bayesian optimization meets {Bayesian} optimal stopping.
\newblock In \emph{Proc. {ICML}}, pages 1496--1506, 2019.

\bibitem[Dai et~al.(2020{\natexlab{a}})Dai, Chen, Low, Jaillet, and
  Ho]{dai2020r2}
Zhongxiang Dai, Yizhou Chen, Bryan Kian~Hsiang Low, Patrick Jaillet, and
  Teck-Hua Ho.
\newblock {R2-B2}: Recursive reasoning-based {Bayesian} optimization for
  no-regret learning in games.
\newblock In \emph{Proc. {ICML}}, 2020{\natexlab{a}}.

\bibitem[Dai et~al.(2020{\natexlab{b}})Dai, Low, and Jaillet]{dai2020federated}
Zhongxiang Dai, Kian~Hsiang Low, and Patrick Jaillet.
\newblock Federated {Bayesian} optimization via {Thompson} sampling.
\newblock In \emph{Proc. {NeurIPS}}, 2020{\natexlab{b}}.

\bibitem[Dai et~al.(2021)Dai, Low, and Jaillet]{dai2021differentially}
Zhongxiang Dai, Bryan Kian~Hsiang Low, and Patrick Jaillet.
\newblock Differentially private federated {Bayesian} optimization with
  distributed exploration.
\newblock In \emph{Proc. {NeurIPS}}, volume~34, 2021.

\bibitem[Duan et~al.(2016)Duan, Schulman, Chen, Bartlett, Sutskever, and
  Abbeel]{duan2016rl}
Yan Duan, John Schulman, Xi~Chen, Peter~L Bartlett, Ilya Sutskever, and Pieter
  Abbeel.
\newblock {RL}$^2$: Fast reinforcement learning via slow reinforcement
  learning.
\newblock {arXiv}:1611.02779, 2016.

\bibitem[Feurer et~al.(2015)Feurer, Springenberg, and
  Hutter]{feurer2015initializing}
Matthias Feurer, Jost~Tobias Springenberg, and Frank Hutter.
\newblock Initializing {Bayesian} hyperparameter optimization via
  meta-learning.
\newblock In \emph{Proc. {AAAI}}, pages 1128--1135, 2015.

\bibitem[Feurer et~al.(2018)Feurer, Letham, and Bakshy]{feurer2018scalable}
Matthias Feurer, Benjamin Letham, and Eytan Bakshy.
\newblock Scalable meta-learning for {Bayesian} optimization using
  ranking-weighted {Gaussian} process ensembles.
\newblock In \emph{Proc. {ICML} Workshop on Automatic Machine Learning}, 2018.

\bibitem[Finn et~al.(2017)Finn, Abbeel, and Levine]{finn2017model}
Chelsea Finn, Pieter Abbeel, and Sergey Levine.
\newblock Model-agnostic meta-learning for fast adaptation of deep networks.
\newblock In \emph{Proc. {ICML}}, pages 1126--1135, 2017.

\bibitem[Golovin et~al.(2017)Golovin, Solnik, Moitra, Kochanski, Karro, and
  Sculley]{golovin2017google}
Daniel Golovin, Benjamin Solnik, Subhodeep Moitra, Greg Kochanski, John Karro,
  and D~Sculley.
\newblock {Google Vizier}: A service for black-box optimization.
\newblock In \emph{Proc. {ACM SIGKDD}}, pages 1487--1495, 2017.

\bibitem[Kandasamy et~al.(2016)Kandasamy, Dasarathy, Oliva, Schneider, and
  P{\'o}czos]{kandasamy2016gaussian}
Kirthevasan Kandasamy, Gautam Dasarathy, Junier~B Oliva, Jeff Schneider, and
  Barnab{\'a}s P{\'o}czos.
\newblock {Gaussian} process bandit optimisation with multi-fidelity
  evaluations.
\newblock In \emph{Proc. {NeurIPS}}, pages 992--1000, 2016.

\bibitem[Lake et~al.(2015)Lake, Salakhutdinov, and Tenenbaum]{lake2015human}
Brenden~M Lake, Ruslan Salakhutdinov, and Joshua~B Tenenbaum.
\newblock Human-level concept learning through probabilistic program induction.
\newblock \emph{Science}, 350\penalty0 (6266):\penalty0 1332--1338, 2015.

\bibitem[Pang et~al.(2018)Pang, Dong, Wu, and Hospedales]{pang2018meta}
Kunkun Pang, Mingzhi Dong, Yang Wu, and Timothy Hospedales.
\newblock Meta-learning transferable active learning policies by deep
  reinforcement learning.
\newblock {arXiv}:1806.04798, 2018.

\bibitem[Perrone et~al.(2018)Perrone, Jenatton, Seeger, and
  Archambeau]{perrone2018scalable}
Valerio Perrone, Rodolphe Jenatton, Matthias~W Seeger, and Cedric Archambeau.
\newblock Scalable hyperparameter transfer learning.
\newblock In \emph{Proc. {NIPS}}, pages 6845--6855, 2018.

\bibitem[Perrone et~al.(2019)Perrone, Shen, Seeger, Archambeau, and
  Jenatton]{perrone2019learning}
Valerio Perrone, Huibin Shen, Matthias~W Seeger, C{\'e}dric Archambeau, and
  Rodolphe Jenatton.
\newblock Learning search spaces for bayesian optimization: Another view of
  hyperparameter transfer learning.
\newblock In \emph{Proc. {NeurIPS}}, pages 12771--12781, 2019.

\bibitem[Poloczek et~al.(2016)Poloczek, Wang, and Frazier]{poloczek2016warm}
Matthias Poloczek, Jialei Wang, and Peter~I. Frazier.
\newblock Warm starting {Bayesian} optimization.
\newblock In \emph{Proc. {WSC}}, pages 770--781, 2016.

\bibitem[Poloczek et~al.(2017)Poloczek, Wang, and Frazier]{poloczek2017multi}
Matthias Poloczek, Jialei Wang, and Peter Frazier.
\newblock Multi-information source optimization.
\newblock In \emph{Proc. {NeurIPS}}, pages 4288--4298, 2017.

\bibitem[Rahimi and Recht(2008)]{rahimi2008random}
Ali Rahimi and Benjamin Recht.
\newblock Random features for large-scale kernel machines.
\newblock In \emph{Proc. {NeurIPS}}, pages 1177--1184, 2008.

\bibitem[Ramachandran et~al.(2018)Ramachandran, Gupta, Rana, and
  Venkatesh]{ramachandran2018information}
Anil Ramachandran, Sunil Gupta, Santu Rana, and Svetha Venkatesh.
\newblock Information-theoretic transfer learning framework for {Bayesian}
  optimisation.
\newblock In \emph{Proc. {ECML/PKDD}}, pages 827--842. Springer, 2018.

\bibitem[Rasmussen and Williams(2006)]{rasmussen2004gaussian}
C.~E. Rasmussen and C.~K.~I. Williams.
\newblock \emph{{Gaussian Processes} for {Machine Learning}}.
\newblock MIT Press, 2006.

\bibitem[Reyes-Ortiz et~al.(2013)Reyes-Ortiz, Ghio, Parra, Anguita, Cabestany,
  and Catala]{reyes2013human}
Jorge~Luis Reyes-Ortiz, Alessandro Ghio, Xavier Parra, Davide Anguita, Joan
  Cabestany, and Andreu Catala.
\newblock Human activity and motion disorder recognition: Towards smarter
  interactive cognitive environments.
\newblock In \emph{Proc. {ESANN}}, 2013.

\bibitem[Russo et~al.(2017)Russo, Van~Roy, Kazerouni, Osband, and
  Wen]{russo2017tutorial}
Daniel Russo, Benjamin Van~Roy, Abbas Kazerouni, Ian Osband, and Zheng Wen.
\newblock A tutorial on {Thompson} sampling.
\newblock arxiv:1707.02038, 2017.

\bibitem[Salinas et~al.(2020)Salinas, Shen, and Perrone]{salinas2020quantile}
David Salinas, Huibin Shen, and Valerio Perrone.
\newblock A quantile-based approach for hyperparameter transfer learning.
\newblock In \emph{Proc. {ICML}}, pages 8438--8448. PMLR, 2020.

\bibitem[Schilling et~al.(2016)Schilling, Wistuba, and
  Schmidt-Thieme]{schilling2016scalable}
Nicolas Schilling, Martin Wistuba, and Lars Schmidt-Thieme.
\newblock Scalable hyperparameter optimization with products of {Gaussian}
  process experts.
\newblock In \emph{Proc. {ECML/PKDD}}, pages 33--48, 2016.

\bibitem[Sessa et~al.(2019)Sessa, Bogunovic, Kamgarpour, and
  Krause]{sessa2019no}
Pier~Giuseppe Sessa, Ilija Bogunovic, Maryam Kamgarpour, and Andreas Krause.
\newblock No-regret learning in unknown games with correlated payoffs.
\newblock In \emph{Proc. {NeurIPS}}, 2019.

\bibitem[Shahriari et~al.(2016)Shahriari, Swersky, Wang, Adams, and {de
  Freitas}]{shahriari2016taking}
Bobak Shahriari, Kevin Swersky, Ziyu Wang, Ryan~P. Adams, and Nando {de
  Freitas}.
\newblock Taking the human out of the loop: A review of {Bayesian}
  optimization.
\newblock \emph{Proceedings of the IEEE}, 104\penalty0 (1):\penalty0 148--175,
  2016.

\bibitem[Shilton et~al.(2017)Shilton, Gupta, Rana, and
  Venkatesh]{shilton2017regret}
Alistair Shilton, Sunil Gupta, Santu Rana, and Svetha Venkatesh.
\newblock Regret bounds for transfer learning in {Bayesian} optimisation.
\newblock In \emph{Proc. {AISTATS}}, pages 1--9, 2017.

\bibitem[Sim et~al.(2021)Sim, Zhang, Low, and Jaillet]{sim2021collaborative}
Rachael Hwee~Ling Sim, Yehong Zhang, Bryan Kian~Hsiang Low, and Patrick
  Jaillet.
\newblock Collaborative {Bayesian} optimization with fair regret.
\newblock In \emph{Proc. {ICML}}, pages 9691--9701. PMLR, 2021.

\bibitem[Smith et~al.(1988)Smith, Everhart, Dickson, Knowler, and
  Johannes]{smith1988using}
Jack~W. Smith, J.~E. Everhart, W.~C. Dickson, W.~C. Knowler, and R.~S.
  Johannes.
\newblock Using the {ADAP} learning algorithm to forecast the onset of diabetes
  mellitus.
\newblock In \emph{Proc Annu. Symp. Comput. Appl. Med. Care}, pages 261--265,
  1988.

\bibitem[Snoek et~al.(2012)Snoek, Larochelle, and Adams]{snoek2012practical}
Jasper Snoek, Hugo Larochelle, and Ryan~P Adams.
\newblock Practical {Bayesian} optimization of machine learning algorithms.
\newblock In \emph{Proc. {NeurIPS}}, pages 2951--2959, 2012.

\bibitem[Soleimani and Nazerfard(2019)]{soleimani2019cross}
Elnaz Soleimani and Ehsan Nazerfard.
\newblock Cross-subject transfer learning in human activity recognition systems
  using generative adversarial networks.
\newblock arxiv:1903.12489, 2019.

\bibitem[Srinivas et~al.(2010)Srinivas, Krause, Kakade, and
  Seeger]{srinivas2009gaussian}
Niranjan Srinivas, Andreas Krause, Sham~M. Kakade, and Matthias Seeger.
\newblock {Gaussian} process optimization in the bandit setting: No regret and
  experimental design.
\newblock In \emph{Proc. {ICML}}, pages 1015--1022, 2010.

\bibitem[Swersky et~al.(2013)Swersky, Snoek, and Adams]{swersky2013multi}
Kevin Swersky, Jasper Snoek, and Ryan~P. Adams.
\newblock Multi-task {Bayesian} optimization.
\newblock In \emph{Proc. {NeurIPS}}, pages 2004--2012, 2013.

\bibitem[Vanschoren(2018)]{vanschoren2018meta}
Joaquin Vanschoren.
\newblock Meta-learning: A survey.
\newblock {arXiv}:1810.03548, 2018.

\bibitem[Volpp et~al.(2020)Volpp, Froehlich, Fischer, Doerr, Falkner, Hutter,
  and Daniel]{volpp2020meta}
Michael Volpp, Lukas Froehlich, Kirsten Fischer, Andreas Doerr, Stefan Falkner,
  Frank Hutter, and Christian Daniel.
\newblock Meta-learning acquisition functions for transfer learning in
  {Bayesian} optimization.
\newblock In \emph{Proc. {ICLR}}, 2020.

\bibitem[Wang et~al.(2016)Wang, Kurth-Nelson, Tirumala, Soyer, Leibo, Munos,
  Blundell, Kumaran, and Botvinick]{wang1611learning}
Jane~X. Wang, Zeb Kurth-Nelson, Dhruva Tirumala, Hubert Soyer, Joel~Z Leibo,
  R{\'e}mi Munos, Charles Blundell, Dharshan Kumaran, and Matt Botvinick.
\newblock Learning to reinforcement learn.
\newblock {arXiv}:1611.05763, 2016.

\bibitem[Wang et~al.(2018)Wang, Kim, and Kaelbling]{wang2018regret}
Zi~Wang, Beomjoon Kim, and Leslie~Pack Kaelbling.
\newblock Regret bounds for meta {Bayesian} optimization with an unknown
  {Gaussian} process prior.
\newblock In \emph{Proc. {NeurIPS}}, pages 10477--10488, 2018.

\bibitem[Wilson et~al.(2014)Wilson, Fern, and Tadepalli]{wilson2014using}
Aaron Wilson, Alan Fern, and Prasad Tadepalli.
\newblock Using trajectory data to improve {Bayesian} optimization for
  reinforcement learning.
\newblock \emph{Journal of Machine Learning Research}, 15\penalty0
  (1):\penalty0 253--282, 2014.

\bibitem[Wistuba et~al.(2015{\natexlab{a}})Wistuba, Schilling, and
  Schmidt-Thieme]{wistuba2015learning}
Martin Wistuba, Nicolas Schilling, and Lars Schmidt-Thieme.
\newblock Learning hyperparameter optimization initializations.
\newblock In \emph{Proc. {DSAA}}, pages 1--10. IEEE, 2015{\natexlab{a}}.

\bibitem[Wistuba et~al.(2015{\natexlab{b}})Wistuba, Schilling, and
  Schmidt-Thieme]{wistuba2015sequential}
Martin Wistuba, Nicolas Schilling, and Lars Schmidt-Thieme.
\newblock Sequential model-free hyperparameter tuning.
\newblock In \emph{Proc. {ICDM}}, pages 1033--1038, 2015{\natexlab{b}}.

\bibitem[Wistuba et~al.(2016)Wistuba, Schilling, and
  Schmidt-Thieme]{wistuba2016two}
Martin Wistuba, Nicolas Schilling, and Lars Schmidt-Thieme.
\newblock Two-stage transfer surrogate model for automatic hyperparameter
  optimization.
\newblock In \emph{Proc. {ECML/PKDD}}, pages 199--214, 2016.

\bibitem[Wistuba et~al.(2018)Wistuba, Schilling, and
  Schmidt-Thieme]{wistuba2018scalable}
Martin Wistuba, Nicolas Schilling, and Lars Schmidt-Thieme.
\newblock Scalable {Gaussian} process-based transfer surrogates for
  hyperparameter optimization.
\newblock \emph{Machine Learning}, 107\penalty0 (1):\penalty0 43--78, 2018.

\bibitem[Wu et~al.(2020)Wu, Toscano-Palmerin, Frazier, and
  Wilson]{wu2020practical}
Jian Wu, Saul Toscano-Palmerin, Peter~I Frazier, and Andrew~Gordon Wilson.
\newblock Practical multi-fidelity {Bayesian} optimization for hyperparameter
  tuning.
\newblock In \emph{Proc. {UAI}}, pages 788--798, 2020.

\bibitem[Xu et~al.(2018)Xu, van Hasselt, and Silver]{xu2018meta}
Zhongwen Xu, Hado~P van Hasselt, and David Silver.
\newblock Meta-gradient reinforcement learning.
\newblock In \emph{Proc. {NeurIPS}}, pages 2396--2407, 2018.

\bibitem[Yogatama and Mann(2014)]{yogatama2014efficient}
Dani Yogatama and Gideon Mann.
\newblock Efficient transfer learning method for automatic hyperparameter
  tuning.
\newblock In \emph{Proc. {AISTATS}}, pages 1077--1085, 2014.

\bibitem[Zhang et~al.(2017)Zhang, Hoang, Low, and
  Kankanhalli]{zhang2017information}
Yehong Zhang, Trong~Nghia Hoang, Bryan Kian~Hsiang Low, and Mohan Kankanhalli.
\newblock Information-based multi-fidelity {Bayesian} optimization.
\newblock In \emph{Proc. {NeurIPS} Workshop on {Bayesian} Optimization}, 2017.

\bibitem[Zhang et~al.(2020)Zhang, Dai, and Low]{zhang2020bayesian}
Yehong Zhang, Zhongxiang Dai, and Bryan Kian~Hsiang Low.
\newblock Bayesian optimization with binary auxiliary information.
\newblock In \emph{Proc. {UAI}}, pages 1222--1232, 2020.

\end{thebibliography}
\appendix
\renewcommand\thesection{\Alph{section}}
\renewcommand\thesubsection{\thesection.\arabic{subsection}}

\renewcommand\thefigure{\arabic{figure}}
\setcounter{figure}{2}

\onecolumn

\title{On Provably Robust Meta-Bayesian Optimization (Supplementary material)}
%
%

\maketitle

\setcounter{equation}{7}
\setcounter{lemma}{1}

\section{Proof of Theorem~\ref{regret_bound}}
\label{app:first_section}
To begin with, we need the following lemma to give a high-probability confidence bound on the target function, which will be used in the theoretical analysis of both Theorems~\ref{regret_bound} and~\ref{regret_bound_ts}.
\begin{lemma}
\label{gaussian_bound}
Let $\delta \in (0,1)$ and $\beta_t=B + \sigma \sqrt{2(\gamma_{t-1} + 1 + \log(4/\delta))}$, then
\begin{equation*}
|f(\mathbf{x})-\mu_{t-1}(\mathbf{x})| \leq \beta_t\sigma_{t-1}(\mathbf{x}) \qquad \forall \mathbf{x}\in \mathcal{D},\, t\geq1
\end{equation*}
which holds with probability of $\geq1-\delta/4$.
\end{lemma}
Lemma~\ref{app:first_section} follows directly from Theorem 2 of~\citep{chowdhury2017kernelized}.

To facilitate the theoretical analysis of RM-GP-UCB, we introduce the following auxiliary term:
\begin{equation}
\widetilde{\zeta}_t(\mathbf{x})=\nu_t\left[\sum^M_{i=1}\omega_i\left[\widetilde{\mu}_{i}(\mathbf{x}) + \tau\widetilde{\sigma}_{i}(\mathbf{x})\right]\right] 
+ (1-\nu_t)\left[\mu_{t-1}(\mathbf{x})+\beta_t\sigma_{t-1}(\mathbf{x})\right]
\label{acq_fake}
\end{equation}
in which $\widetilde{\mu}_{i}(\mathbf{x})$ and $\widetilde{\sigma}_{i}(\mathbf{x})$ are obtained by replacing each noisy output of the meta-observations $y_{i,j}$ in the calculation of $\overline{\mu}_{i}(\mathbf{x})$ and $\overline{\sigma}_{i}(\mathbf{x})$~\eqref{acq_func}
by the (hypothetically available) noisy target function output observation at the corresponding input $\mathbf{x}_{i,j}$.
Eq.~\eqref{acq_fake} will serve as the bridge to connect the acquisition function of RM-GP-UCB~\eqref{acq_func} with the target function $f$ in the subsequent theoretical analysis, which will be demonstrated in Appendix~\ref{app:proof_theorem_1}.
To simplify exposition, we omit the superscript in our notation to represent the acquisition function~\eqref{acq_func}, i.e., we use $\overline{\zeta}_t$ to denote the acquisition function of RM-GP-UCB instead of $\overline{\zeta}^{\text{UCB}}_t$.
The next lemma shows that the difference between $\overline{\zeta}_t(\mathbf{x})$~\eqref{acq_func} and $\widetilde{\zeta}_t(\mathbf{x})$~\eqref{acq_fake} is bounded $\forall \mathbf{x} \in \mathcal{D}$, whose proof is given in Appendix~\ref{app:proof_lemma_1}.
\begin{lemma}
\label{ucb_diff}
Let $\delta \in (0, 1)$. Suppose the RM-GP-UCB algorithm is run with parameters $\nu_t\in [0,1]$ $\forall t\geq 1$, and $\omega_i\geq 0$ for $i=1,\ldots,M$ and $\sum_{i=1,\ldots,M}\omega_i=1$. 
Then with probability of $\geq 1 - \delta / 4$,
\[    \left|\overline{\zeta}_t(\mathbf{x})-\widetilde{\zeta}_t(\mathbf{x})\right| \leq \nu_t\alpha \qquad \forall \mathbf{x} \in \mathcal{D}
\]
in which 
\[
\alpha \triangleq \sum^M_{i=1}\omega_i \frac{N_i}{\sigma^2}(2\sqrt{2\sigma^2\log\frac{8N_i}{\delta}}+d_i).
\]
\end{lemma}

Next, because $\widetilde{\mu}_{i}(\mathbf{x})$ and $\widetilde{\sigma}_{i}(\mathbf{x})$ are calculated using the (hypothetically available) noisy observations of the target function (i.e., same as $\mu_{t-1}(\mathbf{x})$ and $\sigma_{t-1}(\mathbf{x})$), we can also get the following lemma on the concentration of the target function $f$
which, similar to Lemma~\ref{gaussian_bound} above, also follows directly from Theorem 2 of~\citep{chowdhury2017kernelized}.
\begin{lemma}
\label{lemma:confidence:bound:tau}
Let 
$\tau=B + \sigma \sqrt{2(\gamma_{N} + 1 + \log(4M/\delta))}$,
we have that
\begin{equation*}
|f(\mathbf{x})-\widetilde{\mu}_{i}(\mathbf{x})| \leq \tau\widetilde{\sigma}_{i}(\mathbf{x}) \qquad \forall \mathbf{x}\in \mathcal{D},\, i=1,...,M,
\end{equation*}
which also holds with probability $\geq1-\delta/4$.
\end{lemma}

\subsection{Proof of Lemma~\ref{ucb_diff}}
\label{app:proof_lemma_1}
Let $\mathbf{K}_i=[k(\mathbf{x}_{i,j}, \mathbf{x}_{i,j'})]_{j,j'=1,\ldots,N_i}$ represent the Gram matrix corresponding to the inputs of the meta-observations from meta-task $i$, and $\mathbf{k}_i=[k(\mathbf{x}_{i,j}, \mathbf{x})]^{\top}_{j=1,\ldots,N_i}$.
Denote by $\lambda_j[\mathbf{A}]$ the $j$-th eigenvalue of matrix $\mathbf{A}$.
Firstly, we need the following lemma proving an upper bound on matrix $L_2$ norm:
\begin{lemma}
\label{frob_norm}
For all $i=1,\ldots,M$, we have that
\[
\norm{\left(\mathbf{K}_{i}+\sigma^2I\right)^{-1}}_2 \leq \frac{1}{\sigma^2}.
\]
\end{lemma}
\begin{proof}
\begin{equation*}
\begin{split}
    \norm{\left(\mathbf{K}_{i}+\sigma^2I\right)^{-1}}_2 &= \sqrt{\max_{j=1,\ldots,N_i} \lambda_j\left[\left(\left(\mathbf{K}_{i}+\sigma^2I\right)^{-1}\right)^{\top} \left(\mathbf{K}_{i}+\sigma^2I\right)^{-1}\right]}\\
    &=\sqrt{\max_{j=1,\ldots,N_i} \lambda_j\left[\left(\mathbf{K}_{i}+\sigma^2I\right)^{-1}\right]^2}\\
    &\leq \frac{1}{\sigma^2}
\end{split}
\end{equation*}
\end{proof}

Next, define $\overline{\mathbf{f}}_{i}=[f_i(\mathbf{x}_{i,j})]_{j=1,\ldots,N_i}$ (in which $f_i(\mathbf{x}_{i,j})$ represents the value of meta-function $i$ at input $\mathbf{x}_{i,j}$), 
and $\widetilde{\mathbf{f}}_{i}=[f(\mathbf{x}_{i,j})]_{j=1,\ldots,N_i}$ (in which $f(\mathbf{x}_{i,j})$ represents the value of target function at input $\mathbf{x}_{i,j}$).
Similarly, define $\overline{\mathbf{y}}_{i}=[y_{i,j}]_{j=1,\ldots,N_i}$ (in which $y_{i,j}$ represents the noisy output observation of meta-task $i$ at input $\mathbf{x}_{i,j}$), 
and $\widetilde{\mathbf{y}}_{i}=[y(\mathbf{x}_{i,j})]_{j=1,\ldots,N_i}$ (in which $y(\mathbf{x}_{i,j})$ represents the hypothetically observed noisy output observation of the target function at input $\mathbf{x}_{i,j}$).
With these definitions, the next lemma shows upper bounds on the distance between $\overline{\mathbf{y}}_{i}$ and $\overline{\mathbf{f}}_{i}$, as well as that distance between $\widetilde{\mathbf{y}}_{i}$ and $\widetilde{\mathbf{f}}_{i}$.
\begin{lemma}
\label{bound_with_obs_noise}
With probability $\geq 1 - \delta/4$,
\begin{equation*}
\begin{split}
\norm{\overline{\mathbf{y}}_{i} - \overline{\mathbf{f}}_{i}}_2 \leq \sqrt{N_i} \sqrt{2\sigma^2\log\frac{8N_i}{\delta}},\\
\norm{\widetilde{\mathbf{y}}_{i} - \widetilde{\mathbf{f}}_{i}}_2 \leq \sqrt{N_i} \sqrt{2\sigma^2\log\frac{8N_i}{\delta}}.
\end{split}
\end{equation*}
\end{lemma}
\begin{proof}
Following the same analysis as Lemma 5.1 of \citep{srinivas2009gaussian}, we have that for the standard Gaussian random variable $z \sim \mathcal{N}(0, 1)$,
\begin{equation}
\mathbb{P}(\left|z\right| > c) \leq e^{-\frac{c^2}{2}}.
\label{standard_gaussian}
\end{equation}
Since for each $j=1,\ldots,N_i$, we have that $y_{i,j} - f_i(\mathbf{x}_{i,j}) \sim \mathcal{N}(0,\sigma^2)$ and that $y(\mathbf{x}_{i,j}) - f(\mathbf{x}_{i,j}) \sim \mathcal{N}(0,\sigma^2)$, which leads to the following,
\begin{equation*}
\begin{split}
\mathbb{P}\left(\left|\frac{y_{i,j} - f_i(\mathbf{x}_{i,j})}{\sigma}\right| > \sqrt{2\log\frac{8N_i}{\delta}}\right) = \mathbb{P}\left(\left|y_{i,j} - f_i(\mathbf{x}_{i,j}) \right| > \sqrt{2\sigma^2\log\frac{8N_i}{\delta}}\right) \leq \frac{\delta}{8N_i},\\
\mathbb{P}\left(\left|\frac{y(\mathbf{x}_{i,j}) - f(\mathbf{x}_{i,j})}{\sigma}\right| > \sqrt{2\log\frac{8N_i}{\delta}}\right) = \mathbb{P}\left(\left|y(\mathbf{x}_{i,j}) - f(\mathbf{x}_{i,j}) \right| > \sqrt{2\sigma^2\log\frac{8N_i}{\delta}}\right) \leq \frac{\delta}{8N_i}.
\end{split}
\end{equation*}
Taking a union bound over $j=1,\ldots,N_i$ for each of the two equations above, we have that for \emph{all} $j=1,\ldots,N_i$,
\begin{equation*}
\begin{split}
\left|y_{i,j} - f_i(\mathbf{x}_{i,j}) \right| \leq \sqrt{2\sigma^2\log\frac{8N_i}{\delta}}, \\
\left|y(\mathbf{x}_{i,j}) - f(\mathbf{x}_{i,j}) \right| \leq \sqrt{2\sigma^2\log\frac{8N_i}{\delta}},
\end{split}
\end{equation*}
both of which hold with probability $\geq 1 - \delta/8$.
Therefore, with probability $\geq 1 - \delta/8$,
\begin{equation}
\norm{\overline{\mathbf{y}}_{i} - \overline{\mathbf{f}}_{i}}_2 = \sqrt{\sum^{N_i}_{j=1} \left|y_{i,j} - f_i(\mathbf{x}_{i,j})\right|^2 } \leq \sqrt{\sum^{N_i}_{j=1} 2\sigma^2\log\frac{8N_i}{\delta}} \leq \sqrt{N_i}\sqrt{2\sigma^2\log\frac{8N_i}{\delta}}.
\label{eq:tmp_1}
\end{equation}
Repeating the procedure above leads to
\begin{equation}
\norm{\widetilde{\mathbf{y}}_{i} - \widetilde{\mathbf{f}}_{i}}_2 \leq \sqrt{N_i}\sqrt{2\sigma^2\log\frac{8N_i}{\delta}}
\label{eq:tmp_2}
\end{equation}
which also holds with probability $\geq 1 - \delta/8$.
Taking a union bound over equations~\eqref{eq:tmp_1} and~\eqref{eq:tmp_2} completes the proof.
\end{proof}

With these supporting lemmas, Lemma~\ref{ucb_diff} can be proved as follows:
\begin{align}
    \left|\overline{\zeta}_t(\mathbf{x})-\widetilde{\zeta}_t(\mathbf{x})\right| &= \left|\nu_t\left[\sum^M_{i=1}\omega_i[\overline{\mu}_{i}(\mathbf{x}) + \sqrt{\tau}\overline{\sigma}_{i}(\mathbf{x})]\right] - \nu_t\left[\sum^M_{i=1}\omega_i[\widetilde{\mu}_{i}(\mathbf{x}) + \sqrt{\tau}\widetilde{\sigma}_{i}(\mathbf{x})]\right]\right| \nonumber\\
    &\stackrel{\text{(a)}}{=}\left|\nu_t\sum^M_{i=1}\omega_i[\overline{\mu}_{i}(\mathbf{x})-\widetilde{\mu}_{i}(\mathbf{x})]\right| \nonumber\\
    &\leq \nu_t\sum^M_{i=1}\omega_i\left|\overline{\mu}_{i}(\mathbf{x})-\widetilde{\mu}_{i}(\mathbf{x})\right|\nonumber\\
    &\leq \nu_t\sum^M_{i=1}\omega_i\left|\mathbf{k}_{i}(\mathbf{x})^{\top} (\mathbf{K}_{i}+\sigma^2I)^{-1}(\overline{\mathbf{y}}_{i}-\widetilde{\mathbf{y}}_{i})\right| \nonumber\\
    &\stackrel{\text{(b)}}{\leq} \nu_t\sum^M_{i=1}\omega_i \norm{\mathbf{k}_{i}(\mathbf{x})}_2 \norm{(\mathbf{K}_{i}+\sigma^2I)^{-1}}_2 \norm{\overline{\mathbf{y}}_{i}-\widetilde{\mathbf{y}}_{i}}_2 \nonumber\\
    &\stackrel{\text{(c)}}{\leq} \nu_t\sum^M_{i=1}\omega_i \norm{\mathbf{k}_{i}(\mathbf{x})}_2 \frac{1}{\sigma^2} \norm{\overline{\mathbf{y}}_{i}-\widetilde{\mathbf{y}}_{i}}_2\nonumber\\
    &\stackrel{\text{(d)}}{\leq} \nu_t\sum^M_{i=1}\omega_i \sqrt{N_i} \frac{1}{\sigma^2} \norm{\overline{\mathbf{y}}_{i}-\widetilde{\mathbf{y}}_{i}}_2\nonumber\\
    &\leq \nu_t\sum^M_{i=1}\omega_i \frac{\sqrt{N_i}}{\sigma^2}\norm{\overline{\mathbf{y}}_{i} - \overline{\mathbf{f}}_{i} + \overline{\mathbf{f}}_{i} - \widetilde{\mathbf{f}}_{i} + \widetilde{\mathbf{f}}_{i} - \widetilde{\mathbf{y}}_{i}}_2\nonumber\\
    &\leq \nu_t\sum^M_{i=1}\omega_i \frac{\sqrt{N_i}}{\sigma^2}\left[\norm{\overline{\mathbf{y}}_{i} - \overline{\mathbf{f}}_{i}}_2+\norm{\overline{\mathbf{f}}_{i} - \widetilde{\mathbf{f}}_{i}}_2+\norm{\widetilde{\mathbf{f}}_{i} - \widetilde{\mathbf{y}}_{i}}_2\right]\nonumber\\
    &\stackrel{\text{(e)}}{\leq} \nu_t\sum^M_{i=1}\omega_i \frac{\sqrt{N_i}}{\sigma^2}\left(2\sqrt{N_i}\sqrt{2\sigma^2\log\frac{8N_i}{\delta}}+\norm{\overline{\mathbf{f}}_{i} - \widetilde{\mathbf{f}}_{i}}_2\right) \nonumber\\
    &= \nu_t\sum^M_{i=1}\omega_i \frac{\sqrt{N_i}}{\sigma^2}\left(2\sqrt{N_i}\sqrt{2\sigma^2\log\frac{8N_i}{\delta}}+\sqrt{\sum^{N_i}_{j=1}\left(f_i(\mathbf{x}_{i,j}) - f(\mathbf{x}_{i,j})\right)^2}\right) \nonumber\\
    &\stackrel{\text{(f)}}{\leq} \nu_t\sum^M_{i=1}\omega_i \frac{\sqrt{N_i}}{\sigma^2}\left(2\sqrt{N_i}\sqrt{2\sigma^2\log\frac{8N_i}{\delta}}+d_i\sqrt{N_i}\right) \nonumber\\
    &= \nu_t\sum^M_{i=1}\omega_i \frac{N_i}{\sigma^2}\left(2\sqrt{2\sigma^2\log\frac{8N_i}{\delta}}+d_i\right) \nonumber\\
    &\triangleq \nu_t \alpha
\label{ucb_diff_proof}
\end{align}
which holds with probability $\geq 1-\delta/4$. (a) holds because $\overline{\sigma}_i(\mathbf{x})=\widetilde{\sigma}_i(\mathbf{x})$ for all $\mathbf{x} \in \mathcal{D}$, 
because 
the posterior standard deviation only depends on the input locations and is independent of the corresponding output responses; 
(b) follows from Cauchy-Schwarz inequality,
(c) follows from Lemma~\ref{frob_norm}, 
(d) results from the assumption w.l.o.g.~that $k\left(\mathbf{x}, \mathbf{x}'\right) \leq 1$ for all $\mathbf{x}, \mathbf{x}' \in \mathcal{D}$,
(e) follows from Lemma~\ref{bound_with_obs_noise}, 
(f) is obtained from the definition of the function gap: $d_i\triangleq \max_{j=1,\ldots,N_i}\left|f(\mathbf{x}_{i,j})-f_i(\mathbf{x}_{i,j})\right|$ for $i=1,\ldots,M$.
This completes the proof of Lemma~\ref{ucb_diff}.

\subsection{Proof of Theorem~\ref{regret_bound}} 
\label{app:proof_theorem_1}
To begin with, we need the following lemma showing a high-probability upper bound on the global maximum of the target function.
\begin{lemma}
\label{bound_opt_func_val}
Given $\delta \in (0,1)$.
Let $\mathbf{x}^*$ denote a global maximizer of the target function $f$, and $\alpha$ be as defined in Lemma~\ref{ucb_diff}. 
Suppose the RM-GP-UCB algorithm is run with the parameter $\nu_t\in [0, 1]$ for all $t\geq 1$.
Then, with probability $\geq 1-3\delta/4$,
\[
f(\mathbf{x}^*) \leq \overline{\zeta}_t(\mathbf{x}_t)+\nu_t\alpha \qquad \forall t\geq 1.
\]
\end{lemma}
\begin{proof} 
Firstly, as a result of Lemma~\ref{gaussian_bound} and Lemma~\ref{lemma:confidence:bound:tau} (both hold with probability of $\geq 1 - \delta/4$), at any iteration $t\geq 1$ and for all $\mathbf{x} \in \mathcal{D}$, we have that with probability $\geq 1 - \delta/4 - \delta/4$, $\widetilde{\zeta}_t(\mathbf{x})$ is an upper bound on $f(\mathbf{x})$:
\begin{equation}
\begin{split}
    \widetilde{\zeta}_t&(\mathbf{x})-f(\mathbf{x})=\widetilde{\zeta}_t(\mathbf{x})-\left[\nu_t\sum^M_{i=1}\omega_i f(\mathbf{x}) + (1-\eta_t)f(\mathbf{x})\right]\\
    &=\nu_t\sum^M_{i=1}\omega_i\left[\widetilde{\mu}_{i}(\mathbf{x}) + \sqrt{\tau}\widetilde{\sigma}_{i}(\mathbf{x})-f(\mathbf{x})\right]+ (1-\nu_t)\left[\mu_{t-1}(\mathbf{x})+\sqrt{\beta_t}\sigma_{t-1}(\mathbf{x})-f(\mathbf{x})\right] \geq 0.
\end{split}
\label{upper_bound}
\end{equation}

Therefore, with probability $\geq 1-\delta/4 - \delta/4 - \delta/4$,
\begin{equation}
\begin{split}
    f(\mathbf{x}^*) \stackrel{\text{(a)}}{\leq} \widetilde{\zeta}_t(\mathbf{x}^*) \stackrel{\text{(b)}}{\leq} \overline{\zeta}_t(\mathbf{x}^*)+\nu_t\alpha \stackrel{\text{(c)}}{\leq} \overline{\zeta}_t(\mathbf{x}_t)+\nu_t\alpha
\end{split}
\label{eq:tmp}
\end{equation}
in which (a) results from~\eqref{upper_bound}, (b) is obtained via Lemma~\ref{ucb_diff} which holds with probability of $\geq 1-\delta/4$, and (c) follows from the policy for selecting $\mathbf{x}_t$, i.e., by maximizing~\eqref{acq_func}.
This completes the proof.
\end{proof}
Subsequently, we can show a high-probability upper bound on the instantaneous regret with the following lemma .
\begin{lemma}
\label{inst_regret_analysis}
Given $\delta \in (0,1)$. Let $\alpha$ be as defined in Lemma~\ref{ucb_diff}. Suppose the RM-GP-UCB algorithm is run with the parameters $\beta_t$, $\tau$ and $\nu_t$.
Then, with probability $\geq 1-3\delta/4$, $\forall t\geq 1$,
\[
r_t \leq 2\nu_t(\alpha+\tau)+2(1-\nu_t)\beta_t\sigma_{t-1}(\mathbf{x}_t).
\]
\end{lemma}
\begin{proof}
The instantaneous regret can be upper-bounded by
\begin{equation}
\begin{split}
    r_t &= f(\mathbf{x}^*)-f(\mathbf{x}_t) \stackrel{\text{(a)}}{\leq} \overline{\zeta}_t(\mathbf{x}_t)+\nu_t\alpha- f(\mathbf{x}_t)\\ 
    &\leq \underline{\overline{\zeta}_t(\mathbf{x}_t) - \widetilde{\zeta_t}(\mathbf{x}_t)} + \widetilde{\zeta_t}(\mathbf{x}_t)-f(\mathbf{x}_t)+\nu_t\alpha \\
    &\stackrel{\text{(b)}}{\leq} \nu_t\alpha + \nu_t\sum^M_{i=1}\omega_i \left[\widetilde{u}_{i}(\mathbf{x}_t) +\tau\widetilde{\sigma}_{i}(\mathbf{x}_t)\right] +(1-\nu_t)\left[u_{t-1}(\mathbf{x}_t)+\beta_t\sigma_{t-1}(\mathbf{x}_t)\right] \\
    &\qquad - f(\mathbf{x}_t) + \nu_t\alpha\\
    &= \nu_t\alpha + \nu_t\sum^M_{i=1}\omega_i \left[\widetilde{u}_{i}(\mathbf{x}_t) +\tau\widetilde{\sigma}_{i}(\mathbf{x}_t)\right] +(1-\nu_t)\left[u_{t-1}(\mathbf{x}_t)+\beta_t\sigma_{t-1}(\mathbf{x}_t)\right] \\
    &\qquad - \left[\nu_t\sum^M_{i=1}\omega_i f(\mathbf{x}_t) + (1-\nu_t)f(\mathbf{x}_t) \right] + \nu_t\alpha\\
    &\leq \nu_t\alpha + \nu_t\sum^M_{i=1}\omega_i \left[\widetilde{u}_{i}(\mathbf{x}_t) -f(\mathbf{x}_t)\right]+
    \nu_t\sum^M_{i=1}\omega_i\tau\widetilde{\sigma}_{i}(\mathbf{x}_t) \\
    &\qquad +(1-\nu_t)\left[u_{t-1}(\mathbf{x}_t)-f(\mathbf{x}_t)\right] + (1-\nu_t)\beta_t\sigma_{t-1}(\mathbf{x}_t) + \nu_t\alpha\\
    &\stackrel{\text{(c)}}{\leq} 2\nu_t\alpha + 2\nu_t\sum^M_{i=1}\omega_i\tau\widetilde{\sigma}_{i}(\mathbf{x}_t)+2(1-\nu_t)\beta_t\sigma_{t-1}(\mathbf{x}_t)\\
    &\stackrel{\text{(d)}}{\leq} 2\nu_t\alpha + 2\nu_t\tau+2(1-\nu_t)\beta_t\sigma_{t-1}(\mathbf{x}_t)\\
    &\leq 2\nu_t(\alpha+\tau)+2(1-\nu_t)\beta_t\sigma_{t-1}(\mathbf{x}_t)
\end{split}
\label{eq:analyze_insta_regret}
\end{equation}
which holds with probability $\geq 1-3\delta/4$. (a) follows from Lemma~\ref{bound_opt_func_val} which holds with probability of $\geq 1-3\delta/4$, 
(b) results from Lemma~\ref{ucb_diff} as well as the definition of $\widetilde{\zeta}_t(\mathbf{x}_t)$~\eqref{acq_fake}, 
(c) is a result of Lemma~\ref{gaussian_bound} and Lemma~\ref{lemma:confidence:bound:tau}, 
and (d) follows because $\widetilde{\sigma}_{i}(\mathbf{x}_t) \leq 1$ for all $\mathbf{x}_t \in \mathcal{D}$, which can be easily verified using the formula of the GP posterior variance~\eqref{gp_posterior} and the assumption that $k(\mathbf{x},\mathbf{x}')\leq1$ for all $\mathbf{x},\mathbf{x}'\in \mathcal{D}$.
The error probabilities $3\delta/4=\delta/4+\delta/4+\delta/4$ result from Lemmas~\ref{gaussian_bound},~\ref{ucb_diff} and~\ref{lemma:confidence:bound:tau}.
\end{proof}

Next, we need to connect the second term from Lemma~\ref{inst_regret_analysis} with the information gain. The following lemma, which is Lemma 5.3 of~\citep{srinivas2009gaussian}, defines the information gain on the target function from any set of observations.
\begin{lemma}
\label{info_gain}
Let $\mathbf{f}_T$ and $\mathbf{y}_T$ denote the set of function values and noisy observations of the target function respectively after $T$ iterations. Then, the information gain about $f$ from the first $T$ observations can be expressed as
\[
I(\mathbf{y}_T;\mathbf{f}_T)=\frac{1}{2}\sum^T_{t=1}\log \left[1+\sigma^{-2}\sigma^2_{t-1}(\mathbf{x}_t)\right].
\]
\end{lemma}
Subsequently, we can upper bound the second term from Lemma~\ref{inst_regret_analysis} (summed from iterations 1 to $T$) by the maximum information gain via the following lemma.
\begin{lemma}
\label{upper_bound_by_info_gain}
Suppose the RM-GP-UCB algorithm is run with the parameters $\beta_t$ $\forall t\geq 1$ and a non-increasing sequence $\nu_t \in [0,1]$ $\forall t \geq 1$. 
Define the maximum information gain as $\gamma_T=\max_{A\in \mathcal{D}, |A|=T}I(\mathbf{y}_A;\mathbf{f}_A)$ 
in which $\mathbf{f}_A$ and $\mathbf{y}_A$ represent the function values and noisy observations from a set $A$ of inputs of size $T$. Then,
\[
\sum^T_{t=1}\left[2(1-\nu_t)\beta_t\sigma_{t-1}(\mathbf{x}_t)\right]^2 \leq (1-\nu_T)^2 C_1 \beta_T^2 \gamma_T
\]
in which $C_1\triangleq \frac{8}{\log(1+\sigma^{-2})}$.
\end{lemma}
\begin{proof}
Each term inside the summation can be upper-bounded by
\begin{equation}
\label{bound_each_regret_term}
\begin{split}
4(1-\nu_t)^2\beta_t^2\sigma^2_{t-1}(\mathbf{x}_t) &\stackrel{\text{(a)}}{\leq} 4(1-\nu_T)^2\beta_T^2 \sigma^2 \left(\sigma^{-2} \sigma^2_{t-1}(\mathbf{x}_t) \right) \\
&\stackrel{\text{(b)}}{\leq} 4(1-\nu_T)^2\beta_T^2 \sigma^2 \left(\frac{\sigma^{-2}}{\log(1+\sigma^{-2})} \log\left(1+\sigma^{-2}\sigma^2_{t-1}(\mathbf{x}_t)\right)\right)\\
&= (1-\nu_T)^2\beta_T^2 \frac{8}{\log(1+\sigma^{-2})} \left[\frac{1}{2} \log\left(1+\sigma^{-2}\sigma^2_{t-1}(\mathbf{x}_t)\right)\right]
\end{split}
\end{equation}
in which (a) follows since $\beta_t$ is non-decreasing in $t$ and $\nu_t$ is non-increasing in $t$, 
(b) follows since $\sigma^{-2} x \leq \frac{\sigma^{-2}}{\log(1+\sigma^{-2})} \log(1+\sigma^{-2}x)$ for all $x\in (0,1]$ and $\sigma^2_{t-1}(\mathbf{x}_t)\in (0, 1]$.

As a result, the summation can be decomposed as
\begin{equation*}
\begin{split}
\sum^T_{t=1}\left[2(1-\nu_t)\beta_t\sigma_{t-1}(\mathbf{x}_t)\right]^2 &\stackrel{\text{(a)}}{\leq} (1-\nu_T)^2\beta_T^2 \frac{8}{\log(1+\sigma^{-2})} \sum^T_{t=1} \left[\frac{1}{2} \log\left(1+\sigma^{-2}\sigma^2_{t-1}(\mathbf{x}_t)\right)\right]\\
&\stackrel{\text{(b)}}{=}(1-\nu_T)^2\beta_T^2 \frac{8}{\log(1+\sigma^{-2})} I(\mathbf{y}_T;\mathbf{f}_T)\\
&\stackrel{\text{(c)}}{\leq} (1-\nu_T)^2 C_1 \beta_T^2 \gamma_T
\end{split}
\end{equation*}
in which (a) results from~\eqref{bound_each_regret_term}, (b) follows from Lemma~\ref{info_gain}, 
and (c) is obtained by making use of the definition of $C_1$ and $\gamma_T$.
\end{proof}
Finally, an upper bound on the cumulative regret follows from combining these supporting lemmas:
\begin{equation}
\begin{split}
    R_T&=\sum^T_{t=1}r_t \stackrel{\text{(a)}}{\leq} \sum^T_{t=1}\left[ 2\nu_t(\alpha+\tau)+2\left(1-\nu_t\right)\beta_t\sigma_{t-1}(\mathbf{x}_t)\right]\\
    &= 2(\alpha+\tau) \sum^T_{t=1} \nu_t + \sum^T_{t=1}2 (1-\nu_t)\beta_t\sigma_{t-1}(\mathbf{x}_t)\\
    &\stackrel{\text{(b)}}{\leq} 2(\alpha+\tau) \sum^T_{t=1} \nu_t + \sqrt{T} \sqrt{\sum^T_{t=1}\left[2 (1-\nu_t)\beta_t\sigma_{t-1}(\mathbf{x}_t)\right]^2}\\
    &\stackrel{\text{(c)}}{\leq} 2(\alpha+\tau) \sum^T_{t=1} \nu_t + \sqrt{C_1 T (1-\nu_T)^2\beta_T^2\gamma_T}\\
    &\stackrel{\text{(d)}}{\leq} 2(\alpha+\tau) \sum^T_{t=1} \nu_t + \beta_T\sqrt{C_1 T \gamma_T}
\end{split}
\label{cum_reg}
\end{equation}
which holds with probability $\geq 1 - 3\delta/4$. (a) is a result of Lemma~\ref{inst_regret_analysis}, (b) follows from Cauchy-Schwarz inequality, (c) is obtained using Lemma~\ref{upper_bound_by_info_gain}, and (d) follows since $1-\nu_T \leq 1$. This completes the proof.

If the meta-weights $\omega_i$'s are allowed to change with $t$ (i.e., when our online meta-weight optimization is used), then the proof here only needs to be modified to let $\alpha$ depend on $t$: 
$R_T \leq 2\tau \sum^T_{t=1} \nu_t + 2\sum^T_{t=1} \nu_t \alpha_t + \beta_T\sqrt{C_1 T \gamma_T}$.
In this case, the no-regret convergence guarantee of RM-GP-UCB (Sec.~\ref{subsec:theory:rm_gp_ucb}) is still preserved since in this case, we can simply upper-bound every $\omega_{i,t}$ by $1$. That is
$R_T \leq 2(\alpha'+\tau) \sum^T_{t=1} \nu_t + \beta_T\sqrt{C_1 T \gamma_T}$,
with $\alpha' \triangleq \sum^M_{i=1} \frac{N_i}{\sigma^2}(2\sqrt{2\sigma^2\log\frac{8N_i}{\delta}}+d_i)$.

\subsection{Meta-tasks Can Improve the Convergence by Accelerating Exploration}
\label{app:improved_bound}
Here, we utilize the analysis in Appendix~\ref{app:proof_theorem_1} to illustrate how the meta-tasks (if similar to the target task) can help RM-GP-UCB obtain a better regret bound than standard GP-UCB in the early stage of the algorithm. For simplicity, we focus on the most favorable scenario where all meta-functions have equal values to the target function at their corresponding input locations, i.e., all function gaps are $0$: 
$d_i = \max_{j=1,\ldots,N_i}\left|f(\mathbf{x}_{i,j})-f_i(\mathbf{x}_{i,j})\right|=0, \forall i=1,\ldots,M$.
Although not realistic, this scenario is useful for illustrating how the meta-tasks help our RM-GP-UCB algorithm achieve a better convergence at the initial stage.

In this case, according to the definition of $\widetilde{\zeta}_t$~\eqref{acq_fake} and $\overline{\zeta}_t$~\eqref{acq_func}, we have that $\widetilde{\zeta}_t(\mathbf{x})=\overline{\zeta}_t(\mathbf{x}), \forall \mathbf{x}\in\mathcal{D}, t\geq1$. As a result, the analysis of~\eqref{eq:tmp} in the proof of Lemma~\ref{bound_opt_func_val} can be similarly applied, yielding:
\begin{equation}
f(\mathbf{x}^*) \leq \widetilde{\zeta}_t(\mathbf{x}^*)=\overline{\zeta}_t(\mathbf{x}^*) \leq \overline{\zeta}_t(\mathbf{x}_t).
\end{equation}
Next, we can re-analyze the instantaneous regret following similar steps to~\eqref{eq:analyze_insta_regret}:
\begin{equation}
\begin{split}
r_t &= f(\mathbf{x}^*) - f(\mathbf{x}_t) \leq \overline{\zeta}_t(\mathbf{x}_t) - f(\mathbf{x}_t)\\
&\leq 2\nu_t\sum^M_{i=1}\omega_i\tau\overline{\sigma}_{i}(\mathbf{x}_t)+2(1-\nu_t)\beta_t\sigma_{t-1}(\mathbf{x}_t)\\
&=\underbrace{2\nu_t\left( \sum^M_{i=1}\omega_i\tau\overline{\sigma}_{i}(\mathbf{x}_t) - \beta_t\sigma_{t-1}(\mathbf{x}_t) \right)}_{A1} + \underbrace{2\beta_t\sigma_{t-1}(\mathbf{x}_t)}_{A2},
\end{split}
\label{eq:better_bound}
\end{equation}
in which some intermediate steps that are identical to those used in~\eqref{eq:analyze_insta_regret} have been omitted for simplicity.
Note that term $A_2$ in~\eqref{eq:better_bound} is identical to the upper bound on the instantaneous regret for the standard GP-UCB algorithm~\citep{srinivas2009gaussian}. Therefore, the meta-tasks affect the upper bound on the instantaneous regret through the term $A_1$.

Recall Theorem~\ref{regret_bound} has told us that we should choose $\nu_t \rightarrow 0$ as $t \rightarrow \infty$. In the initial stage of the algorithm when $\nu_t$ is large, the impact of $A_1$ on the regret of the algorithm is large. In this case, the meta-tasks improve the upper bound on the instantaneous regret (compared with standard GP-UCB) if $A_1 < 0$, that is:
\begin{equation}
\sum^M_{i=1}\omega_i \overline{\sigma}_{i}(\mathbf{x}_t) < \frac{\beta_t}{\tau}\sigma_{t-1}(\mathbf{x}_t).
\label{eq:better_bound_2}
\end{equation}
In other words, RM-GP-UCB converges faster than standard GP-UCB in the initial stage if the (weighted combination of) meta-tasks have smaller uncertainty (i.e., posterior standard deviation) at $\mathbf{x}_t$ compared with the target task (scaled by $\beta_t/\tau$). 
Fortunately, in the early stage of the algorithm, this condition is highly likely to be satisfied: 
When the number of observations of the target task is small, the posterior standard deviation of the target GP posterior (i.e., RHS of Equation~\eqref{eq:better_bound_2}) is usually large; therefore, Equation~\eqref{eq:better_bound_2} is highly likely to be satisfied.
This insight turns out to have an intuitive and elegant interpretation as well.
In the initial stage of the standard GP-UCB algorithm, due to the lack of observations, the algorithm \emph{has large uncertainty} regarding the objective function and hence tends to \emph{explore}; however, the meta-tasks (assuming that they are similar to the target task) provides additional information for the algorithm, which \emph{reduces the uncertainty} about the objective function and hence \emph{decreases the requirement for initial exploration}. To summarize, in the initial stage, the meta-tasks, if similar to the target task, help RM-GP-UCB achieve smaller regret upper bound (hence converge faster) than GP-UCB by reducing the degree of exploration.
In less favorable scenarios where the function gaps are nonzero (i.e., the meta-functions are not exactly equal to the target function), some amount of errors will be introduced to the upper bound on the instantaneous regret~\eqref{eq:better_bound}. As a results, a positive error term will be added to the LHS of~\eqref{eq:better_bound_2}, making the theoretical condition for a faster convergence~\eqref{eq:better_bound_2} harder to satisfy.
At later stages where $\nu_t$ is already small and close to $0$, the impact of the term $A_1$ is significantly diminished, thus allowing our RM-GP-UCB algorithm to converge to no regret at a similar rate to standard GP-UCB.

\section{Proof of Theorem~\ref{regret_bound_ts}}
\label{app:proof:theorem:ts}
Our theoretical analysis of RM-GP-TS shares similarity with the works of~\citep{dai2020federated,dai2021differentially} but has important differences, e.g., unlike the works of~\citep{dai2020federated,dai2021differentially}, RM-GP-TS does not suffer from the error introduced by random Fourier features approximation since we do not need to consider the issues of communication efficiency and retaining (hence not transmitting) the raw data.

Based on the acquisition function $\overline{\zeta}_t$ for RM-GP-TS~\eqref{eq:acq_func_ts} (we have again removed the superscript for simplicity), define $\mathcal{E}^1_t$ as the event that $\overline{\zeta}_t(\mathbf{x}) = f^t(\mathbf{x})$ which happens with probability $1-\nu_t$, and define $\mathcal{E}^2_t$ as the event that $\overline{\zeta}_t(\mathbf{x}) = {\sum}^M_{i=1}\omega_i \left[\overline{f}^t_{i}(\mathbf{x})\right]$ which happens with probability $\nu_t$.
Define $\mathcal{F}_{t-1}$ as the filtration containing the history of input-output pairs of the target task up to and including iteration $t-1$.

\begin{lemma}
\label{confidence_bound:ts}
With $\tau$ defined in Lemma~\ref{lemma:confidence:bound:tau}, we have that
\begin{equation*}
|f_i(\mathbf{x})-\overline{\mu}_{i}(\mathbf{x})| \leq \tau\overline{\sigma}_{i}(\mathbf{x}) \qquad \forall \mathbf{x}\in \mathcal{D},\, i=1,...,M
\end{equation*}
which holds with probability $\geq1-\delta/4$.
\end{lemma}
Similar to Lemma~\ref{gaussian_bound} and Lemma~\ref{lemma:confidence:bound:tau}, Lemma~\ref{confidence_bound:ts} also follows from Theorem 2 of~\citep{chowdhury2017kernelized}.
Next, we also need the following lemma showing the concentration of functions sampled from the GP posterior around the posterior mean, for both the target function and the meta-functions.
\begin{lemma}
\label{lemma:concentration:sampled:function:around:mean}
With $\beta_t$ defined in Lemma~\ref{gaussian_bound} and $\tau$ defined in Lemma~\ref{lemma:confidence:bound:tau}, we have that
\[
|f^t(\mathbf{x}) - \mu_{t-1}(\mathbf{x})| \leq \beta_t \sqrt{2\log(\frac{|\mathcal{D}|t^2 2\pi^2}{\delta})} \sigma_{t-1}(\mathbf{x}), \qquad \forall \mathbf{x}\in \mathcal{D},\, t\geq1,
\]
which holds with probability $\geq 1 - \delta/12$, and that
\[
|f_i^t(\mathbf{x}) - \overline{\mu}_{i}(\mathbf{x})| \leq \tau \sqrt{2\log( \frac{M |\mathcal{D}|t^2 2\pi^2}{\delta})} \overline{\sigma}_{i}(\mathbf{x}), \qquad \forall \mathbf{x}\in \mathcal{D},\, t\geq 1,\, i=1,...,M
\]
which holds with probability $\geq 1 - \delta/12$.
\end{lemma}
The proof of Lemma~\ref{lemma:concentration:sampled:function:around:mean} follows straightforwardly from Lemma 5 of~\citep{chowdhury2017kernelized}, together with a union bound over all $\mathbf{x}\in\mathcal{D}$ and over all $t\geq 1$, as well as an additional union bound over all $M$ meta-tasks for the second inequality.

\begin{lemma}
\label{lemma:bound:ft:weighted:fti}
Define $d_i' \triangleq \max_{\mathbf{x}\in\mathcal{D}}| f(\mathbf{x}) - f_i(\mathbf{x}) |$.
Define $c_t \triangleq \beta_t \left(1 + \sqrt{2\log(\frac{|\mathcal{D}|t^2 2\pi^2}{\delta})}\right)$, and $c_t' \triangleq \tau \left(1 + \sqrt{2\log( \frac{M |\mathcal{D}|t^2 2\pi^2}{\delta})}\right)$.
With probability $\geq 1 - \delta/4 - \delta/4 - \delta/12 - \delta/12 = 1 - 2\delta/3$, we have that
\[
|f^t(\mathbf{x}) - \sum^M_{i=1} \omega_i f_i^t(\mathbf{x})| \leq c_t  + c_t' + \sum^M_{i=1}\omega_i d_i', \qquad\forall \mathbf{x}\in\mathcal{D}, t\geq 1.
\]
\end{lemma}
\begin{proof}
Firstly, we can bound the difference between the target function and a sampled function from its GP posterior.
\begin{equation}
\begin{split}
|f^t(\mathbf{x}) - f(\mathbf{x})| &\leq |f^t(\mathbf{x}) - \mu_{t-1}(\mathbf{x})| + |\mu_{t-1}(\mathbf{x}) - f(\mathbf{x})|\\
&\stackrel{(a)}{\leq} \beta_t \sqrt{2\log(\frac{|\mathcal{D}|t^2 2\pi^2}{\delta})} \sigma_{t-1}(\mathbf{x}) + \beta_t \sigma_{t-1}(\mathbf{x})\\
&=c_t \sigma_{t-1}(\mathbf{x}),
\end{split}
\label{eq:bound:f:ft}
\end{equation}
where $(a)$ results from Lemma~\ref{lemma:concentration:sampled:function:around:mean} and Lemma~\ref{gaussian_bound}, and hence holds with probability of $\geq 1 - \delta/12 - \delta/4$. Next, we do the same for all meta-functions $i=1,\ldots,M$.
\begin{equation}
\begin{split}
|f^t_i(\mathbf{x}) - f_i(\mathbf{x})| &\leq |f^t_i(\mathbf{x}) - \overline{\mu}_{i}(\mathbf{x})| + |\overline{\mu}_{i}(\mathbf{x}) - f_i(\mathbf{x})|\\
&\stackrel{(a)}{\leq} \tau\overline{\sigma}_{i}(\mathbf{x}) + \tau \sqrt{2\log( \frac{M |\mathcal{D}|t^2 2\pi^2}{\delta})} \overline{\sigma}_{i}(\mathbf{x})\\
&= c_t' \overline{\sigma}_{i}(\mathbf{x}),
\end{split}
\end{equation}
where $(a)$ results from Lemma~\ref{lemma:concentration:sampled:function:around:mean} and Lemma~\ref{confidence_bound:ts}, and hence also holds with probability of $\geq 1 - \delta/12 - \delta/4$.
Therefore, combining the above two inequalities gives us:
\begin{equation}
\begin{split}
|f^t(\mathbf{x}) - f_i^t(\mathbf{x})| &\leq |f^t(\mathbf{x}) - f(\mathbf{x})| + |f(\mathbf{x}) - f_i(\mathbf{x})| + |f_i(\mathbf{x}) - f^t_i(\mathbf{x})|\\
&\leq c_t \sigma_{t-1}(\mathbf{x}) + c_t' \overline{\sigma}_{i}(\mathbf{x}) + d_i'\\
&\leq c_t \sigma_{t-1}(\mathbf{x}) + c_t' + d_i',
\end{split}
\end{equation}
in which the last inequality follows since $\overline{\sigma}_{i}(\mathbf{x})\leq 1$.
Finally, the lemma can be proved as:
\begin{equation}
\begin{split}
|f^t(\mathbf{x}) - \sum^M_{i=1} \omega_i f_i^t(\mathbf{x})| &\leq \sum^M_{i=1} \omega_i |f^t(\mathbf{x}) - f_i^t(\mathbf{x})|\\
&\leq \sum^M_{i=1}\omega_i \left( c_t \sigma_{t-1}(\mathbf{x}) + c_t' + d_i' \right)\\
&\leq c_t + c_t' + \sum^M_{i=1}\omega_i d_i'.
\end{split}
\end{equation}
\end{proof}

Next, we define the set of "saturated points" in an iteration $t$, which are those inputs which incur large regrets in iteration $t$.
\begin{definition}
\label{def:saturated:point}
At iteration $t$, define the set of saturated points as
\[
S_t \triangleq \{ \mathbf{x}\in\mathcal{D} | \Delta(\mathbf{x}) > c_t \sigma_{t-1}(\mathbf{x}) \},
\]
where $\Delta(\mathbf{x}) \triangleq f(\mathbf{x}^*) - f(\mathbf{x})$.
\end{definition}

The next lemma will be useful in proving that the input we query in iteration $t$ is unsaturated (i.e., in proving Lemma~\ref{lemma:x_t:unsaturated}), and its proof makes use of Gaussian anti-concentration inequality.
\begin{lemma}
\label{lemma:use:gaussian:anti}
With probability of $\geq 1 - \delta/4$,
\[
\mathbb{P}\left( f^t(\mathbf{x}) > f(\mathbf{x}) | \mathcal{F}_{t-1}, \mathcal{E}^1_t\right) \geq p,\qquad \forall t\geq 1.
\]
where $p\triangleq \frac{e^{-1}}{4\sqrt{\pi}}$.
\end{lemma}
\begin{proof}
Define $\theta_t \triangleq \frac{|f(\mathbf{x}) - \mu_{t-1}(\mathbf{x}) |}{\beta_t \sigma_{t-1}(\mathbf{x})}$. 
\begin{equation}
\begin{split}
\mathbb{P}\left( f^t(\mathbf{x}) > f(\mathbf{x}) | \mathcal{F}_{t-1}, \mathcal{E}^1_t\right) &= \mathbb{P}\left( \frac{f^t(\mathbf{x}) - \mu_{t-1}(\mathbf{x})}{\beta_t \sigma_{t-1}(\mathbf{x})} > \frac{f(\mathbf{x}) - \mu_{t-1}(\mathbf{x})}{\beta_t \sigma_{t-1}(\mathbf{x})} | \mathcal{F}_{t-1}, \mathcal{E}^1_t\right)\\
&\geq \mathbb{P}\left( \frac{f^t(\mathbf{x}) - \mu_{t-1}(\mathbf{x})}{\beta_t \sigma_{t-1}(\mathbf{x})} > \frac{|f(\mathbf{x}) - \mu_{t-1}(\mathbf{x}) |}{\beta_t \sigma_{t-1}(\mathbf{x})} | \mathcal{F}_{t-1}, \mathcal{E}^1_t\right)\\
&= \mathbb{P}\left( \frac{f^t(\mathbf{x}) - \mu_{t-1}(\mathbf{x})}{\beta_t \sigma_{t-1}(\mathbf{x})} > \theta_t | \mathcal{F}_{t-1}, \mathcal{E}^1_t\right)\\
&\stackrel{(a)}{\geq} \frac{e^{-\theta_t^2}}{4\sqrt{\pi} \theta_t} \stackrel{(b)}{\geq} \frac{e^{-1}}{4\sqrt{\pi}}.
\end{split}
\end{equation}
Note that due to the way in which the function $f^t$ is sampled from the GP posterior, i.e., $f^t \sim \mathcal{GP}\left(\mu_{t-1}(\cdot), \beta_t^2 \sigma_{t-1}^2(\cdot)\right)$ (Sec.~\ref{sec:om_gp_ucb}), we have that $\frac{f^t(\mathbf{x}) - \mu_{t-1}(\mathbf{x})}{\beta_t \sigma_{t-1}(\mathbf{x})}$ follows a standard Gaussian distribution.
Therefore, step $(a)$ above results from the Gaussian anti-concentration inequality: denote by $Z$ the standard Gaussian distribution $\mathcal{N}(0, 1)$, then $\mathbb{P}(Z > \theta_t) \geq \frac{e^{-\theta_t^2}}{4\sqrt{\pi}\theta_t}$.
Step $(b)$ follows from Lemma~\ref{gaussian_bound} (i.e., $\theta_t \leq 1$) and hence holds with probability of $\geq 1 - \delta/4$.
\end{proof}

The next lemma shows that in every iteration, the probability that we choose an unsaturated input is lower-bounded.
\begin{lemma}
\label{lemma:x_t:unsaturated}
With probability of $\geq 1 - \delta/4 - \delta/12=1-\delta/3$,
\[
\mathbb{P}\left( \mathbf{x}_t \in \mathcal{D}\setminus S_t | \mathcal{F}_{t-1} \right) \geq (1-\nu_t)p, \qquad \forall t\geq 1.
\]
\end{lemma}
\begin{proof}
Firstly, we have that
\begin{equation}
\begin{split}
\mathbb{P}\left( \mathbf{x}_t \in \mathcal{D} \setminus S_t | \mathcal{F}_{t-1} \right) \geq \mathbb{P}\left( \mathbf{x}_t \in \mathcal{D} \setminus S_t | \mathcal{F}_{t-1}, \mathcal{E}^1_t \right) \mathbb{P}\left(\mathcal{E}^1_t\right) = \mathbb{P}\left( \mathbf{x}_t \in \mathcal{D} \setminus S_t | \mathcal{F}_{t-1}, \mathcal{E}^1_t \right) (1-\nu_t).
\end{split}
\end{equation}
Next, we attempt to lower-bound the term $\mathbb{P}\left( \mathbf{x}_t \in \mathcal{D} \setminus S_t | \mathcal{F}_{t-1}, \mathcal{E}^1_t \right)$.
\begin{equation}
\begin{split}
\mathbb{P}\left( \mathbf{x}_t \in \mathcal{D} \setminus S_t | \mathcal{F}_{t-1}, \mathcal{E}^1_t \right) \geq \mathbb{P}\left( f^t(\mathbf{x}^*) > f^t(\mathbf{x}),\forall \mathbf{x}\in S_t | \mathcal{F}_{t-1}, \mathcal{E}^1_t \right).
\end{split}
\end{equation}
The above inequality follows because $\mathbf{x}^*$ is always unsaturated: $\Delta(\mathbf{x}^*) = f(\mathbf{x}^*)-f(\mathbf{x}^*)=0\leq c_t\sigma_{t-1}(\mathbf{x}^*)$.
As a result, if the event on the RHS of the above inequality holds (i.e., an unsaturated input has larger value of $f^t$ than all saturated inputs), then the event on the LHS (i.e., $\mathbf{x}_t$ is unsaturated) also holds.
Next, we also have that $\forall \mathbf{x} \in S_t$,
\begin{equation}
f^t(\mathbf{x}) \leq f(\mathbf{x}) + c_t \sigma_{t-1}(\mathbf{x}) \leq f(\mathbf{x}) + \Delta(\mathbf{x}) = f(\mathbf{x}) + f(\mathbf{x}^*) - f(\mathbf{x}) = f(\mathbf{x}^*),
\end{equation}
in which the first inequality follows from~\eqref{eq:bound:f:ft} and hence holds with probability $\geq 1-\delta/12-\delta/4$, the second inequality is a result of the definition of saturated inputs (Definition~\ref{def:saturated:point}).
The above inequality implies that
\begin{equation}
\mathbb{P}\left( f^t(\mathbf{x}^*) > f^t(\mathbf{x}),\forall \mathbf{x}\in S_t | \mathcal{F}_{t-1}, \mathcal{E}^1_t \right) \geq \mathbb{P}\left( f^t(\mathbf{x}^*) > f(\mathbf{x}^*) | \mathcal{F}_{t-1}, \mathcal{E}^1_t \right).
\end{equation}

Lastly, combining the above inequalities gives us
\begin{equation}
\begin{split}
\mathbb{P}\left( \mathbf{x}_t \in \mathcal{D} \setminus S_t | \mathcal{F}_{t-1}, \mathcal{E}^1_t \right) \geq \mathbb{P}\left( f^t(\mathbf{x}^*) > f(\mathbf{x}^*) | \mathcal{F}_{t-1}, \mathcal{E}^1_t \right) \geq p,
\end{split}
\end{equation}
where the last inequality follows from Lemma~\ref{lemma:use:gaussian:anti}.
This completes the proof.
Note that the error probabilities for this lemma come from Lemma~\ref{lemma:concentration:sampled:function:around:mean} ($\delta/12$) and Lemma~\ref{gaussian_bound} ($\delta/4$).
\end{proof}

Next, we prove an upper bound on the expected instantaneous regret $r_t=f(\mathbf{x}^*)-f(\mathbf{x}_t)$.
\begin{lemma}
\label{lemma:upper:bound:expected:inst:regret}
With probability of $\geq 1 - \delta/4 - \delta/4 - \delta/12 - \delta/12 = 1-2\delta/3$,
\[
\mathbb{E}[r_t | \mathcal{F}_{t-1}] \leq c_t \left(1 + \frac{2}{(1-\nu_1)p}\right) \mathbb{E} [\sigma_{t-1}(\mathbf{x}_t) | \mathcal{F}_{t-1}] + \psi_t,
\]
where $\psi_t \triangleq 2\nu_t \left( c_t + c_t' + \sum^M_{i=1}\omega_i d_i' \right)$.
\end{lemma}
\begin{proof}
To begin with, define the unsaturated input with the smallest posterior standard deviation as
\begin{equation}
\overline{\mathbf{x}}_t \triangleq {\arg\min}_{\mathbf{x}\in \mathcal{D} \setminus S_t} \sigma_{t-1}(\mathbf{x}).
\end{equation}
This allows us to obtain the following:
\begin{equation}
\begin{split}
\mathbb{E}[\sigma_{t-1}(\mathbf{x}_t) | \mathcal{F}_{t-1}] \geq \mathbb{E}[\sigma_{t-1}(\mathbf{x}_t) | \mathcal{F}_{t-1}, \mathbf{x}_t\in\mathcal{D}\setminus S_t] \mathbb{P}\left(\mathbf{x}_t\in\mathcal{D}\setminus S_t\right) \stackrel{(a)}{\geq} \sigma_{t-1}(\overline{\mathbf{x}}_t)(1-\nu_t)p,
\end{split}
\label{eq:ues:xt:bar}
\end{equation}
where $(a)$ results from Lemma~\ref{lemma:x_t:unsaturated} and hence holds with probability $\geq 1-\delta/12-\delta/4$ (the error probabilities come from Lemma~\ref{lemma:concentration:sampled:function:around:mean} and Lemma~\ref{gaussian_bound}).
Subsequently, the instataneous regret can be upper-bounded as
\begin{equation}
\begin{split}
r_t &= \Delta(\mathbf{x}_t) = f(\mathbf{x}^*) - f(\overline{\mathbf{x}}_t) + f(\overline{\mathbf{x}}_t) - f(\mathbf{x}_t)\\
&\stackrel{(a)}{\leq} \Delta(\overline{\mathbf{x}}_t) + f^t(\overline{\mathbf{x}}_t) + c_t \sigma_{t-1}(\overline{\mathbf{x}}_t) - f^t(\mathbf{x}_t) + c_t \sigma_{t-1}(\mathbf{x}_t)\\
&\stackrel{(b)}{\leq} c_t \sigma_{t-1}(\overline{\mathbf{x}}_t) + c_t \sigma_{t-1}(\overline{\mathbf{x}}_t) + c_t \sigma_{t-1}(\mathbf{x}_t) + f^t(\overline{\mathbf{x}}_t) - f^t(\mathbf{x}_t)\\
&= c_t (2\sigma_{t-1}(\overline{\mathbf{x}}_t) + \sigma_{t-1}(\mathbf{x}_t)) + f^t(\overline{\mathbf{x}}_t) - f^t(\mathbf{x}_t),
\end{split}
\end{equation}
in which $(a)$ follows from~\eqref{eq:bound:f:ft}, and $(b)$ results from the definition of saturated input (Definition~\ref{def:saturated:point}) and that $\overline{\mathbf{x}}_t$ is unsaturated.
Next, we attempt to upper-bound the expected value of the term $f^t(\overline{\mathbf{x}}_t) - f^t(\mathbf{x}_t)$ from the equation above:
\begin{equation}
\begin{split}
\mathbb{E}&\left[f^t(\overline{\mathbf{x}}_t) - f^t(\mathbf{x}_t) | \mathcal{F}_{t-1}\right] \\
&=\mathbb{P}(\mathcal{E}^1_t) \mathbb{E}\left[f^t(\overline{\mathbf{x}}_t) - f^t(\mathbf{x}_t) | \mathcal{F}_{t-1}, \mathcal{E}^1_t\right] + \mathbb{P}(\mathcal{E}^2_t) \mathbb{E}\left[f^t(\overline{\mathbf{x}}_t) - f^t(\mathbf{x}_t) | \mathcal{F}_{t-1}, \mathcal{E}^2_t\right]\\
&\stackrel{(a)}{\leq} \nu_t \mathbb{E}\left[f^t(\overline{\mathbf{x}}_t) - f^t(\mathbf{x}_t) | \mathcal{F}_{t-1}, \mathcal{E}^2_t\right]\\
&\stackrel{(b)}{\leq} \nu_t \mathbb{E}\Big[\sum^M_{i=1}\omega_i f^t_i(\overline{\mathbf{x}}_t) + c_t + c_t' + \sum^M_{i=1}\omega_i d_i' +  c_t + c_t' + \sum^M_{i=1}\omega_i d_i' - \sum^M_{i=1}\omega_i f^t_i(\mathbf{x}_t) | \mathcal{F}_{t-1}, \mathcal{E}^2_t\Big]\\
&\stackrel{(c)}{\leq} 2\nu_t \left( c_t + c_t' + \sum^M_{i=1}\omega_i d_i' \right) \triangleq \psi_t.
\end{split}
\end{equation}
Step $(a)$ follows since conditioned on the event $\mathcal{E}^1_t$ ($\overline{\zeta}_t(\mathbf{x}) = f^t(\mathbf{x})$), we have that $f^t(\mathbf{x}) \leq f^t(\mathbf{x}_t),\forall \mathbf{x}\in\mathcal{D}$; step $(b)$ results from Lemma~\ref{lemma:bound:ft:weighted:fti}; step $(c)$ follows since conditioned on the event $\mathcal{E}^2_t$ (i.e., $\overline{\zeta}_t(\mathbf{x}) = {\sum}^M_{i=1}\omega_i \left[\overline{f}^t_{i}(\mathbf{x})\right]$), we have that ${\sum}^M_{i=1}\omega_i \left[\overline{f}^t_{i}(\mathbf{x})\right] \leq {\sum}^M_{i=1}\omega_i \left[\overline{f}^t_{i}(\mathbf{x}_t)\right],\forall\mathbf{x}\in\mathcal{D}$.
Lastly,
\begin{equation}
\begin{split}
\mathbb{E}[r_t | \mathcal{F}_{t-1}] &\leq \mathbb{E}[ c_t (2\sigma_{t-1}(\overline{\mathbf{x}}_t) + \sigma_{t-1}(\mathbf{x}_t)) + \psi_t | \mathcal{F}_{t-1}]\\
&\leq \mathbb{E}\left[ c_t \left(\frac{2}{(1-\nu_t)p} \sigma_{t-1}(\mathbf{x}_t) + \sigma_{t-1}(\mathbf{x}_t) \right) + \psi_t |\mathcal{F}_{t-1} \right]\\
&\leq c_t \left(1 + \frac{2}{(1-\nu_1)p}\right) \mathbb{E} [\sigma_{t-1}(\mathbf{x}_t) | \mathcal{F}_{t-1}] + \psi_t,
\end{split}
\end{equation}
in which the second inequality results from~\eqref{eq:ues:xt:bar}.
Note that the error probabilities for this Lemma follow from Lemma~\ref{lemma:bound:ft:weighted:fti}.
\end{proof}

Subsequently, we make use of martingale concentration inequalities to bound the cumulative regret.
\begin{definition}
Define $Y_0 = 0$, and for $t\geq 1$,
\[
X_t = r_t - c_t \left(1 + \frac{2}{(1-\nu_1)p}\right) \sigma_{t-1}(\mathbf{x}_t) - \psi_t,
\]
\[
Y_t = \sum^t_{s=1} X_s.
\]
\end{definition}
The next lemma shows that $\{Y_t\}_{t\geq 1}$ is a super-martingale.
\begin{lemma}
With probability $\geq 1 - \delta/4 - \delta/4 - \delta/12 - \delta/12=1-2\delta/3$,
$\{Y_t\}_{t\geq 1}$ is a super-martingale with respect to the filtration $\mathcal{F}_{t-1}$.
\end{lemma}
\begin{proof}
\begin{equation}
\begin{split}
\mathbb{E}[Y_t - Y_{t-1} | \mathcal{F}_{t-1}] &= \mathbb{E}[X_t | \mathcal{F}_{t-1}]\\
&=\mathbb{E}[ r_t - c_t \left(1 + \frac{2}{(1-\nu_1)p}\right) \sigma_{t-1}(\mathbf{x}_t) + \psi_t | \mathcal{F}_{t-1}]\\
&= \mathbb{E}[r_t | \mathcal{F}_{t-1}] - \left[ c_t \left(1 + \frac{2}{(1-\nu_1)p}\right) \mathbb{E}[\sigma_{t-1}(\mathbf{x}_t) | \mathcal{F}_{t-1}] + \psi_t \right] \leq 0,
\end{split}
\end{equation}
where the last inequality follows from Lemma~\ref{lemma:upper:bound:expected:inst:regret}.
\end{proof}

Finally, we are ready to use martingale concentration inequalities to bound the cumulative regret.
\begin{lemma}
\label{lemma:ts:final:upper:bound:RT}
With probability of $\geq 1 - \delta/4 - \delta/4 - \delta/12 - \delta/12 - \delta/12=1-3\delta/4$,
\[
R_T \leq \left(2B + c_T \left(1 + \frac{2}{(1-\nu_1)p}\right) + \psi_1\right) \sqrt{T (C_1 \gamma_T + 2\log(12/\delta))} + 2\sum^T_{t=1} \nu_t (c_t+c_t'+\sum^M_{i=1}\omega_i d_i')
\]
where $C_1=2/\log(1+\sigma^{-2})$.
\end{lemma}
\begin{proof}
To begin with, we have that
\begin{equation}
\begin{split}
|Y_t - Y_{t-1}| &= |X_t| \leq |r_t| + |c_t \left(1 + \frac{2}{(1-\nu_1)p}\right) \sigma_{t-1}(\mathbf{x}_t)| + |\psi_t|\\
&\leq 2B + c_t \left(1 + \frac{2}{(1-\nu_1)p}\right) + \psi_t,
\end{split}
\end{equation}
where the last inequality follows since $|r_t|=|f(\mathbf{x}^*)-f(\mathbf{x}_t)|\leq 2B$ (because $\norm{f}_k\leq B$ as we have assumed in Sec.~\ref{sec:background}, which immediately implies that $|f(\mathbf{x})|\leq B,\forall \mathbf{x}\in\mathcal{D}$), and $\sigma_{t-1}(\mathbf{x})\leq 1,\forall \mathbf{x}\in\mathcal{D}$.

Next, we apply the Azuma-Hoeffding Inequality with an error probability of $\delta/12$ (first inequality):
\begin{equation}
\begin{split}
\sum^T_{t=1} r_t &\leq \sum^T_{t=1} c_t \left(1 + \frac{2}{(1-\nu_1)p}\right) \sigma_{t-1}(\mathbf{x}_t) + \sum^T_{t=1} \psi_t + \\
&\qquad \sqrt{2 \log\frac{10}{\delta} \sum^T_{t=1} \left(2B + c_t \left(1 + \frac{2}{(1-\nu_1)p}\right) + \psi_t\right)^2 }\\
&\leq c_T \left(1 + \frac{2}{(1-\nu_1)p}\right) \sum^T_{t=1} \sigma_{t-1}(\mathbf{x}_t) + \sum^T_{t=1} \psi_t + \\
&\qquad \left(2B + c_T \left(1 + \frac{2}{(1-\nu_1)p}\right) + \psi_1\right)\sqrt{2 T \log\frac{12}{\delta}}\\
&\leq c_T \left(1 + \frac{2}{(1-\nu_1)p}\right) \sqrt{C_1' \gamma_T T} + \sum^T_{t=1} \psi_t + \\
&\qquad \left(2B + c_T \left(1 + \frac{2}{(1-\nu_1)p}\right) + \psi_1\right)\sqrt{2 T \log\frac{12}{\delta}}\\
&\leq \left(2B + c_T \left(1 + \frac{2}{(1-\nu_1)p}\right) + \psi_1\right) \sqrt{T (C_1' \gamma_T + 2\log(12/\delta))} + \\
&\qquad 2\sum^T_{t=1} \nu_t (c_t+c_t'+\sum^M_{i=1}\omega_i d_i).
\end{split}
\end{equation}
The second last inequality makes use of Lemma~\ref{upper_bound_by_info_gain} from the proof of RM-GP-UCB (excluding the factor of $(1-\nu_t)\beta_t$) with $C_1'\triangleq2/\log(1+\sigma^{-2})$.
\end{proof}

Recall that $c_t=\mathcal{O}(\sqrt{\gamma_t} \log t)$, $c_t'=\mathcal{O}(\log t)$.
Therefore, Lemma~\ref{lemma:ts:final:upper:bound:RT} can be further analyzed as:
\begin{equation}
\begin{split}
R_T &= \mathcal{O}\left(c_T \sqrt{T\gamma_T} + \sum^T_{t=1} \nu_t (c_t+c_t'+\sum^M_{i=1}\omega_i d_i)\right)\\
&= \mathcal{O}\Big(\Big(\sum^M_{i=1}\omega_i d_i' \Big) \sum^T_{t=1} \nu_t + 
\sum^T_{t=1} \nu_t \sqrt{\gamma_t}\log t + 
\gamma_T\log T \sqrt{T} \Big).
\end{split}
\end{equation}

Lastly, similar to our analysis of RM-GP-UCB for the case where the $\omega_i$'s change with $t$ (i.e., at the end of Appendix~\ref{app:proof_theorem_1}), when our online meta-weight optimization is used, we simply need to slightly modify the definition of $\psi_t$:
$\psi_t \triangleq 2\nu_t \left( c_t + c_t' + \sum^M_{i=1}\omega_{i,t} d_i' \right)$ by allowing $\omega_{i,t}$ to change with $t$, and the subsequent analysis still holds by simply replacing $\omega_i$ by $\omega_{i,t}$.
As a result, the no-regret guarantee of RM-GP-TS (Theorem~\ref{regret_bound_ts}) still holds (since we can simply upper-bound every $\omega_{i,t}$ by $1$):
\begin{equation*}
\begin{split}
R_T &= \mathcal{O}\Big( \sum^T_{t=1} \nu_t \Big(\sum^M_{i=1}\omega_{i,t} d_i' \Big) + 
\sum^T_{t=1} \nu_t \sqrt{\gamma_t}\log t + 
\gamma_T\log T \sqrt{T} \Big)\\
&=\mathcal{O}\Big(\Big(\sum^M_{i=1} d_i' \Big) \sum^T_{t=1} \nu_t + 
\sum^T_{t=1} \nu_t \sqrt{\gamma_t}\log t + 
\gamma_T\log T \sqrt{T} \Big)\\
&=\widetilde{\mathcal{O}}\Big(\Big(\sum^M_{i=1} d_i' \Big) \sum^T_{t=1} \nu_t + 
\sum^T_{t=1} \nu_t \sqrt{\gamma_t} + 
\gamma_T \sqrt{T} \Big).
\end{split}
\end{equation*}




\section{Analysis of Online Meta-Weight Optimization}
\label{app:analysis:online:weight}
\subsection{Proof of Lemma~\ref{estimate_di}} 
\label{app:upper_bound_func_gap}
From the definitions of $U_{t,i,j}$ and $L_{t,i,j}$~\eqref{UL}, and the fact that $L_{t,i,j} \leq f(\mathbf{x}_{i,j}) \leq U_{t,i,j}, \forall t, i, j\ $ with probability $\geq 1-\delta/4$ 
(Section~\ref{sec:estimate_d}), we have that
\begin{equation}
\begin{split}
    d_i&=\max_{j=1,...,N_i} \left|f_i(\mathbf{x}_{i,j}) - f(\mathbf{x}_{i,j})\right| \\
    &\leq \max_{j=1,...,N_i}\left[\max \{\left|f_i(\mathbf{x}_{i,j}) - U_{t,i,j}\right|, \left|f_i(\mathbf{x}_{i,j}) - L_{t,i,j}\right|\}\right] \qquad \forall \, i=1,\ldots,M,\forall t\geq 1
\end{split}
\label{func_gap_proof_aux_1}
\end{equation}
which holds with probability $\geq 1 - \delta/4$. Next, we derive upper bounds on $\left|f_i(\mathbf{x}_{i,j}) - U_{t,i,j}\right|$ and $\left|f_i(\mathbf{x}_{i,j}) - L_{t,i,j}\right|$ that only consist of known or computable 
terms, such that the upper bounds on $d_i$ can be efficiently calculated in practice.
\begin{lemma}
\label{func_gap_proof_aux_2}
With probability $\geq 1 - \delta/4$, $\forall \, t\geq 1$, $\forall i, j$, 
\begin{equation*}
\begin{split}
\left|f_i(\mathbf{x}_{i,j}) - U_{t,i,j}\right| \leq \sqrt{2\sigma^2\log\frac{8\sum^M_{i=1}N_i}{\delta}} + \left|y_{i,j} - U_{t,i,j}\right|,\\
\left|f_i(\mathbf{x}_{i,j}) - L_{t,i,j}\right| \leq \sqrt{2\sigma^2\log\frac{8\sum^M_{i=1}N_i}{\delta}} + \left|y_{i,j} - L_{t,i,j}\right|.
\end{split}
\end{equation*}
\end{lemma}
\begin{proof}
To begin with, note that $f_i(\mathbf{x}_{i,j}) - y_{i,j} \sim \mathcal{N}(0, \sigma^2)$. Therefore,~\eqref{standard_gaussian} suggests that
\begin{equation}
\mathbb{P}\left(\left|f_i(\mathbf{x}_{i,j}) - y_{i,j}\right| > \sigma\sqrt{2\log\frac{8\sum^M_{i=1}N_i}{\delta}}\right) \leq \frac{\delta}{8\sum^M_{i=1}N_i}
\end{equation}
which naturally leads to a high-probability upper bound on $\left|f_i(\mathbf{x}_{i,j}) - U_{t,i,j}\right|$:
\begin{equation}
\begin{split}
    \left|f_i(\mathbf{x}_{i,j}) - U_{t,i,j}\right| &= |f_i(\mathbf{x}_{i,j}) - y_{i,j} + y_{i,j} - U_{t,i,j}| \\
    &\leq \left|f_i(\mathbf{x}_{i,j}) - y_{i,j}\right| + \left|y_{i,j} - U_{t,i,j}\right|\\
    &\leq \sqrt{2\sigma^2\log\frac{8\sum^M_{i=1}N_i}{\delta}} + \left|y_{i,j} - U_{t,i,j}\right|
\end{split}
\end{equation}
which holds with probability $\geq 1 - \frac{\delta}{8\sum^M_{i=1}N_i}$. Applying the same reasoning to $\left|f_i(\mathbf{x}_{i,j}) - L_{t,i,j}\right|$ results in a similar high-probability upper bound:
\begin{equation}
    \left|f_i(\mathbf{x}_{i,j}) - L_{t,i,j}\right| \leq \sqrt{2\sigma^2\log\frac{8\sum^M_{i=1}N_i}{\delta}} + \left|y_{i,j} - L_{t,i,j}\right|.
\end{equation}
Next, the proof is completed by taking a union bound over both $U_{t,i,j}$ and $L_{t,i,j}$, as well as all $\sum^M_{i=1}N_i$ observations of the meta-tasks.
\end{proof}

Finally, Lemma~\ref{estimate_di} follows by combining~\eqref{func_gap_proof_aux_1} and Lemma~\ref{func_gap_proof_aux_2}.

\subsection{Proof of Proposition \ref{regret_bound_2}}
\label{app:prop_1_proof}
In iteration $t$, define $\overline{\alpha}_t$ by replacing $d_i$ in $\alpha$ with $\overline{d}_{i,t}$:
\begin{equation}
\overline{\alpha}_t= \sum^M_{i=1}\omega_i \frac{N_i}{\sigma^2}(2\sqrt{2\sigma^2\log\frac{8N_i}{\delta}}+\overline{d}_{i,t}).
\label{alpha_t_bar}
\end{equation}
Since according to Lemma~\ref{estimate_di}, $d_i \leq \overline{d}_{i,t} \, \forall i=1,\ldots,M, t\geq 1$ with probability $\geq 1 - \delta/2$, 
we have that $\alpha \leq \overline{\alpha}_t \, \forall t\geq 1$, which also holds with probability $\geq 1 - \delta/2$.\\

Therefore, Theorem~\ref{regret_bound} implies that, with probability $\geq 1 - \delta$,
\begin{equation}
\begin{split}
    R_T \leq \underline{2\sum^T_{t=1} \overline{\alpha}_t\nu_t} + 2\tau \sum^T_{t=1} \nu_t + \beta_T\sqrt{C_1 T \gamma_T}.
\end{split}
\label{R_T_new}
\end{equation}

In~\eqref{R_T_new}, only the underlined term depends on the $\omega_i$'s. 
Define two column vectors $\overline{\boldsymbol{\alpha}}=[\overline{\alpha}_{t}]^{\top}_{t=1,\ldots,T}$ and $\boldsymbol{\nu}=[\nu_t]^{\top}_{t=1,\ldots,T}$.
Then, the underlined term in~\eqref{R_T_new} can be further decomposed as
\begin{equation}
\begin{split}
    2 \sum^T_{t=1} \overline{\alpha}_{t} \nu_t &\triangleq 2 \overline{\boldsymbol{\alpha}}^{\top}\boldsymbol{\nu} \stackrel{\text{(a)}}{\leq} 2\norm{\overline{\boldsymbol{\alpha}}}_2 \norm{\boldsymbol{\nu}}_2 \stackrel{\text{(b)}}{\leq} 2 \norm{\overline{\boldsymbol{\alpha}}}_1 \norm{\boldsymbol{\nu}}_1 \stackrel{\text{(c)}}{=} 2 \underline{\sum^T_{t=1} \overline{\alpha}_{t}} \sum^T_{t=1}\nu_t
\end{split}
\label{R_T_new_aux_1}
\end{equation}
in which (a) results from Cauchy-Schwarz inequality, (b) follows because the L2 norm is upper-bounded by the L1 norm, and (c) is obtained because $\overline{\alpha}_t>0,\nu_t\geq 0, \forall t\geq 1$.

In~\eqref{R_T_new_aux_1}, the dependence on the $\omega_i$'s appears in the underlined term, which can be further decomposed as
\begin{equation}
\begin{split}
    \sum^T_{t=1} \overline{\alpha}_{t} &=  \sum^T_{t=1}\left[\sum^M_{i=1}\omega_i\frac{N_i}{\sigma^2}\left(2\sqrt{2\sigma^2\log\frac{8N_i}{\delta}}+\overline{d}_{i,t} \right)\right]\\
    &\stackrel{\triangle}{=}\frac{1}{\sigma^2} \sum^T_{t=1}\left[ \sum^M_{i=1}\omega_i l_{i,t} \right] \\
    &\stackrel{\triangle}{=} \frac{1}{\sigma^2}\sum^T_{t=1}\boldsymbol{\omega}^{\top} \boldsymbol{l}_t
\end{split}
\label{R_T_new_aux_2}
\end{equation}
in which we have defined $\boldsymbol{\omega} \triangleq [\omega_i]_{i=1,\ldots,M}$, $\boldsymbol{l}_t\triangleq[l_{i,t}]_{i=1,\ldots,M}$, with
\begin{equation}
    l_{i,t} \triangleq N_i \left(2\sqrt{2\sigma^2\log\frac{8N_i}{\delta}}+\overline{d}_{i,t}\right).
\end{equation}

Plugging~\eqref{R_T_new_aux_1} and~\eqref{R_T_new_aux_2} in to~\eqref{R_T_new} completes the proof.

\subsection{Derivation of Equation~\ref{estimate_wi}}
\label{lagran}
Recall that our objective is to minimize 
\[
    \sum^{t-1}_{s=1}\boldsymbol{\omega}'^{\top} \boldsymbol{l}_s + \frac{1}{\eta}\sum^M_{i=1}\omega_i'\log \omega_i'
\]
subject to the constraint that $\boldsymbol{\omega}'$ forms a probability simplex: $\sum^M_{i=1}\omega_i' = 1.0$
and $\omega_i' \geq 0$ for all $i=1,\ldots,M$.
Define the Lagrangian as 
\begin{equation}
L(\boldsymbol{\omega}, \lambda) = \sum^{t-1}_{s=1}\boldsymbol{\omega}'^{\top} \boldsymbol{l}_s + \frac{1}{\eta}\sum^M_{i=1}\omega_i'\log \omega_i' + 
\lambda \left(1 - \sum^M_{i=1}\omega_i'\right).
\end{equation}
Taking the derivative of $L(\boldsymbol{\omega}, \lambda)$ with respect to $\omega_i'$, we get
\begin{equation}
\label{L_derive}
\frac{\partial L(\boldsymbol{\omega}, \lambda)}{\partial \omega_i'} = \sum^{t-1}_{s=1} l_{i,s} + \frac{1}{\eta} \left( \log \omega_i' + 1 \right) - \lambda.
\end{equation}
Setting~\eqref{L_derive} to 0 gives us
\begin{equation}
\omega_i' = e^{\eta \lambda - 1} e^{-\eta \sum^{t-1}_{s=1} l_{i,s} } \propto e^{-\eta \sum^{t-1}_{s=1} l_{i,s} }.
\end{equation}
Normalizing the $\omega_i'$'s for all $i=1\ldots,M$ to form a probability simplex leads to~\eqref{estimate_wi}.

\subsection{Analysis for RM-GP-TS}
\label{app:meta:weight:optimization:ts}
Here we use the function gap $d_i$ to approximate $d_i'$ (defined in Theorem~\ref{regret_bound_ts}), 
i.e., $d_i'\approx d_i,\forall i=1,\ldots,M$.
Combining Lemma~\ref{estimate_di} and Theorem~\ref{regret_bound_ts}, we have for RM-GP-TS that with probability of $\geq 1 - \delta$,
\begin{equation}
\begin{split}
R_T &= \mathcal{O}\Big( \sum^T_{t=1} \nu_t \left(\sum^M_{i=1}\omega_i \overline{d}_{i,t} \right) + 
\sum^T_{t=1} \nu_t \sqrt{\gamma_t}\log t + 
\gamma_T\log T \sqrt{T} \Big)\\
&\leq \mathcal{O}\Big(  \underline{\left(\sum^T_{t=1}\sum^M_{i=1}\omega_i \overline{d}_{i,t}\right)}  \left(\sum^T_{t=1} \nu_t\right) + 
\sum^T_{t=1} \nu_t \sqrt{\gamma_t}\log t + 
\gamma_T\log T \sqrt{T} \Big),
\end{split}
\label{eq:ts:online:regret:minimize:eq:1}
\end{equation}
in which the inequality can be proved in a similar way as equation~\eqref{R_T_new_aux_1}.
Next, define $\boldsymbol{\omega}\triangleq [\omega_i]_{i=1,\ldots,M}$, $\boldsymbol{d}_t \triangleq [\overline{d}_{i,t}]_{i=1,\ldots,M}$, then the underlined term above can be denoted as:
\begin{equation}
\sum^T_{t=1} \sum^M_{i=1}\omega_i \overline{d}_{i,t} = \sum^T_{t=1} \boldsymbol{\omega}^{\top} \boldsymbol{d}_t.
\end{equation}
Therefore, equation~\eqref{eq:ts:online:regret:minimize:eq:1} can be further upper-bounded as:
\begin{equation}
\begin{split}
R_T = \mathcal{O}\Big(  \underline{\left( \sum^T_{t=1} \boldsymbol{\omega}^{\top} \boldsymbol{d}_t \right)}  \left(\sum^T_{t=1} \nu_t\right) + 
\sum^T_{t=1} \nu_t \sqrt{\gamma_t}\log t + 
\gamma_T\log T \sqrt{T} \Big).
\end{split}
\end{equation}
Next, applying similar derivations as Appendix~\ref{lagran} (treating the underlined term above as the loss to be minimized) leads to the same update rule for the meta-weights as equation~\eqref{estimate_wi}.
Approximating $d_i'$ using $d_i$ also allows us to derive the same update rule for $\nu_t$ (Sec.~\ref{sec:online_weight_estimation}).




\section{More Experimental Details and Results}
\label{app:experiments}
In every experiment, the same set of random initializations are used for all methods
to ensure fair comparisons.
The kernel bandwidth parameter $\rho$ in TAF is set to $\rho=0.5$ in all experiments, but we have observed that other values of $\rho$ (such as $0.1$ and $0.9$)
lead to similar performances. $S=500$ posterior samples are used to compute the ensemble weights in RGPE.
All experiments are run on a server with 16 cores of Intel Xeon processor, 256G of RAM and 5 NVIDIA GTX1080 Ti GPUs.

\subsection{Optimization of Synthetic Functions}
\label{app:synth}
\subsubsection{Synthetic Functions Sampled from GPs}
The objective functions are drawn from GP's with the Squared Exponential kernel (with a length scale of $0.05$) from the domain $\mathcal{D}=[0,1]$.
Fig.~\ref{fig:synth_func} shows an example of such synthetic functions.
\begin{figure}[tb]
\centering
\includegraphics[width=0.45\columnwidth]{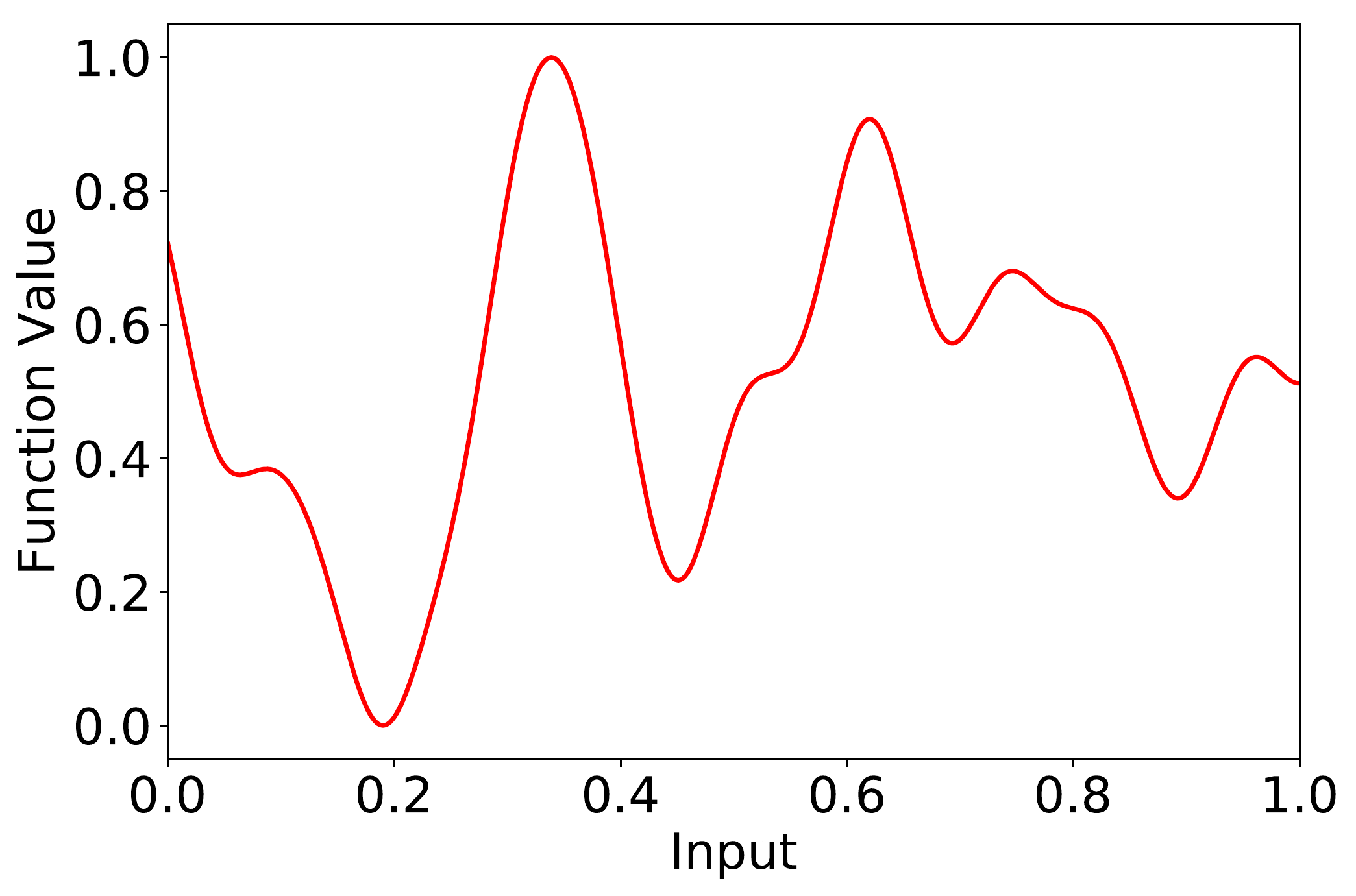}
\caption{An example synthetic function sampled from a GP.}
\label{fig:synth_func}
\end{figure}
The meta-functions and meta-tasks are generated in the following way. 
To begin with, we fix the number of meta-tasks $M=4$, the number of observations (input-output pairs) for each meta-task $N=N_i=20$ for $i=1\ldots M$, and the function gaps: $d_1=d_2=0.05$, $d_3=d_4=4.0$. 
For the $i$-th meta-task, firstly, $N_i$ inputs are randomly drawn from the entire domain $\mathcal{D}=[0,1]$. 
Then for each of the $N_i$ inputs $\mathbf{x}_{i,j}$, a number is randomly drawn from $[-d_i,d_i]$, 
which is added to the value of the target function $f(\mathbf{x}_{i,j})$ to produce the corresponding function value of the meta-function $f_i(\mathbf{x}_{i,j})$. 
Subsequently, a zero-mean Gaussian noise (with a noise variance of $0.01$) is added to $f_i(\mathbf{x}_{i,j})$, resulting in the corresponding output of the meta-observation $y_i(\mathbf{x}_{i,j})$. 
The above-mentioned procedure is repeated for each of the $M=4$ meta-tasks.
Note that according to the specified function gaps, meta-tasks 1 and 2 are relatively more similar to the target task, 
whereas meta-tasks 3 and 4 are dissimilar to the target task due to the larger function gaps. 

Fig.~\ref{fig:meta_weights_curves} plots the evolution of the meta-weights for each of the $4$ meta-tasks in the experiments exploring the impact of $\eta$,
i.e., corresponding to Fig.~\ref{fig:synth_func_results}c in Section~\ref{exp:synth}. These figures are used to demonstrate the observations that overly large and excessively small
values of $\eta$ can both degrade the performance of RM-GP-UCB.
\begin{figure}[tb]
	\centering
	\begin{subfigure}[b]{0.325\linewidth}
		\includegraphics[width=\linewidth]{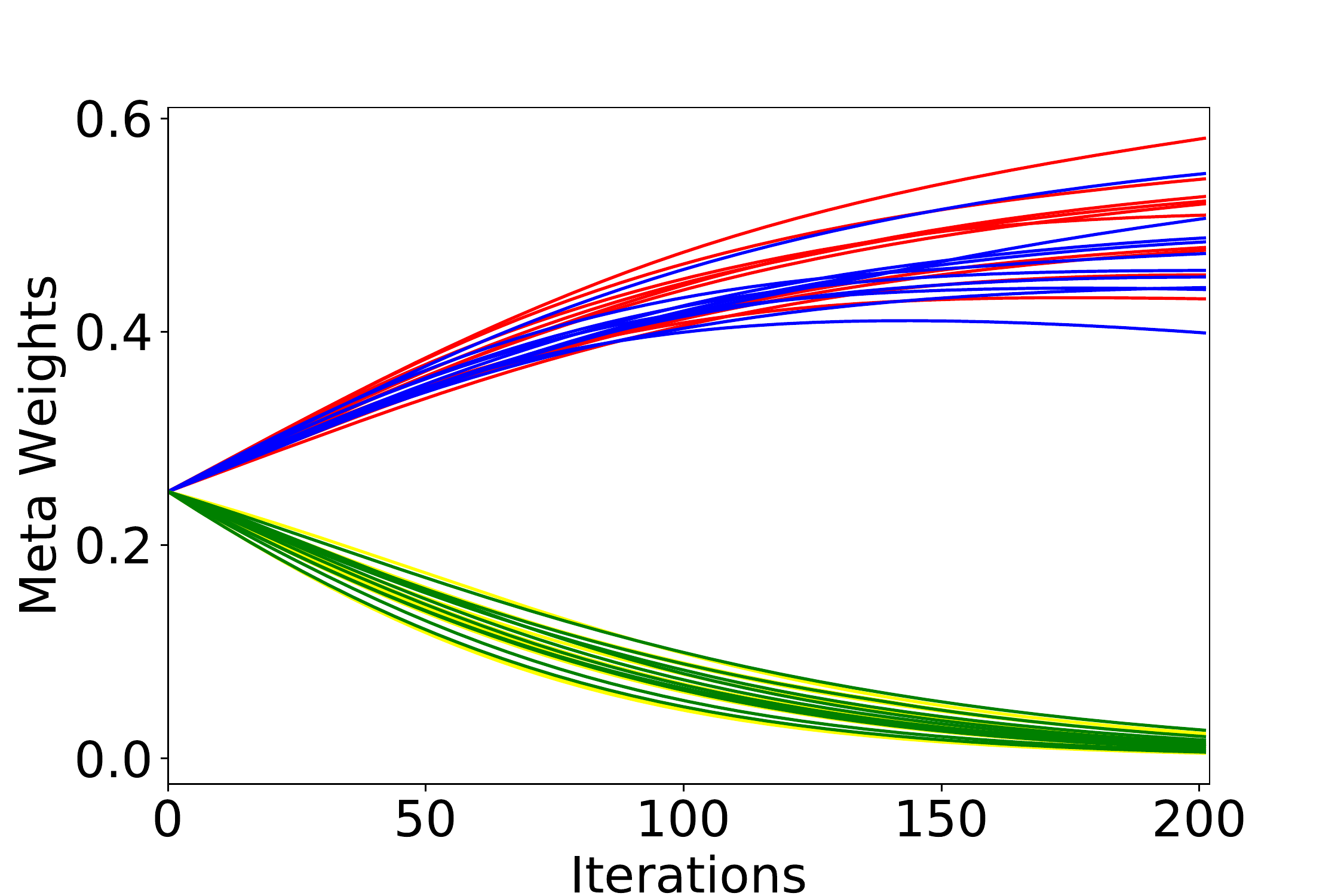}
		\caption{$\eta N=0.01$.}
	\end{subfigure}
	\hfill
	\begin{subfigure}[b]{0.325\linewidth}
		\includegraphics[width=\linewidth]{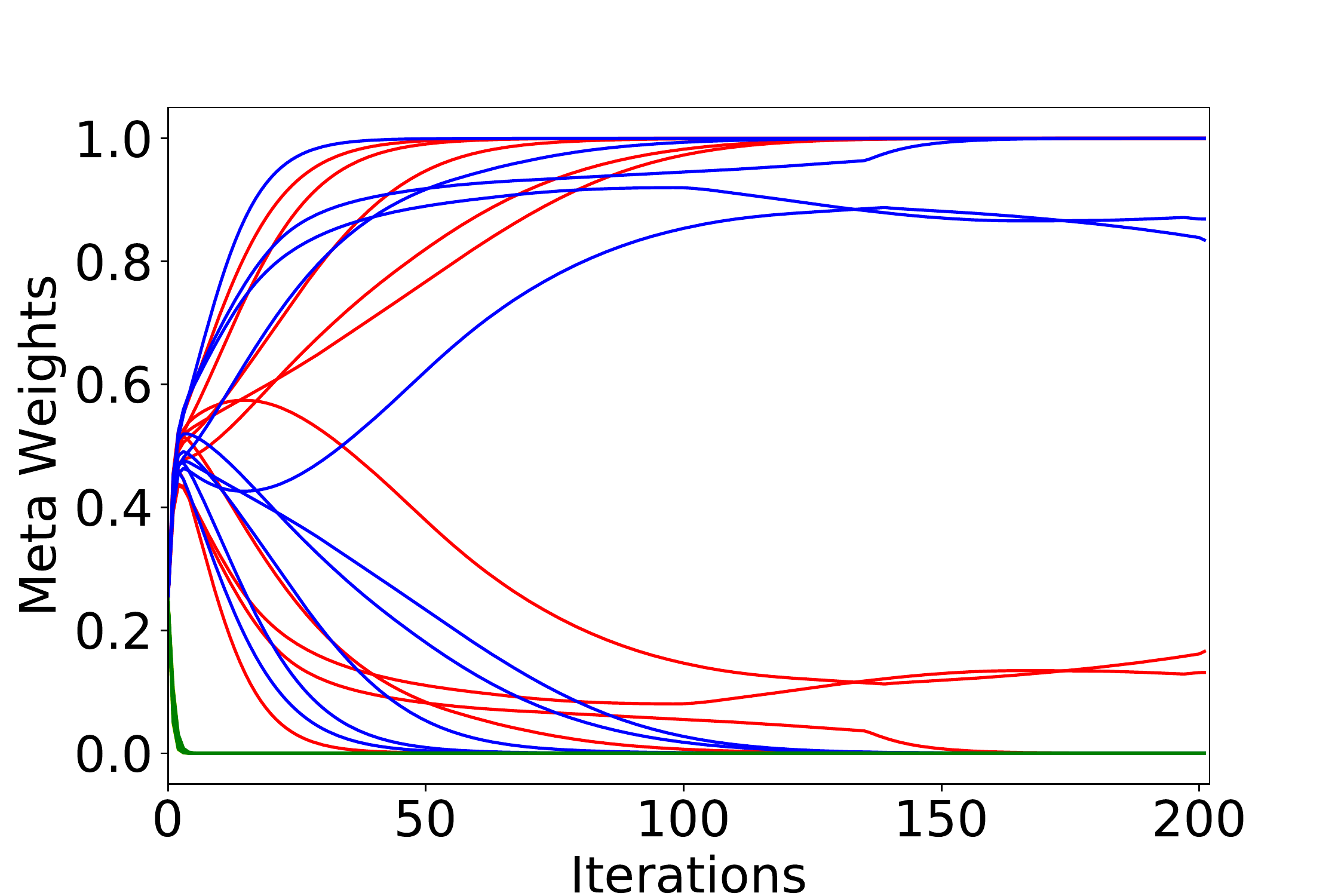}
		\caption{$\eta N=1.0$.}
	\end{subfigure}
	\hfill
	\begin{subfigure}[b]{0.325\linewidth}
		\includegraphics[width=\linewidth]{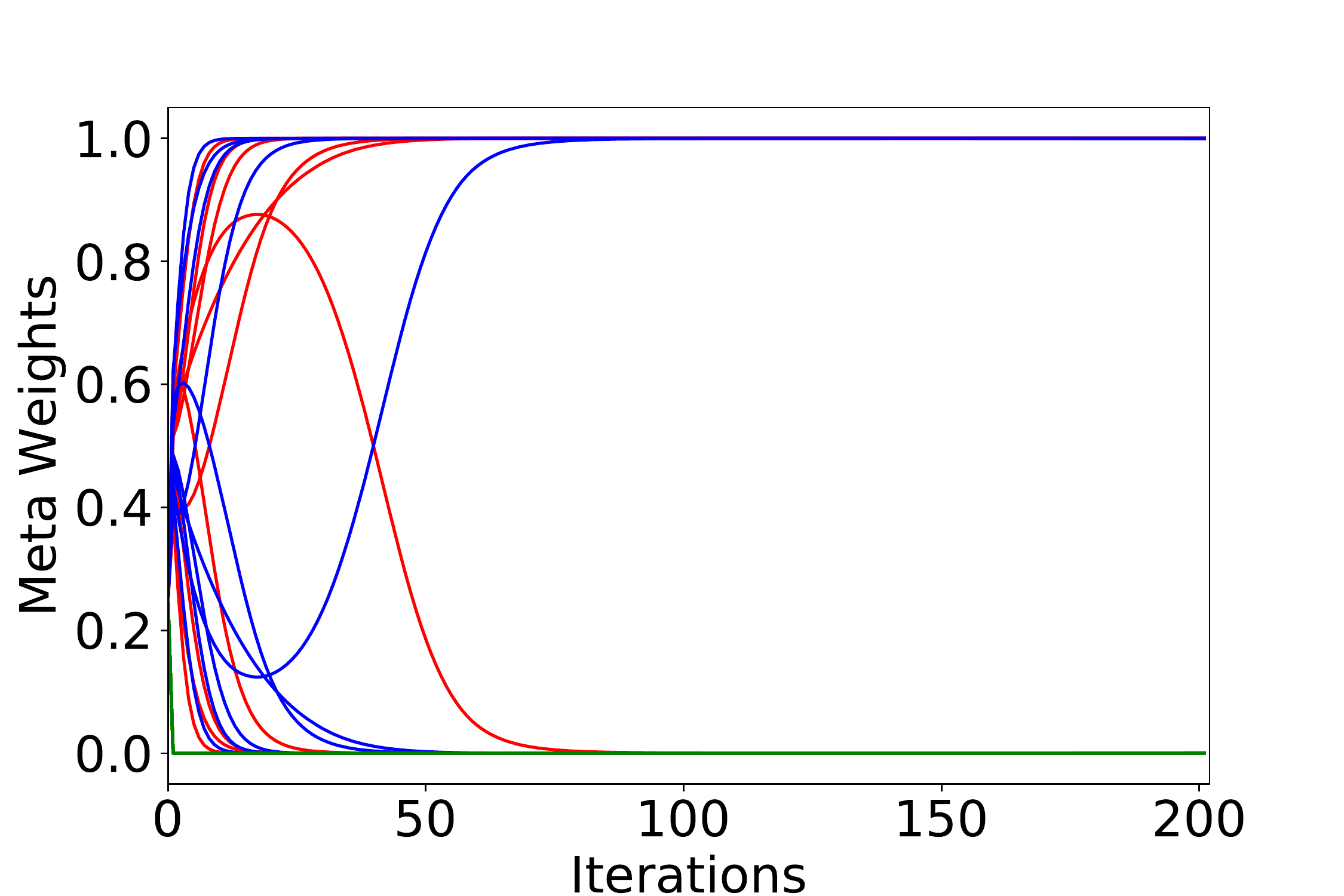}
		\caption{$\eta N=5.0$.}
	\end{subfigure}
	\caption{Evolution of the meta-weights with different learning rate, $\eta$, for online meta-weight optimization in the synthetic experiments.
	In each figure, the red and blue curves represent the meta-weights of the two meta-tasks that are more similar to the target task (i.e., the first two meta-tasks),
	whereas the green and yellow curves correspond to the meta-weights of the other two dissimilar meta-tasks.
	Every color has $10$ curves in each figure, which correspond to $10$ independent runs of the algorithm with different random initializations.}
	\label{fig:meta_weights_curves}
\end{figure}

Moreover, we have added another experiment where the $N_i$'s (i.e., the number of observations from the meta-tasks) are different.
Specifically, we use the same experimental setting involving $M=4$ meta-tasks as described above, and let $N_1=15,N_2=25,N_3=10,N_4=30$, where $d_1=d_2=0.05$, $d_3=d_4=4.0$.
The results (Fig.~\ref{fig:synth_func:with:diff:Ni}) show that when the $N_i$'s are different, our RM-GP-UCB algorithm, despite performing worse than the setting where all $N_i$'s are equal, is still able to significantly outperform standard GP-UCB.
\begin{figure}
\centering
\includegraphics[width=0.45\columnwidth]{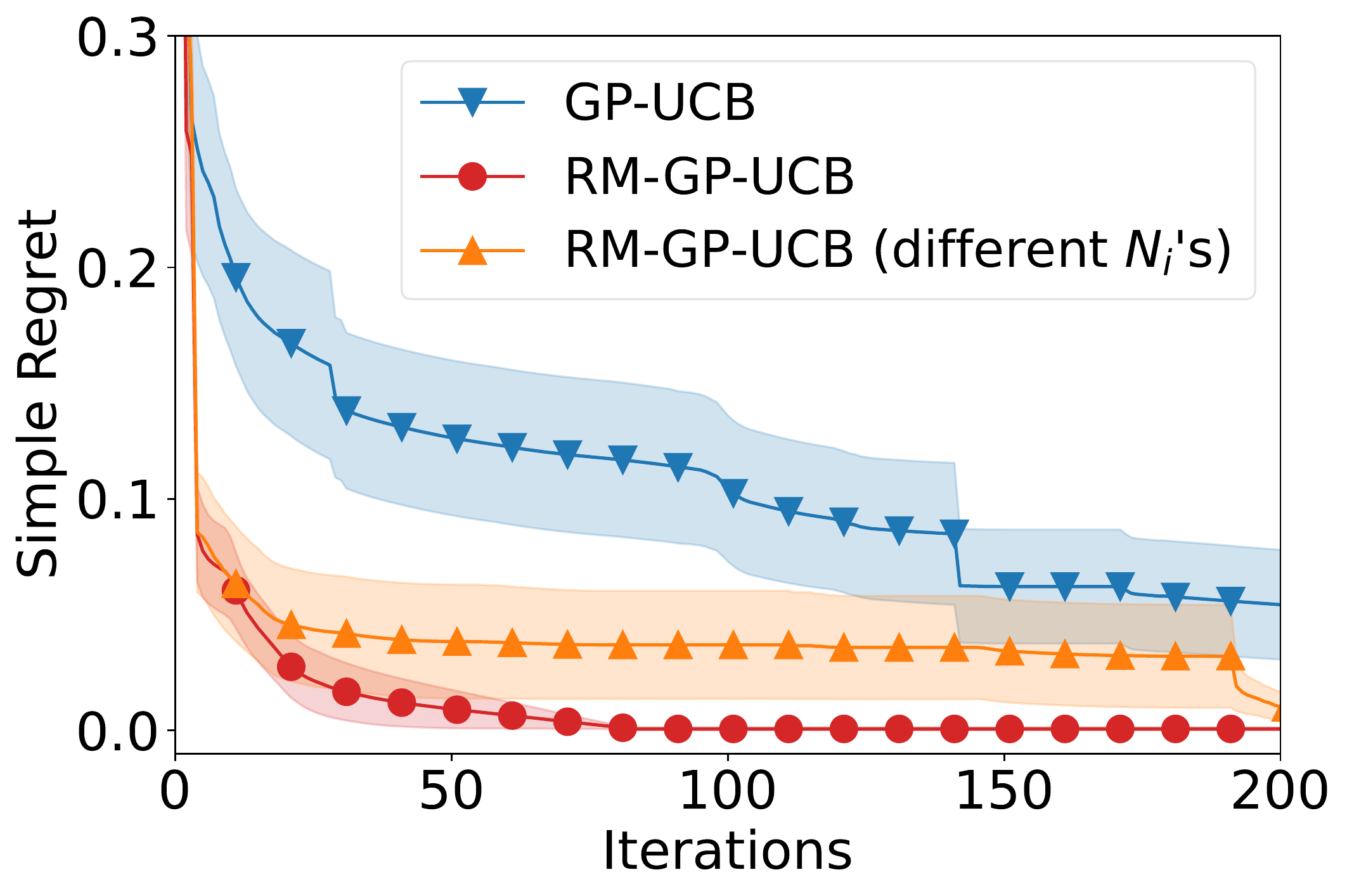}
\caption{The performance of RM-GP-UCB when the $N_i$'s are different.}
\label{fig:synth_func:with:diff:Ni}
\end{figure}


\subsection{Real-world Experiments}
\label{app:auto_ml}

\textbf{Hyperparameter Tuning for Convolutional Neural Networks (CNNs).}
The MNIST, CIFAR-10 and CIFAR-100 datasets can all be directly downloaded using the Keras Python package\footnote{\url{https://keras.io/}},
and the SVHN dataset can be downloaded from \url{http://ufldl.stanford.edu/housenumbers/}.
The MNIST dataset is under the GNU General Public License, CIFAR-10 adn CIFAR-100 are under the MIT License, and SVHN is under the Custom (non-commercial) License.
The image pixel values are all normalized into the range $[0, 1]$.
The CNN hyperparameters being optimized in this set of experiments are the learning rate, learning rate decay, and the L2 regularization parameter, 
all of which have the search space from $10^{-7}$ to $10^{-2}$.
Other than these hyperparameters, a common CNN architecture is used for all datasets, i.e., a CNN containing two convolutional layers (both with 32 filters and each filter has a size of $3\times 3$)
each of which is followed by a Max pooling layer (with a pooling size of $3\times 3$), 
followed by two fully connected layers (both with $64$ hidden units); all non-linear activations are ReLU.
The size of the training set and validation set for the four datasets are: 60,000/10,000 for MNIST, 73,257/26,032 for SVHN, 50,000/10,000 for both CIFAR-10 and CIFAR-100.
For the evaluation of a set of selected hyperparameters, the CNN model is trained using the RMSprop algorithm for $20$ epochs, and the 
final validation error is used as the corresponding output observation.
Fig.~\ref{fig:cnn_2} presents the results when the SVHN and CIFAR-100 datasets are used to produce the target functions.

Comparing Figs.~\ref{fig:synth_func_results}e,~\ref{fig:synth_func_results}f and Fig.~\ref{fig:cnn_2} shows that our RM-GP-UCB performs similarly to RGPE for the CIFAR-10, CIFAR-100 and SVHN datasets, and outperforms RGPE for MNIST.
After inspection, we found that this is because for the first three datasets (Fig.~\ref{fig:synth_func_results}f and Fig.~\ref{fig:cnn_2}), both RM-GP-UCB and RGPE assign most meta-weights to the same meta-task. 
On the other hand, for MNIST (Fig.~\ref{fig:synth_func_results}e), RM-GP-UCB (and RM-GP-TS) is able to assign most weights to SVHN which is indeed more similar to MNIST since they both contain images of digits.
In contrast, RGPE mistakenly assigns more meta-weights to CIFAR-10. The reason is that RGPE chooses the weights based on how accurately each meta-task’s GP surrogate predicts the pairwise ranking of the target observations (more details in Sec.~\ref{sec:related_works}, second paragraph). However, for MNIST, most target observations have very similar values since the overall accuracy is very high due to the simplicity of the MNIST dataset. Therefore, the predicted pairwise rankings become unreliable, thus rendering the weights learned by RGPE inaccurate and deteriorating the performance.

\begin{figure}[tb]
	\centering
	\begin{subfigure}[b]{0.45\linewidth}
		\includegraphics[width=\linewidth]{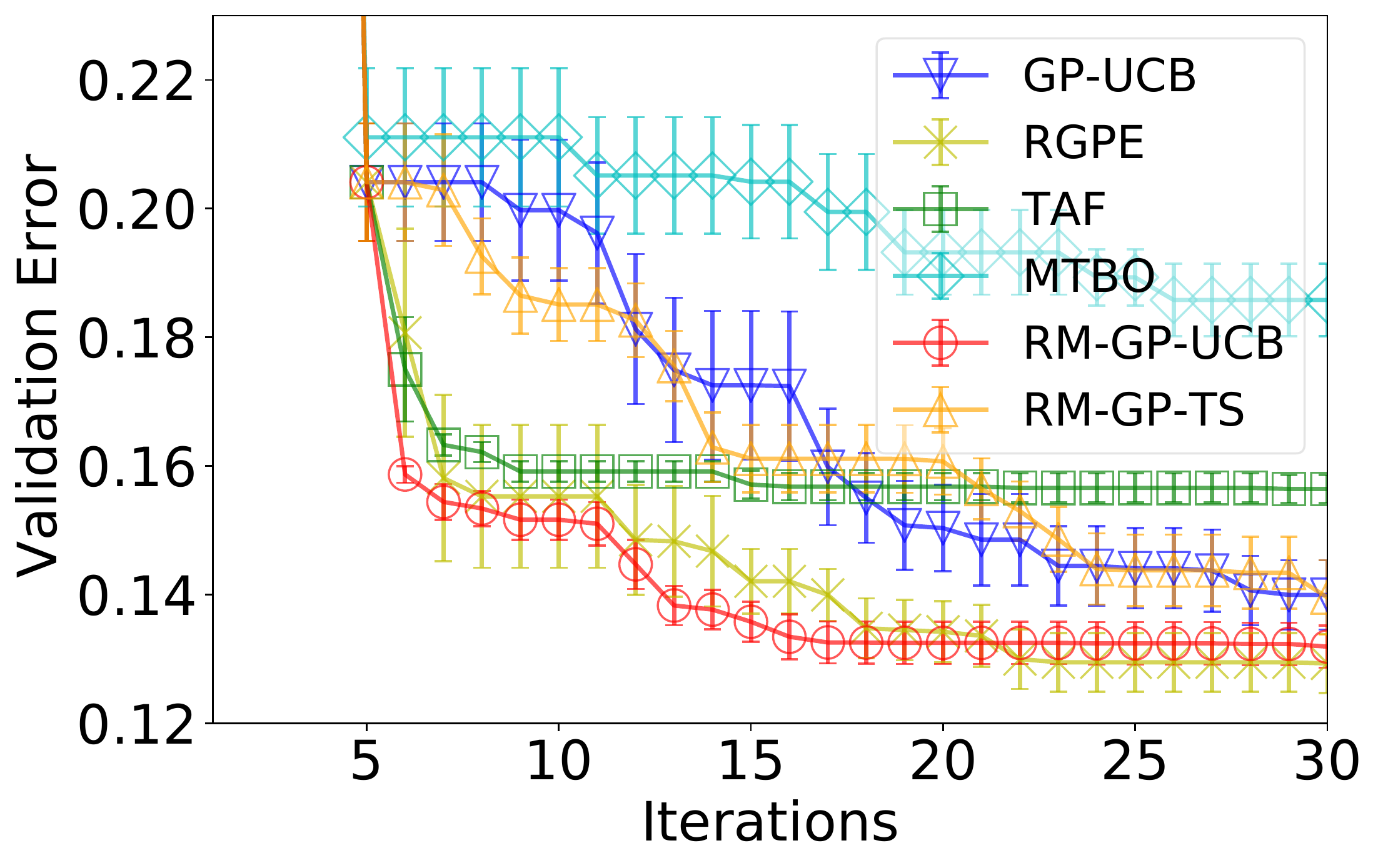}
		\caption{SVHN.}
	\end{subfigure}
	\hfill
	\begin{subfigure}[b]{0.45\linewidth}
		\includegraphics[width=\linewidth]{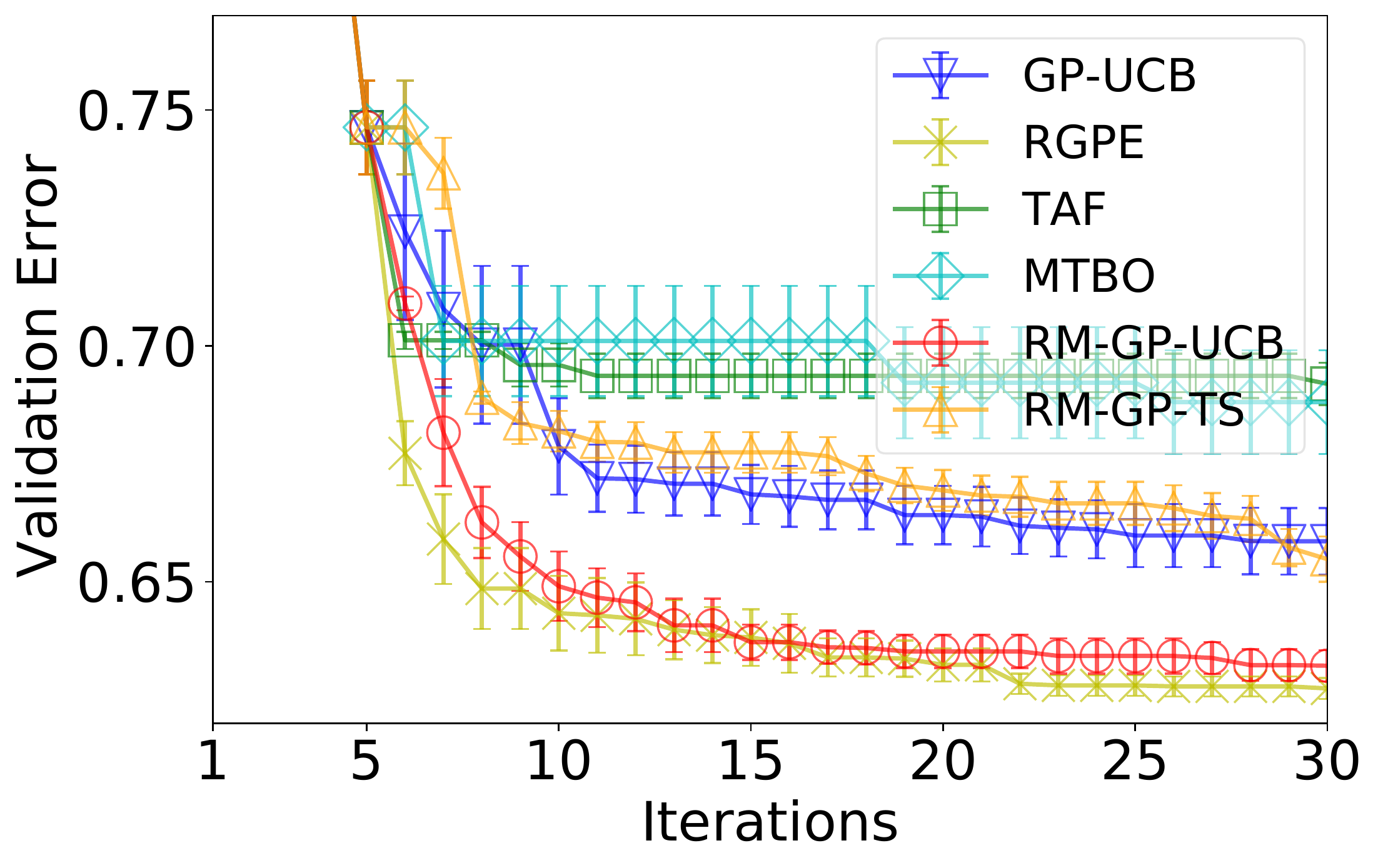}
		\caption{CIFAR-100.}
	\end{subfigure}
	\caption{Best validation error of CNN (both averaged over 10 random initializations).}
	\label{fig:cnn_2}
\end{figure}

\textbf{Hyperparameter Tuning for CNNs Using the Omniglot Dataset.}
The Omniglot dataset can be downloaded from \url{https://github.com/brendenlake/omniglot}, and it is under the MIT License.
The dataset consists of $50$ alphabets, $30$ from the background set and $20$ from the evaluation set. Each alphabet includes a number of characters, and all alphabets combine to have $1623$ characters. Every character only consists of $20$ example images, each drawn by a different person.
To perform one-shot classification, we use a Siamese neural network,
which takes two images as inputs and outputs a score indicating whether the pair of input images are predicted to be the same character.
The evaluation metric we use in the experiment is 2-way validation error. That is, we compare a test image in the validation set with two other images, only one of which is the same character as the test image, and evaluate whether the Siamese network is able to output a higher predictive score for the correct image which is the same character; we do this using every test image, and use the percentage of errors as the 2-way validation error.
In our setting, each task represents tuning $3$ hyperparameters of the Siamese network (the same hyperparameters and ranges as the CNN experiments above) using one alphabet. For each task, we use $75\%$ of the characters in the alphabet to produce the training set, and the remaining $25\%$ to generate the validation set.
We use $10$ alphabets from the background set as $10$ meta-tasks. For each meta-task, we generate $30$ meta-observations by running BO (using GP-UCB) for $30$ iterations. This in total produces $10 \times 30=300$ meta-observations. We use one of the alphabets from the evaluation set as the target task.

\begin{figure}[tb]
\centering
\includegraphics[width=0.4\columnwidth]{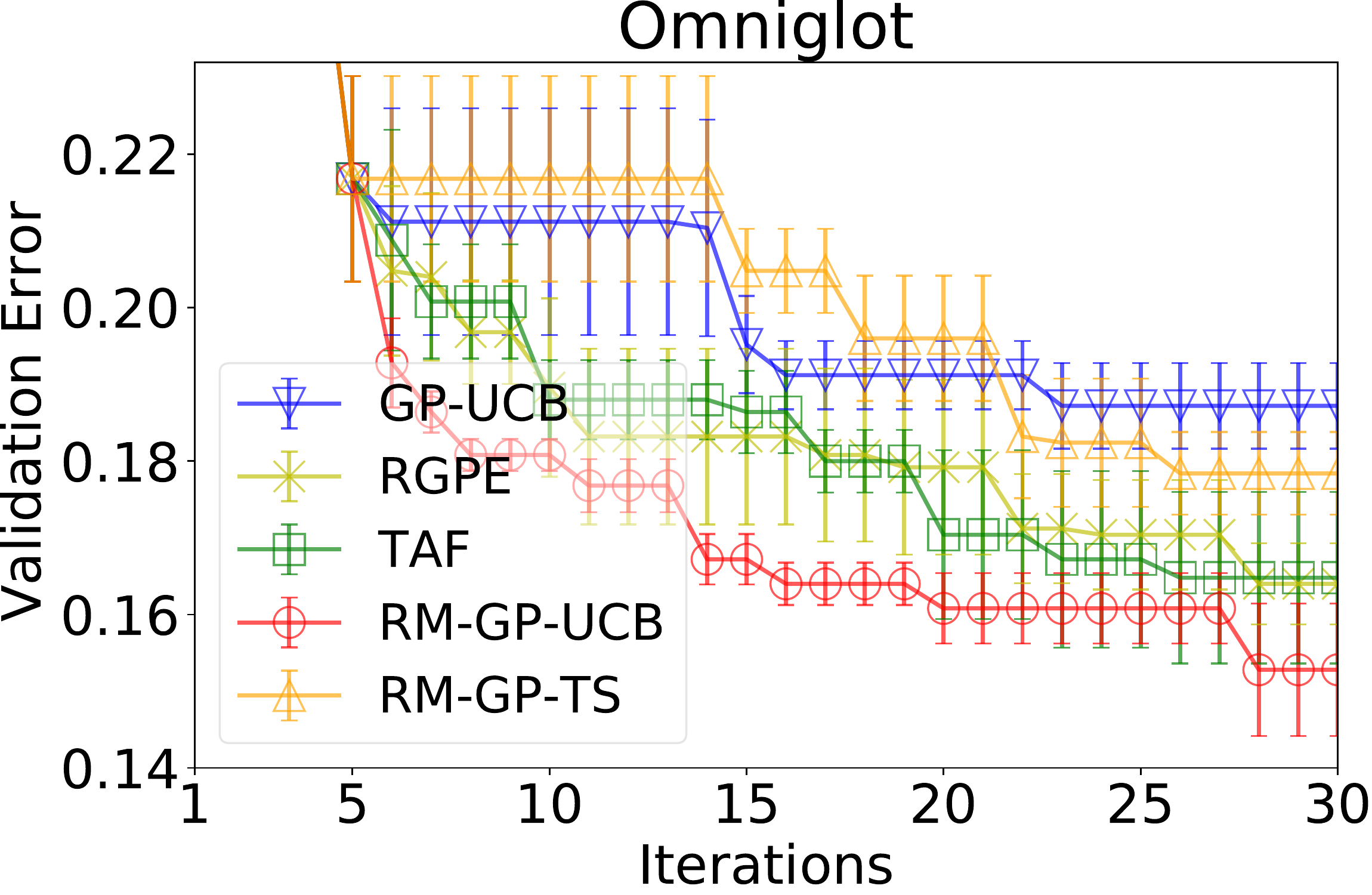}
\caption{2-way validation error on the Omniglot dataset.}
\label{fig:omniglot}
\end{figure}

\textbf{Hyperparameter Tuning for Support Vector Machines (SVMs).}
This benchmark dataset, which was originally introduced by~\citep{wistuba2015learning} and can be downloaded from~\url{https://github.com/wistuba/TST}, is created by performing hyperparameter tuning of SVM using $50$ diverse datasets.
$6$ hyperparameters are tuned: $3$ binary parameters indicating whether a linear, polynomial or radial basis function (RBF) kernel is used, the penalty parameter, the degree of the polynomial kernel, and the bandwidth parameter for the RBF kernel. A fixed grid of hyperparameters of size $288$ is created. For each dataset, every hyperparameter configuration on the grid is evaluated and the corresponding validation accuracy is recorded as the observed output of the objective function.
In our experiments, each dataset corresponds to a task. We treat one of the $50$ tasks as the target task, and the remaining tasks as $49$ meta-tasks. For each meta-task, the meta-observations are produced by randomly sampling $50$ points (hyperparameter configurations) from the grid. The results reported in the main paper (Fig.~\ref{fig:cnn}c) are averaged over $25$ trials, each trial treating a different task as the target task; for each trial/target task, we again average the results over $5$ random initializations.

\textbf{Human Activity Recognition (HAR).}
The dataset used in this experiment can be downloaded from \url{https://archive.ics.uci.edu/ml/datasets/Human+Activity+Recognition+Using+Smartphones}.

In this experiment of human activity prediction, each data instance (input-output pair) is characterized by a feature vector of length 561 and a label corresponding to one of the $6$ activities.
The SVM hyperparameters being optimized are the penalty parameter $C$ (from 0.01 to 10) and the radial basis function (RBF) kernel coefficient $\gamma$ (from 0.01 to 1).
There are in total 7,352 data instances for the 21 subjects that are used to generate the meta-tasks, and 2,947 instances for the 9 subjects used for performance validation.
For each subject, half of the instances are used as the training set, with the other half being used for validation.

\textbf{Non-stationary Bayesian Optimization.}
The clinical diagnosis dataset used in this experiment can be found at \url{https://www.kaggle.com/uciml/pima-indians-diabetes-database}, and it is associated with the CC0 License.
The hyperparameters of the logistic regression (LR) model being optimized are the batch size (20 to 60), 
the L2 regularization parameter ($10^{-6}$ to 0.01) and the learning rate (0.01 to 0.1).
The dataset represents a binary classification problem (whether a patient has diabetes or not), with each input instance consisting of 8 diagnostic features:
number of pregnancies, plasma glucose concentration, blood pressure, skin thickness, insulin, BMI, diabetes pedigree function, and age.

\textbf{Policy Search for Reinforcement Learning.} 
In this experiment, we use the Cart-Pole environment from OpenAI Gym (\url{https://github.com/openai/gym}), which is under the MIT License.
We adopt the linear softmax policy 
which linearly maps a state vector of length 4 
to an action vector of length 2, followed by a softmax operator.
As a result, for a particular state, the action with the largest softmax value is taken.
With this setting, $4\times 2=8$ parameters are tuned in this experiment.
The performance metric used in the experiment is the cumulative rewards (normalized to the range $[0,1]$) in an episode (averaged over $10$ independent episodes), and 
the maximum length of each episode is set to 200.

%

\subsection{Impacts of Max vs Mean in Function Gap Estimation}
\label{app:subsec_max_mean}
Here we explore the impact of the choice between using $\max$ (the outer $\max$ operator over $j=1,...,N_i$) or the empirical mean in the estimated upper bound on the function gap (Lemma~\ref{estimate_di}),
as mentioned in the first paragraph of Section~\ref{sec:experiment}.
Fig.~\ref{fig:inspect_use_max} plots the different performances using these two choices in the MNIST, CIFAR-$10$ and clinical diagnosis (non-stationary BO) experiments.
The results show that the performance deficit resulting from the use of the $\max$ operator is marginal in some experiments (Fig.~\ref{fig:inspect_use_max}a and b),
whereas the difference can be larger in some other experiments (Fig.~\ref{fig:inspect_use_max}c).
Therefore, it is recommended to use the empirical mean when estimating the upper bound on the function gap in practice.
\begin{figure}
	\centering
	\begin{tabular}{ccc}
		\hspace{-3mm} \includegraphics[width=0.328\linewidth]{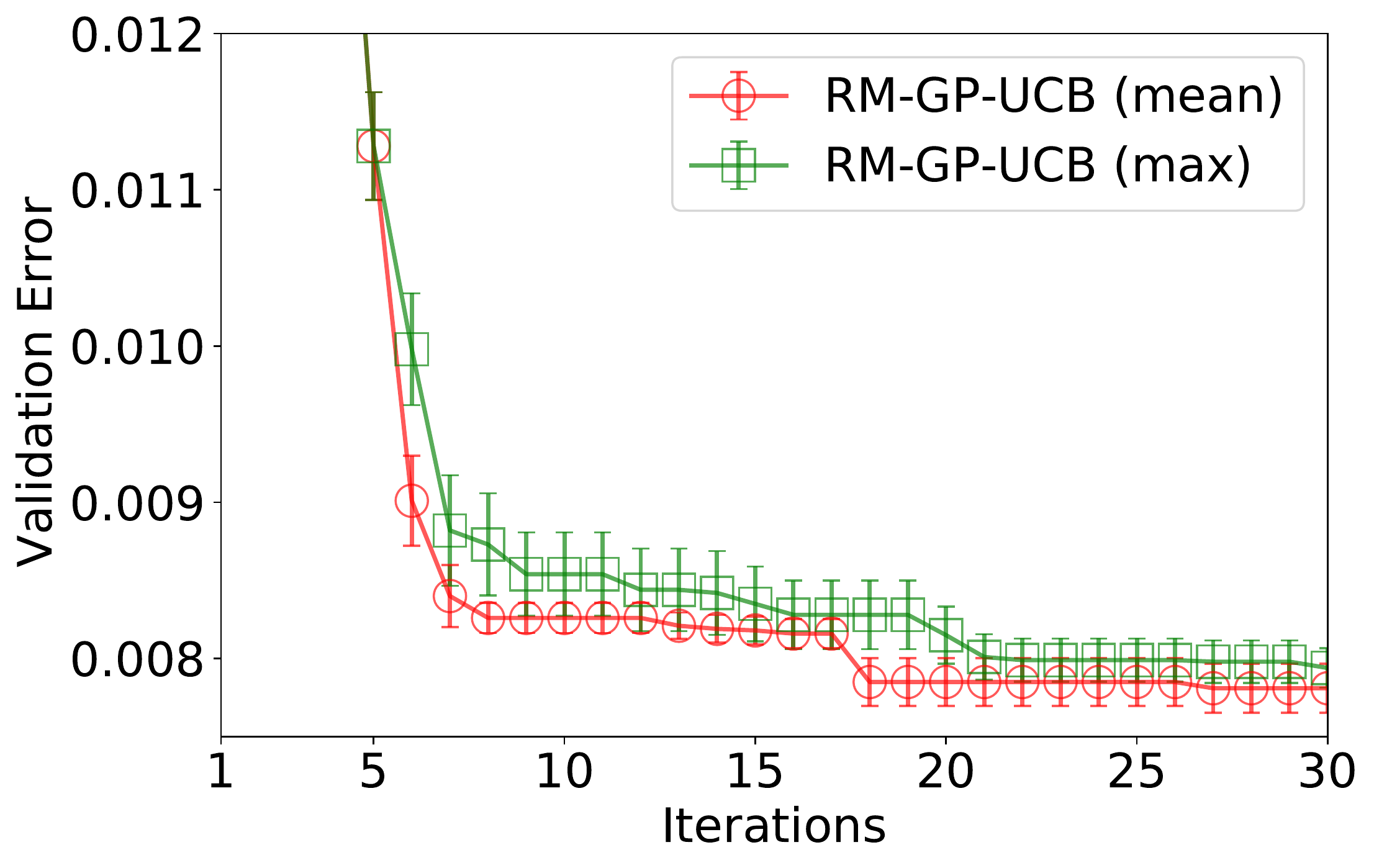} & \hspace{-4mm}
		\includegraphics[width=0.328\linewidth]{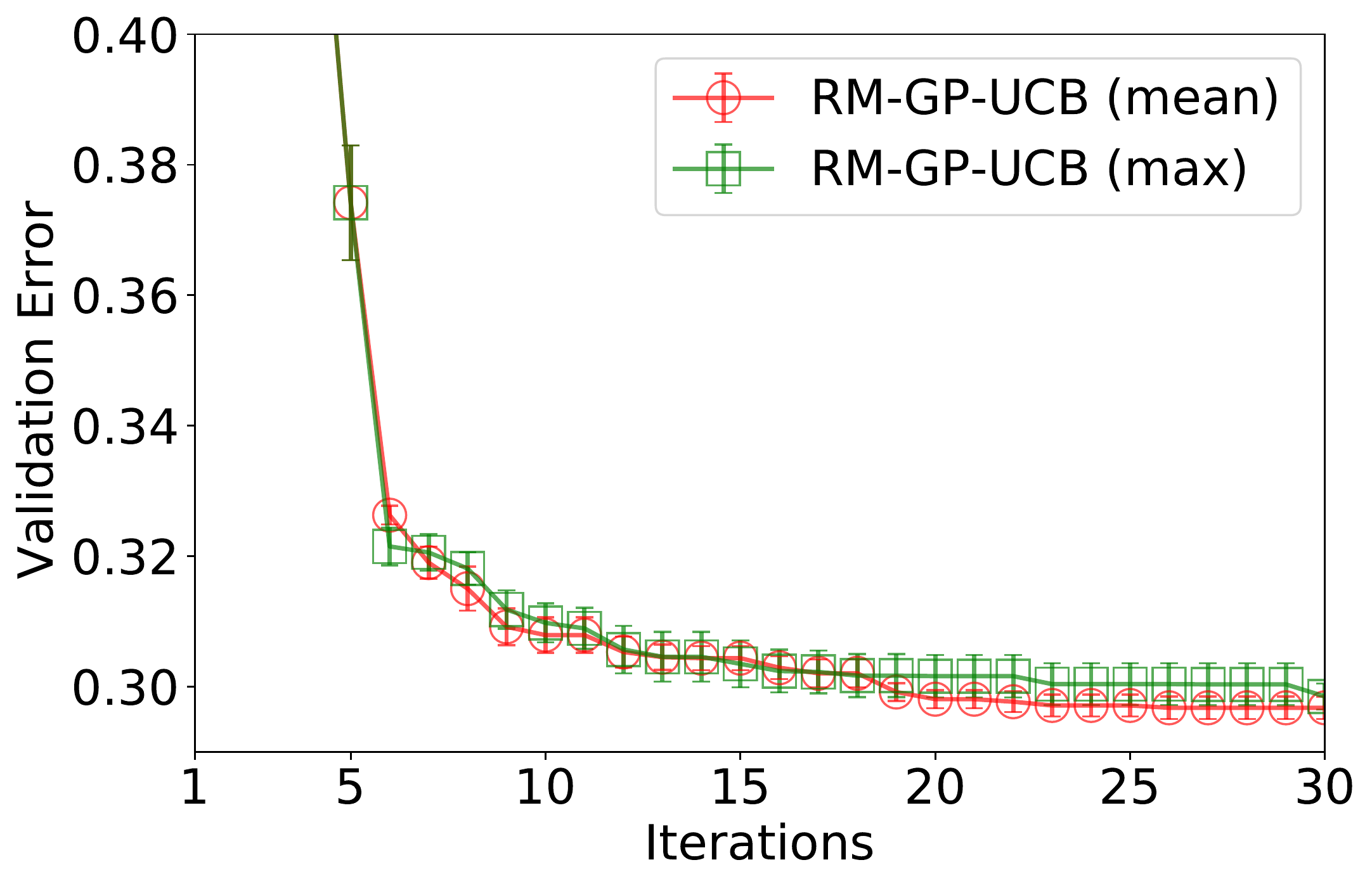} & \hspace{-4mm}
		\includegraphics[width=0.328\linewidth]{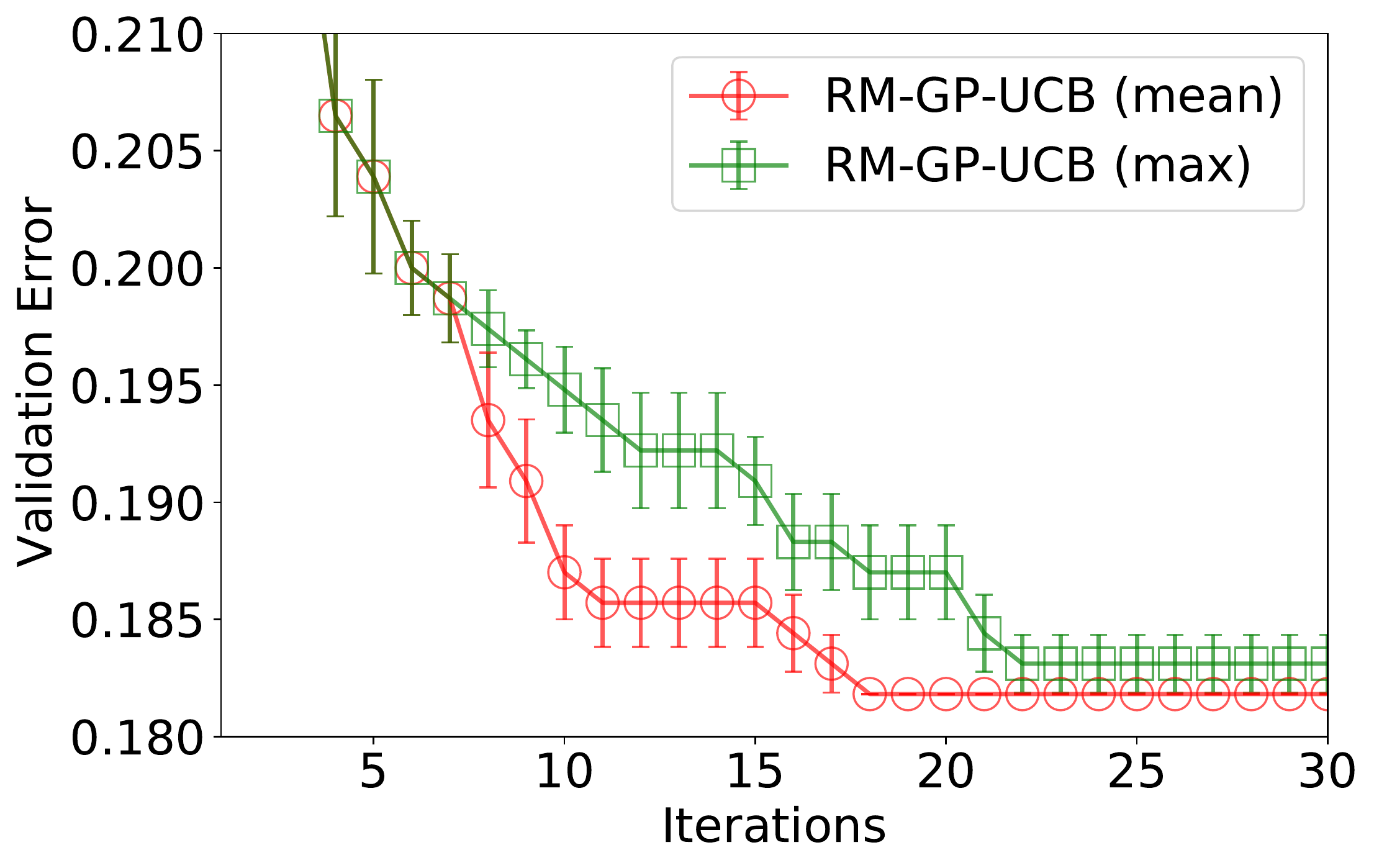}\\
		{(a)} & {(b)} & {(c)}
	\end{tabular}
	\caption{Impacts of using max vs empirical mean in estimating the upper bound on the function gaps, using the (a) MNIST, (b) CIFAR-10 and (c) non-stationary BO (clinical diagnosis) experiments.}
	\label{fig:inspect_use_max}
\end{figure}

\subsection{Scalability of Our Algorithms}
\label{app:scalability}
Here we further demonstrate the scalability of our RM-GP-UCB and RM-GP-TS algorithms.
by showing that our algorithms can be applied to experiments with a very large scale, and still performs competitively.
Specifically, we construct a much larger version of the experiment on policy search for RL, with $60$ meta-tasks each containing $130$ meta-observations. Fig.~\ref{fig:scalability_2}a and b plot the performance and runtime in this large-scale experiment. 
Consistent with Fig.~\ref{fig:cnn}e in the main text, our RM-GP-UCB algorithm still performs the best among all algorithms (Fig.~\ref{fig:scalability_2}a).
RM-GP-TS has a better performance here than in Fig.~\ref{fig:cnn}e, performing comparably with RGPE (Fig.~\ref{fig:scalability_2}a).
Moreover, RM-GP-TS is again significantly more scalable than RM-GP-UCB, RGPE and TAF, and its computational cost is comparable with standard GP-UCB (Fig.~\ref{fig:scalability_2}b).


\begin{figure}[tb]
	\centering
	\begin{subfigure}[b]{0.4\linewidth}
		\includegraphics[width=\linewidth]{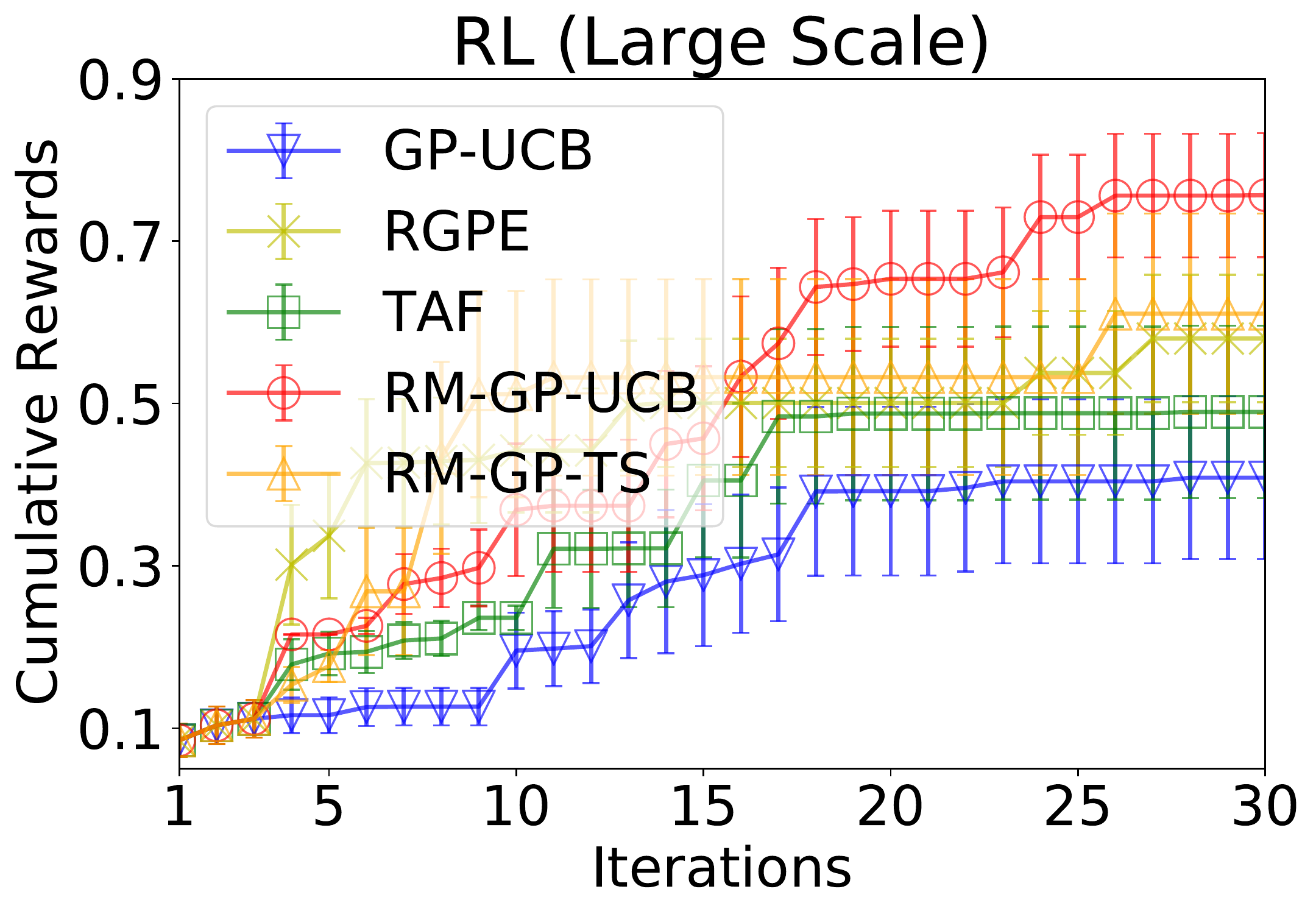}
		\caption{Cumulative rewards.}
	\end{subfigure}
	\hfill
	\begin{subfigure}[b]{0.4\linewidth}
		\includegraphics[width=\linewidth]{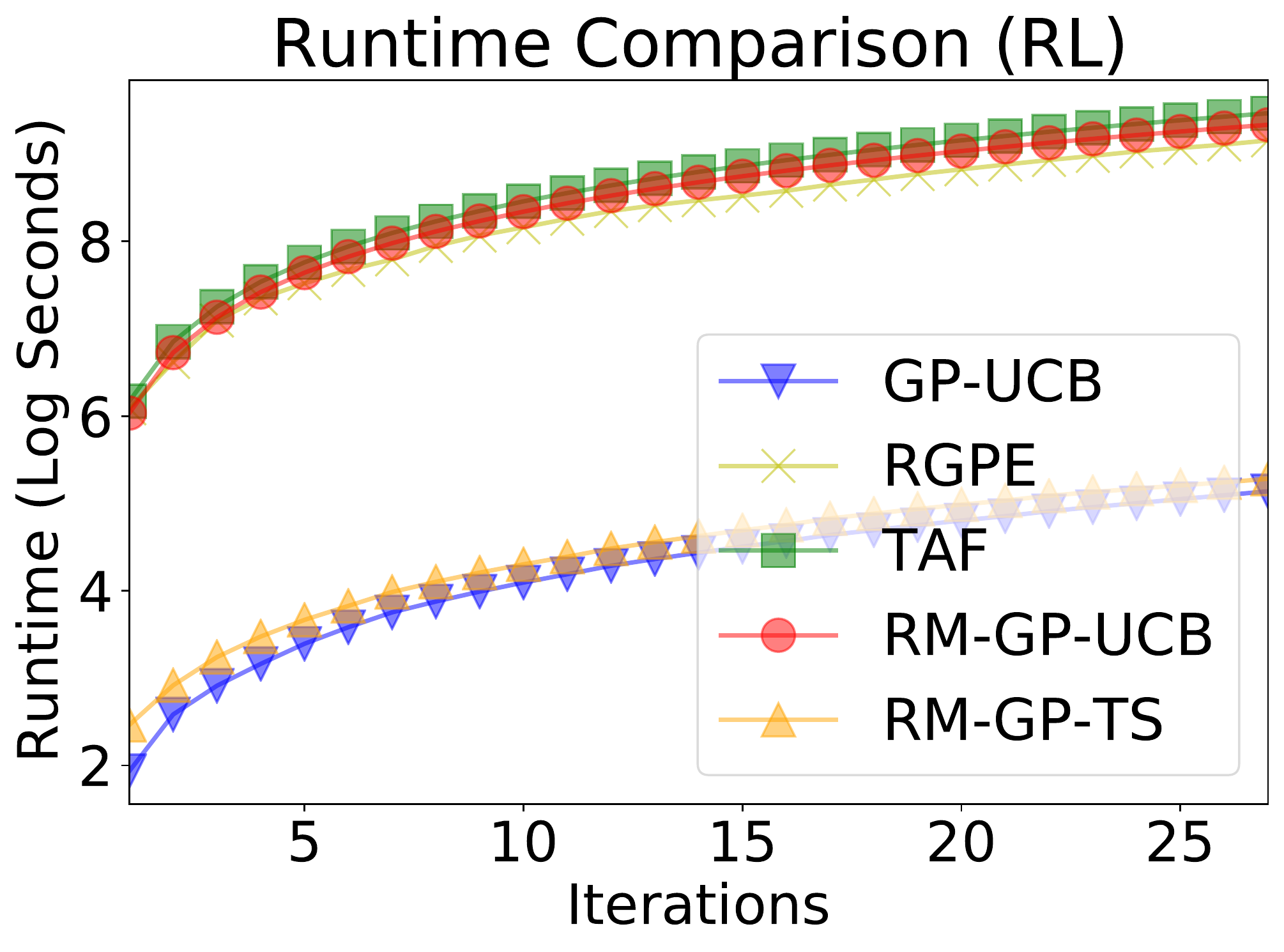}
		\caption{Runtime.}
	\end{subfigure}
	\caption{Results demonstrating that our algorithms can be applied to experiments with a very large scale, using a larger version of the RL experiment (with $60\times 130=7800$ meta-observations).}
	\label{fig:scalability_2}
\end{figure}

\subsection{More Details on RM-GP-TS}
\label{app:ts:details}
In this section, we present more details on the practical implementation of our RM-GP-TS algorithm.
In all experiments, when sampling a function from the GP posterior, we use random Fourier features (RFF)~\citep{dai2020federated,rahimi2008random} with $m=120$ random Fourier features.
Firstly, we need to construct a set of random features. For an SE kernel with hyperparameters $l$ and $\sigma_k$ (i.e., $k(\mathbf{z})=\sigma_k^2e^{-\frac{\norm{\mathbf{z}}^2_2}{2l^2}}$, with $\mathbf{z}=\mathbf{x}_1-\mathbf{x}_2,\forall \mathbf{x}_1,\mathbf{x}_2\in\mathcal{D}$), we firstly sample $m$ vectors $\{\mathbf{s}_i\}_{i=1,\ldots,m}$ from the $D$-dimensional Gaussian distribution: $\mathcal{N}(0, \frac{1}{l^2}I)$, and sample $m$ scalar values $\{b_i\}_{i=1,\ldots,m}$ from the uniform distribution within the domain $[0, 2\pi]$.
Next, for any input $\mathbf{x}\in\mathcal{D}$, its corresponding $m$-dimensional random features can be constructed as $\boldsymbol{\phi}(\mathbf{x})=[\sqrt{2/m}\cos(\mathbf{s}_i^{\top}\mathbf{x} + b_i)]^{\top}_{i=1,\ldots,m}$. Every $\boldsymbol{\phi}(\mathbf{x})$ is then normalized such that $\norm{\boldsymbol{\phi}(\mathbf{x})}^2_2=\sigma_k^2,\forall \mathbf{x}\in\mathcal{D}$.
Based on these, in order to (approximately) sample a function from the GP posterior, we firstly sample a vector $\boldsymbol{\omega}$ from the Gaussian distribution $\boldsymbol{\omega}\sim\mathcal{N}(\boldsymbol{\nu}_t,\sigma^2\boldsymbol{\Sigma}_t)$, with $\boldsymbol{\Sigma}_t=(\boldsymbol{\Phi}_t^{\top}\boldsymbol{\Phi}_t+\sigma^2 \boldsymbol{I})^{-1}$,  $\boldsymbol{\nu}_t=\boldsymbol{\Sigma}_t\boldsymbol{\Phi}_t^{\top}\mathbf{y}_t$, and $\boldsymbol{\Phi}_t=[\boldsymbol{\phi}(\mathbf{x}_1,\ldots,\mathbf{x}_t)]^{\top}$.
Finally, we can use the sampled $\boldsymbol{\omega}$ to construct the sampled function such that $f^t(\mathbf{x})=\boldsymbol{\phi}(\mathbf{x})^{\top}\boldsymbol{\omega},\forall \mathbf{x}\in\mathcal{D}$.
As a result, as mentioned in Sec.~\ref{sec:om_gp_ucb}, for a meta-task $i$, in order to sample multiple functions from the meta-function $f_i$ before the algorithm starts, we simply need to draw multiple samples of the vector $\boldsymbol{\omega}$ from the corresponding multivariate Gaussian distribution using the observations from meta-task $i$.
For both the target function and every meta-function, the kernel hyperparameters ($l$ and $\sigma_k$) used in the posterior sampling steps above are learned by maximizing the marginal likelihood (using full GP without RFF approximation), which is a common practice in BO.


\end{document}